\documentclass{article}
\PassOptionsToPackage{numbers,sort&compress}{natbib}
\ifdefined\anonymousversion
  \usepackage[main]{neurips_2026}
\else
  \usepackage[main,final]{neurips_2026}
\fi
\usepackage[T1]{fontenc}
\usepackage[utf8]{inputenc}
\usepackage{amsmath,amssymb,amsthm,mathrsfs}
\usepackage{graphicx,booktabs,array,tabularx}
\usepackage{tikz}
\usetikzlibrary{arrows.meta,positioning,fit,calc,backgrounds}
\usepackage{algorithm,float}
\usepackage{enumitem,microtype,xcolor}
\usepackage{xurl} 
\usepackage{hyperref}
\hypersetup{colorlinks=true,linkcolor=blue,citecolor=blue,urlcolor=blue,
 pdftitle={Nonparametric In-Context Learning under Growing Geometric Complexity: Minimax Optimality and Local Geometry-Adaptivity of Transformers},
 pdfauthor={Jaehee Seo and Jisu Kim}}
\ifdefined\anonymousversion\hypersetup{pdfauthor={}}\fi

\newtheorem{theorem}{Theorem}
\newtheorem{lemma}{Lemma}
\newtheorem{corollary}{Corollary}

\newtheorem{assumption}{Assumption}
\newtheorem{assumption*}{Assumption}[section]
\theoremstyle{definition}
\newtheorem{definition}{Definition}
\newtheorem{remark}{Remark}

\title{Nonparametric In-Context Learning under Growing Geometric Complexity:
Minimax Optimality and Local Geometry-Adaptivity of Transformers}

\author{Jaehee Seo\\
Department of Statistics\\
Seoul National University, Korea\\
\texttt{seojaehee02@snu.ac.kr}
\And
Jisu Kim\thanks{Corresponding author.}\\
Department of Statistics\\
Seoul National University, Korea\\
\texttt{jkim82133@snu.ac.kr}}

\begin{document}
\maketitle
\begin{abstract}
Transformers have become a central architecture for \emph{in-context learning} (ICL),
particularly through their state-of-the-art performance in large language
models. This success motivates understanding how transformers exploit intrinsic geometry of real-world data. However, existing nonparametric ICL theory has largely focused on simplified settings, such as Euclidean domains or
single manifold models, failing to fully explain its phenomenal empirical success in complex tasks.
To address this gap, we study the prediction problem under unknown heterogeneous geometry, modeled by sample size-dependent mixtures of manifolds
with heterogeneous dimensions, smoothness, and sampling masses. Under the growing geometric complexity model, we establish a minimax lower bound capturing
the aggregate difficulty of the components and construct an oracle tangent
local polynomial estimator with a matching upper bound. This estimator is connected to a structure-informed, two-stage softmax transformer with
a geometric preconditioner and local polynomial solvers.
The transformer achieves negligible approximation error relative to the minimax
rate with logarithmic depth and polynomial size. Finally, we derive an in-context generalization bound for
near empirical risk minimizers (ERM) over this class. Together, these results identify
conditions under which the resulting predictor exploits local geometry and
attains minimax rate.
\end{abstract}

\section{Introduction}\label{sec:introduction}
Transformers were introduced as attention-based architectures that process a
sequence by extracting and recombining task-relevant information across tokens
\citep{vaswani2017attention}. Unlike architectures that rely on a fixed
sequential state, self-attention allows each token to interact directly with all
other tokens in a data-dependent way. This mechanism is particularly well-suited
to prompt-based prediction, where instructions, demonstrations, contextual
information, and a query are presented in a single sequence. It has become a
cornerstone of large language models, whose strong performance
in few-shot and ICL suggests that a model can adapt its
prediction from examples already contained in the prompt, without
task-specific parameter updates \citep{brown2020language}. In this sense, the
prompt is not merely an input string but also a source of task information from
which the model must infer how to predict at the query point.

Recent works have studied a statistical explanation of this phenomenon by viewing
transformers as prompt-dependent statistical estimators implementing learning procedures
inside their forward pass
\citep{garg2022can,akyurek2023what,von2023transformers,bai2023transformers}.
In this paper, we focus on the nonparametric ICL problem: given a prompt
$\mathfrak s=((X_i,Y_i)_{i=1}^n;X_{n+1})$ with
$X_i=X_i^\star+\xi_i$ and $Y_i=f(X_i^\star)+\epsilon_i$, the predictor
estimates $f(X_{n+1}^\star)$. Here $X_i^\star$ is the latent covariate,
$\xi_i$ and $\epsilon_i$ are observation perturbations, and $f$ is an unknown
task function. The query response is not an input, and prediction uses the
examples in the prompt without task-specific parameter updates. Existing nonparametric ICL results have primarily been developed for Euclidean
domains or a single manifold
\citep{kim2024transformers,ching2026efficient,shen2026understanding}.
Such settings provide a useful starting point. However, a central question for transformers, demonstrating the state-of-the-art performances in many complex tasks, still remains unanswered:

\begin{center}
\emph{Can transformers still attain minimax optimality for nonparametric in-context learning \\under unknown heterogeneous geometry?}
\end{center}

Data for contemporary machine learning, such as images and language representations, strongly motivate this question.
Image data have been linked to low intrinsic dimension and a union-of-manifolds
structure \citep{pope2021the,brown2023verifying}; a related account uses
CW complexes \citep{wang2024cw}. Region-dependent intrinsic dimension
\citep{allegra2020data}, stratified language-model embeddings
\citep{li2025unraveling}, and intrinsic-dimension-based text analysis
\citep{tulchinskii2023intrinsic} further motivate local heterogeneity.
Transformer representations also exhibit training-dependent variation in
intrinsic dimension \citep{razzhigaev2024shape}. An embedding layer therefore
need not remove low-dimensional or heterogeneous geometry.

To reflect this geometric heterogeneity, we construct the statistical model through a sample-size-dependent mixture of
manifolds. Components can have different dimensions, smoothness levels, and
masses, while their number, reach, separation, and perturbation scales may
vary with $n$. The geometry remains locally resolvable at the regression
scale. We ask what risk is unavoidable, which estimator attains it, and
whether that estimator can be realized by a transformer that generalizes
across tasks.

\subsection{Contribution}\label{subsec:contribution}
\begin{itemize}[leftmargin=*,itemsep=3pt,topsep=3pt,parsep=0pt]
\item \textbf{Growing Geometric setup.}
We formulate a growing heterogeneous manifold-mixture model using sample size-dependent mixture numbers and weights, and specify its
locally resolvable regression regime.

\item \textbf{Aggregate minimax lower bound.}
We establish a minimax lower bound of order
$\mathfrak R_n$ in \eqref{eq:rate-Rn}, retaining each component's dimension,
smoothness, and effective sample size (Theorem~\ref{thm:minimax-lower}). 

\item \textbf{Oracle tangent local polynomial estimator and upper bound.}
We fit a polynomial graph and then regress in the estimated tangent coordinates (Algorithm~\ref{alg:localpoly}). The oracle supplies dimension, smoothness,
and bandwidth, not the tangent projector or task function. The estimator
attains the matching upper bound $\mathfrak R_n$ (Theorem~\ref{thm:oracle-upper}).

\item \textbf{Transformer approximation.}
We construct a structure-informed transformer using softmax attention and ReLU FFNs that performs
geometric preconditioning and chartwise reduced regression in one forward pass (Figure~\ref{fig:end-to-end}). Stable charts and ReLU-to-transformer
compilation give mean squared oracle-approximation error $O(n^{-A_0})$ for each
fixed $A_0>0$, with logarithmic depth and polynomial embedding dimension,
feed-forward width, and parameter bounds (Theorem~\ref{thm:transformer-approx}). 
\item \textbf{Generalization error bound.}
We separate near-ERM risk into in-context,
approximation, finite-task, and optimization terms, yielding fresh-task risk
$O(\mathfrak R_n)$ with sufficient pretraining and small optimization error (Theorem~\ref{thm:icl-erm}).
A matching $\mathfrak R_n$ lower bound holds for every pretraining budget,
worst-case over task distributions.
\end{itemize}

\paragraph{Relation to prior geometric results.}\label{subsec:related-works}
\citet{shen2025transformers} analyze noisy and task-level manifolds, and
\citet{shen2026understanding} connect structured-manifold ICL to kernel
prediction. Their noisy-manifold regression and structured-manifold ICL
results use a single manifold and H\"older exponents in $(0,1]$. Higher-order
local polynomial based ICL is already available in Euclidean domains
\citep{ching2026efficient}; the issue here is to control regression in
estimated coordinates uniformly across a growing mixture. The perturbation
models and structural information also differ, so our results do not uniformly extend all prior guarantees.

Extended related works are collected in Appendix~\ref{app:related}. Appendix~\ref{app:numerical} reports empirical performance gains with increasing
pretraining meta-sample size and comparisons between the transformer and a
local polynomial estimator using the oracle tangent.

\section{Problem setup}\label{subsec:assumption}
\paragraph{Conventions.}
For scalars, $\land$ and $\lor$ denote minimum and maximum, respectively.
Write $[m]=\{1,\ldots,m\}$ for $m\in\mathbb N$ and $[m]_0=\{0,\ldots,m\}$ for integers $m\ge0$. For $a\in\mathbb R$,
$\lceil a\rceil$ denotes the smallest integer not less than $a$.
Let $\Pi_{T_x\mathcal M}$ be the orthogonal projector onto the tangent space. The reach of a closed manifold is the largest
tubular radius with unique nearest-point projection \citep{federer1959curvature}.
For $\tau>0$, $\mathcal C^\beta_{d,L_{\mathcal M},\tau}$ consists of closed, connected,
embedded $d$-manifolds without boundary satisfying
$\operatorname{reach}(\mathcal M)\ge\tau$, with uniformly regular tangent-normal
charts of fixed radius independent of $\tau$. The H\"older ball $\mathcal H^\alpha(\mathcal M;L_{\mathcal F})$
is defined by uniform chart pullbacks; at integer $\alpha$ we use the
$C^{\alpha-1,1}$ convention. Precise definitions are in
Appendix~\ref{app:preliminaries}; the transformer class is defined in
Section~\ref{subsec:transformers}.

\paragraph{Mixture of manifolds.}
Independently of $n$, fix $D\ge2$, $d_{\max}\in[D-1]$, $0<\alpha_{\min}\le\alpha_{\max}$,
$B_X,L_{\mathcal M}\ge1$, $c_\mu\in(0,1]$, and
$L_{\mathcal F},B_\epsilon,\sigma_Y>0$.
Set $\beta_{\max}=\lceil\alpha_{\max}\rceil+1$.
Component geometry and smoothness may depend on $n$; for convenience, we abbreviate
$(\mathcal M_{k,n},d_{k,n},\alpha_{k,n},\beta_{k,n})$ by
$(\mathcal M_k,d_k,\alpha_k,\beta_k)$. Reach, separation, masses, and
perturbation bounds retain their $n$ subscripts.

\begin{assumption}[Growing heterogeneous geometry]\label{ass:geometry}
For each $n$, let $K_n\in\mathbb N$ and $\tau_{0,n},\delta_{0,n}>0$. There are $K_n$ manifolds
$\mathscr M_n=(\mathcal M_k)_{k=1}^{K_n}$, with
$\mathcal M_k\in\mathcal C^{\beta_k}_{d_k,L_{\mathcal M},\tau_{0,n}}$,
$\mathcal M_k\subset[-B_X,B_X]^D$, $d_k\in[d_{\max}]$, and
$2\le\beta_k\le\beta_{\max}$, such that
\[
 \min_{k\ne\ell}d(\mathcal M_k,\mathcal M_\ell)\ge\delta_{0,n}>0.
\]
Here $d(A,B)=\inf_{a\in A,b\in B}\|a-b\|$; the separation condition is vacuous
when $K_n=1$.
\end{assumption}
The local regularity constant stays fixed, whereas $\tau_{0,n}$ and $\delta_{0,n}$ may
shrink. The local chart derivatives remain uniformly controlled along the
component sequence. Figure~\ref{fig:growing-mixture-geometry} illustrates these distinct roles
of local regularity, component count, and separation.

\begin{assumption}[Latent design and bounded perturbations]\label{ass:latent-model}
The label $Z\in[K_n]$ has masses $\pi_{k,n}>0$ with
$\sum_{k\in[K_n]}\pi_{k,n}=1$. For each $k\in[K_n]$, $X^\star\mid(Z=k)$ has law
$\mu_{k,n}(dx)=\varrho_{k,n}(x)d\operatorname{vol}_{\mathcal M_k}(x)$ with
$c_\mu\le\varrho_{k,n}\le c_\mu^{-1}$
$\operatorname{vol}_{\mathcal M_k}$-a.e. For $\sigma_{k,n}\ge0$,
\begin{equation*}
 \mathbb P(\|\xi\|\le\sigma_{k,n}\mid X^\star,Z=k)=1.
\end{equation*}
The response noise satisfies, almost surely,
\begin{equation}\label{eq:response-condition}
 \mathbb E[\epsilon\mid X^\star,\xi,Z]=0,\qquad
 |\epsilon|\le B_\epsilon,\qquad
 \mathbb E[\epsilon^2\mid X^\star,\xi,Z]\le\sigma_Y^2.
\end{equation}
\end{assumption}
The perturbation magnitude is restricted at the regression resolution below.

\paragraph{Tasks and observations.}
Write $\mathscr M=\mathscr M_n$, $\boldsymbol\pi=(\pi_{k,n})_k$, and
$\boldsymbol\sigma_X=(\sigma_{k,n})_k$.
For $\alpha_k\in[\alpha_{\min},\alpha_{\max}]$ assume
$\beta_k\ge\lceil\alpha_k\rceil+1$, and set
\begin{equation*}
 \mathcal H=\mathcal H(\boldsymbol\alpha,\mathscr M;L_{\mathcal F})
 :=\bigg\{f:\bigsqcup_k\mathcal M_k\to\mathbb R:
         f|_{\mathcal M_k}\in\mathcal H^{\alpha_k}(\mathcal M_k;L_{\mathcal F})\bigg\}.
\end{equation*}
Let $\mathcal P_\star=\mathcal P_\star(\mathscr M,\boldsymbol\pi,
\boldsymbol\sigma_X;\sigma_Y^2,B_\epsilon,c_\mu)$ be the nonempty class of laws
satisfying Assumption~\ref{ass:latent-model}. For $P\in\mathcal P_\star$ and
$f\in\mathcal H$, draw i.i.d. latent/noise tuples and observe
$X_i=X_i^\star+\xi_i$, $Y_i=f(X_i^\star)+\epsilon_i$, for $i\le n$, together
with $X_{n+1}$. The target is $f(X_{n+1}^\star)$. In particular, $|Y_i|\le B_Y:=L_{\mathcal F}+B_\epsilon$.
For a prompt predictor $g$, abbreviate its risk by
\begin{equation*}
 \mathcal E_{P,f}(g)
 :=\mathbb E_{P,f}\left[\{g(\mathfrak s)-f(X_{n+1}^\star)\}^2\right].
\end{equation*}
When $g$ is written as an estimate at $X_{n+1}$, its dependence on the
in-context sample is implicit.

\begin{assumption}[Feasible local regime]\label{ass:feasible-local}
For a fixed $\kappa\in(0,1)$, put
$N_k=n\pi_{k,n}$, $h_k=h_{k,n}=N_k^{-1/(2\alpha_k+d_k)}$, and
$r_n=\max_k h_k$. For all sufficiently large $n$, uniformly over $k$,
\begin{equation*}
 N_k\ge n^\kappa,\qquad
 \tau_{0,n}\wedge\delta_{0,n}\ge a_{\rm sc}^{-1}r_n,\qquad
 \sigma_{k,n}\le a_{\rm sc}h_k^{\alpha_k\vee1},
\end{equation*}
where $a_{\rm sc}>0$ is a sufficiently small constant specified in
Appendix~\ref{app:feasible-local-formal}.
\end{assumption}
The three requirements in Assumption~\ref{ass:feasible-local} provide enough local samples, prevent different
components from entering the same regression window, and keep covariate
perturbations below the target bias scale. They imply $K_n\le n^{1-\kappa}$
and polynomially increasing local counts $N_kh_k^{d_k}$.
Labels index the mixture components; dimensions and smoothness levels may
coincide across components. $X^\star$ is the generating latent point;
it can differ from the nearest-point projection of $X$. The risk and tangent
comparison refer to this same generating point.
We treat $X\in\mathbb R^D$ as a numerical
covariate, possibly obtained from a fixed learned representation map. This map
is distinct from $\mathrm{Emb}_n$, which packs examples into tokens; learning
the representation is outside our analysis.

As \cite{divol2022measure} pointed out, the tubular noise model may not be identifiable, since there are several admissible couples $(X^\star,\xi)$ such that $X=X^\star+\xi$ follows the same distribution. Our model is specified by the full latent law $P\in\mathcal P_\star$;
the target and risk are defined under this law, and our bounds hold
uniformly over admissible $P$. 

\paragraph{Reading the scales.}
The number of observations relevant to a query is smaller than $N_k$:
within a bandwidth-$h_k$ neighborhood it is of order $N_kh_k^{d_k}$.
The local volume and density bounds give this scale; the chosen bandwidth
balances the squared approximation bias and response-noise variance:
\[
 h_k^{2\alpha_k}=\frac1{N_kh_k^{d_k}},\qquad
 N_kh_k^{d_k}=N_k^{2\alpha_k/(2\alpha_k+d_k)}.
\]
The effective-sample condition makes this local count grow polynomially,
so the logarithmic cost of comparing tangent candidates remains smaller.
The geometric condition has a different role: reach permits a stable local
graph, while separation prevents observations from another component from
entering that graph fit. Finally, the perturbation condition ensures
$\sigma_{k,n}\lesssim h_k^{\alpha_k}$ when $\alpha_k>1$ and
$\sigma_{k,n}\lesssim h_k$ otherwise. These bounds keep the perturbation at the regression resolution.

\begin{figure}[t]
\centering
\includegraphics[width=0.8\linewidth]{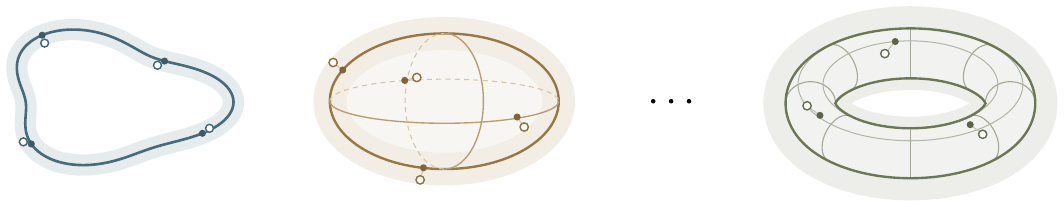}
\caption{A heterogeneous manifold mixture with bounded covariate noise.
The closed curve, ellipsoidal surface, and torus represent components
of intrinsic dimensions one, two, and two, respectively;
the ellipsis denotes omitted components.
Filled dots mark latent covariates $X_i^\star$ on the manifolds,
and open circles mark observed covariates $X_i=X_i^\star+\xi_i$.
The connecting segments indicate the perturbations $\xi_i$,
while the pale shaded neighborhoods schematically depict
component-specific noise bounds $\|\xi_i\|\le \sigma_{k,n}$.
In the growing-complexity setting, the number of components $K_n$
may increase while the reach and separation lower bounds
$\tau_{0,n}$ and $\delta_{0,n}$ may decrease with $n$.}
\label{fig:growing-mixture-geometry}
\end{figure}

\paragraph{Uniformity of constants.}
The asymptotic notations $\lesssim,\gtrsim$ hide constants depending only on the fixed
parameters above, $\kappa$, and fixed construction margins. They are uniform
in $n,K_n,k,P,f$ and in admissible component sequences. Shrinking masses,
reaches, and separations enter through the displayed scales. We distinguish statistical and implementation
constants below.

\section{Minimax rate and oracle estimation}
\subsection{Minimax lower bound}\label{sec:minimax-lower}
If the query belongs to component $k$, its intrinsic dimension is $d_k$, its
smoothness is $\alpha_k$, and its effective sample size is $N_k=n\pi_{k,n}$.
Accordingly, define the target minimax rate as
\begin{equation}\label{eq:rate-Rn}
 \mathfrak R_n=\mathfrak R_n(\boldsymbol\alpha,\mathbf d,\boldsymbol\pi)
 :=\sum_{k\in[K_n]}\pi_{k,n}(n\pi_{k,n})^{-2\alpha_k/(2\alpha_k+d_k)}.
\end{equation}
This averages the component-specific statistical difficulties. In the balanced homogeneous
case it equals $(n/K_n)^{-2\alpha/(2\alpha+d)}$; for $K_n=1$ it is the
classical nonparametric rate
\citep{stone1982optimal,tsybakov2009introduction,gyorfi2002distribution,bickel2007local}.

The aggregate also makes the effect of growing mixture complexity explicit.
For any admissible balanced sequence with $K_n\asymp n^\eta$ and common
$(d,\alpha)$, it becomes $n^{-(1-\eta)2\alpha/(2\alpha+d)}$; the effective
sample condition requires $\eta\le1-\kappa$. This calculation applies to sequences satisfying the geometric assumptions.

More generally, writing
$a_k=2\alpha_k/(2\alpha_k+d_k)$, component $k$ contributes
$n^{-a_k}\pi_{k,n}^{1-a_k}$. Small mass worsens the conditional estimation problem
through $n\pi_{k,n}$, but simultaneously reduces how often that difficulty is
encountered at the query. The benchmark keeps both effects. When masses or
component counts vary with $n$, replacing this average by a single largest
dimension can therefore discard the relevant distribution of difficulty.

\begin{samepage}
\begin{theorem}[Minimax lower bound]\label{thm:minimax-lower}
Under Assumptions~\ref{ass:geometry}--\ref{ass:feasible-local}, for all
sufficiently large $n$,
\[
 \inf_g\sup_{f\in\mathcal H,\,P\in\mathcal P_\star}\mathcal E_{P,f}(g)
 \gtrsim\mathfrak R_n,
\]
where the infimum ranges over measurable predictors $g$ of the observed prompt $\mathfrak s$.
\end{theorem}
\end{samepage}
Here, the hidden lower bound constant depends only on fixed statistical model parameters
and scale margins. Its proof constructs one Assouad family with perturbations
on every component at its own resolution. A single testing argument therefore
makes all component contributions simultaneous. A bounded, smooth response-noise
submodel and a Hellinger translation bound respect
\eqref{eq:response-condition}. Formal details and the proof are given in Appendix~\ref{app:minimax-lower}.

\subsection{Oracle tangent local polynomial estimator}\label{sec:oracle-estimator}
Now we construct the statistical estimator that the transformer will approximate in the next section. For a given query $x=X_{n+1}$ from component $k_x$, we suppose the oracle knows
$d_x=d_{k_x}$, $\alpha_x=\alpha_{k_x}$,
$p_x=\lceil\alpha_x\rceil-1$, $s_x=p_x+1$, and $h_x=h_{k_x}$.
Hence the inputs are the observed prompt and these structural quantities. The
tangent projector and local regression coefficients are fitted from that
prompt: a polynomial graph fit estimates the tangent, followed by regression
in the projected coordinates. Higher-order graph fitting controls geometric bias at the regression resolution
\citep{Aamari2019Nonasymptotic,cheng2013local}.

\paragraph{Why fit the geometry before regression?}
At the latent query $x^\star$, a tangent-normal chart writes a nearby point
as $x^\star+u+G_{x^\star}(u)$, with $G_{x^\star}(0)=DG_{x^\star}(0)=0$.
The manifold regularity supplies a $s_x$-order Taylor expansion of $G$.
Substituting $u=h_xt$ and dividing ambient displacements by $h_x$,
its remainder becomes $O(h_x^{s_x})$. This explains the choice
$s_x=\lceil\alpha_x\rceil$ for geometry and $p_x=s_x-1$ for regression:
the two polynomial fits approximate different objects. Moreover, the
observed query can be off the manifold. In the proof, the candidate offset
$a^\circ=(I-\Pi^\star)\xi_{n+1}/h_x$ compensates its normal displacement,
where $\Pi^\star=\Pi_{T_{x^\star}\mathcal M_{k_x}}$. The remaining coordinate perturbation is controlled at order $\sigma_{k_x,n}/h_x$.
These latent quantities define the comparison candidate used in the proof.
A nearby grid point provides the reference graph loss against which the
observed-data fit is analyzed.

\paragraph{Finite grid and its cost.}
A candidate $q=(a,\Pi,T_2,\ldots,T_s)$ has $a\in\mathbb R^D$ with $\|a\|\le1$, a rank-$d$
orthogonal projector $\Pi$, and symmetric multilinear maps
$T_\ell:(\mathbb R^D)^\ell\to\mathbb R^D$ with
$\|T_\ell\|_{\rm op}\le2L_{\mathcal M}$. Let $\Theta_{d,s}$ be this compact
parameter space. Use the product $n^{-1}$-net $\mathcal Q_{d,s,n}$ of
Lemma~\ref{lem:tangent-product-net}, under the sum of the offset norm,
projector Frobenius norm, and tensor operator norms, and fix
$\Delta_{{\rm tan},n}=n^{-3}$. This grid satisfies
\begin{equation*}
 |\mathcal Q_{d,s,n}|\le Cn^{m_\Theta(d,s,D)},\qquad
 m_\Theta(d,s,D)=D+d(D-d)+D\sum_{\ell=2}^{s}\binom{D+\ell-1}{\ell}.
\end{equation*}
Its logarithmic cardinality enters tangent concentration, while its
cardinality determines computational cost, with an exponent depending on
$D,s$. The deterministic grid and its loss coefficients are encoded in the
comparator weights (Appendix~\ref{app:geometric-preconditional}). The polynomial
bound treats $D,s$ as fixed; its exponent records the ambient-dimensional
cost of the geometric search in Algorithm~\ref{alg:localpoly}.

\paragraph{Localized graph fit.}
For the scaled ambient displacement $z_i=(X_i-x)/h_x$, use the fixed cutoff
$\mathsf A:[0,\infty)\to[0,1]$, where $\mathsf A(t)=1$ for $0\le t\le1/2$,
$\mathsf A(t)=2-2t$ for $1/2<t<1$, and $\mathsf A(t)=0$ for $t\ge1$.
Define
\begin{equation}\label{eq:graph-residual}
 R_q(z)=a+(I-\Pi)z-\sum_{\ell=2}^{s_x}T_\ell\{(\Pi z)^{\otimes\ell}\},
 \qquad
 L_q(x)=\frac{h_x^2}{n}\sum_{i\in[n]}\mathsf A(\|z_i\|_1)\|R_q(z_i)\|^2.
\end{equation}
The tensor sum is empty for $s_x=1$. Loss-thresholded weights average
$\Delta_{{\rm tan},n}$-near-minimizing projectors, with normalizing denominator
at least $\Delta_{{\rm tan},n}>0$. This convex average can have fractional
eigenvalues. We spectrally round it to $\widehat\Pi_x=\Pi_{(d_x)}(\bar P_x)$, where $\Pi_{(d)}$
selects the top-$d$ eigenspace with fixed measurable tie-breaking rule. With high probability, a fixed spectral gap makes this step stable (Lemma~\ref{lem:tangent-projector-good}).

\paragraph{Ambient localization and projected regression weights.}
Let $\widehat v_i=\widehat\Pi_xz_i\in\mathbb R^D$, the scaled displacement
projected onto the estimated tangent subspace. For the fixed kernel
$\mathsf K(v)=(1-\|v\|_1)_+$, where $(t)_+=t\lor 0$, the regression
weight is $\widehat W_i=h_x^{-d_x}\mathsf A(\|z_i\|_1)\mathsf K(\widehat v_i)$.
The cutoff excludes observations with $\|X_i-x\|_1\ge h_x$, yielding a
single-component window under the feasible local regime. Within that window,
the kernel tapers weights according to $\|\widehat v_i\|_1$ and is zero
when $\|\widehat v_i\|_1\ge1$. Projection can hide a large ambient displacement, which motivates the
ambient cutoff. The resulting nonnegative weights define the local
least-squares objective. Its fitted solution is invariant to the common
scale $h_x^{-d_x}$. When the weight sum is positive, dividing by it also
preserves the fitted coefficient (Remark~\ref{rem:regression-normalization}).
Let $\widehat\phi_i=\Phi_{D,p_x}(\widehat v_i)$, where
$\Phi_{D,p}(v)=(v^\nu/\nu!)_{|\nu|\le p}$ and $e_0$ selects the constant
coordinate. Algorithm~\ref{alg:localpoly} summarizes the estimator.

\begin{algorithm}[!htbp]
\caption{Oracle tangent local polynomial estimation}
\label{alg:localpoly}
\begingroup
\renewcommand{\arraystretch}{1.12}
\begin{tabularx}{\linewidth}{@{}r@{\hspace{0.7em}}>{\raggedright\arraybackslash}X@{}}
1: & \textbf{Input:} $(X_i,Y_i)_{i=1}^n$, query $x$, $d_x,p_x,s_x,h_x$,
     $\mathcal Q_{d_x,s_x,n}$, $\Delta_{{\rm tan},n}$, and $L_{\mathcal F}$.\\
2: & Set $z_i=(X_i-x)/h_x$. Evaluate $L_q(x)$ in
     \eqref{eq:graph-residual} for every $q\in\mathcal Q_{d_x,s_x,n}$.\\
3: & Set $L_{\min}=\min_q L_q(x)$ and
     $t_q=\{\Delta_{{\rm tan},n}-(L_q(x)-L_{\min})\}_+$.
     Normalize $\omega_q=t_q/\sum_{q'}t_{q'}$.\\
4: & Form $\bar P_x=\sum_q\omega_q\Pi_q$ and
     $\widehat\Pi_x=\Pi_{(d_x)}(\bar P_x)$ with fixed measurable tie-breaking.\\
5: & For each $i\in[n]$, set $\widehat v_i=\widehat\Pi_xz_i$,
     $\widehat\phi_i=\Phi_{D,p_x}(\widehat v_i)$, and
     $\widehat W_i=h_x^{-d_x}\mathsf A(\|z_i\|_1)\mathsf K(\widehat v_i)$.\\
6: & Form $\widehat G_x^\Pi=n^{-1}\sum_{i\in[n]}\widehat W_i\widehat\phi_i\widehat\phi_i^\top$
     and $\widehat b_x^\Pi=n^{-1}\sum_{i\in[n]}\widehat W_i\widehat\phi_iY_i$.\\
7: & Compute $\widehat w_x=(\widehat G_x^\Pi)^\dagger\widehat b_x^\Pi$.\\
8: & \textbf{Return:}
     $\widehat f_{\rm plug}^{\Pi}(x)=(-L_{\mathcal F})\lor e_0^\top\widehat w_x\land L_{\mathcal F}$.\\
\end{tabularx}
\endgroup
\end{algorithm}

The final prediction is
\begin{equation}\label{eq:tangent-local}
 \widehat f_{\rm plug}^{\Pi}(x)
 =(-L_{\mathcal F})\lor e_0^\top(\widehat G_x^\Pi)^\dagger\widehat b_x^\Pi
     \land L_{\mathcal F}.
\end{equation}
Ambient monomials become redundant after projection, so $\widehat G_x^\Pi$
may be singular even on a good design. The pseudoinverse acts on the
identifiable polynomial subspace, where stability follows from
\[
 \lambda_{\min}^+(\widehat G_x^\Pi)\ge c\pi_{k_x,n},\qquad
 e_0^\top(\widehat G_x^\Pi)^\dagger\widehat G_x^\Pi=e_0^\top
\]
on the regression good event. These properties identify the fitted intercept.

Indeed, $\Phi_{D,p}(Uu)=B_{U,p}\Phi_{d,p}(u)$ for an orthonormal frame $U$
and a full-column-rank $B_{U,p}$. A positive definite reduced Gram $H$
therefore induces the ambient Gram $B_{U,p}HB_{U,p}^\top$, whose zero
eigenvalues reflect redundant coordinates. Since $e_{0,D,p}=B_{U,p}e_{0,d,p}$,
the intercept remains identifiable. The transformer solves in the reduced
dictionary; Lemma~\ref{lem:chart-equivalence} proves that, on that event, all
positive-weight charts return the same intercept as \eqref{eq:tangent-local}.

\paragraph{Local averaging as a special case.}
When $0<\alpha_x\le1$, the regression degree is $p_x=0$, so its dictionary
contains only the constant. On a nonempty weighted window,
\[
 \widehat f_{\rm plug}^{\Pi}(x)
 =(-L_{\mathcal F})\lor\frac{\sum_{i\in[n]}\widehat W_iY_i}{\sum_{i\in[n]}\widehat W_i}
       \land L_{\mathcal F}.
\]
For $p_x\ge1$, higher-order terms fit local variation while the intercept
estimates the value at the query origin. Thus the same matrix formulation
extends local averaging to exploit higher smoothness.

\begin{theorem}[Oracle minimax upper bound]\label{thm:oracle-upper}
Under Assumptions~\ref{ass:geometry}--\ref{ass:feasible-local},
Algorithm~\ref{alg:localpoly} satisfies, for all sufficiently large $n$,
\[
 \sup_{f\in\mathcal H,\,P\in\mathcal P_\star}
 \mathcal E_{P,f}(\widehat f_{\rm plug}^{\Pi})\lesssim\mathfrak R_n.
\]
\end{theorem}
The constant may additionally depend on the fixed kernel, cutoff, grid
resolution exponent, and near-minimizer rule, uniformly in $n,K_n,k$ and the
growing grid cardinality. Together with Theorem~\ref{thm:minimax-lower}, this
attains the benchmark $\mathfrak R_n$.

\paragraph{Geometric error and the regression rate.}
Condition on the query variables with component $k$, and put
$s_k=\lceil\alpha_k\rceil$ and
$\delta_x=\|\widehat\Pi_x-\Pi_{T_{x^\star}\mathcal M_k}\|_{\rm op}$.
Appendix~\ref{app:tangent-local-poly-upper} gives
\begin{align*}
 \mathbb E[\delta_x^2\mid k,x^\star,\xi_{n+1}]
 &\lesssim h_k^{2s_k}+\sigma_{k,n}^2/h_k^2
       +\frac{\log(en)}{N_kh_k^{d_k}}+h_k^{2\alpha_k},\\
 \mathbb E[(\widehat f_{\rm plug}^{\Pi}(x)-f(x^\star))^2\mid k,x^\star,\xi_{n+1}]
 &\lesssim h_k^{2\alpha_k}+h_k^2\mathbb E[\delta_x^2\mid k,x^\star,\xi_{n+1}]
        +\sigma_{k,n}^2+\frac1{N_kh_k^{d_k}}.
\end{align*}
The tail and bad-design terms are absorbed at order $h_k^{2\alpha_k}$.
The comparison polynomial is $O(h_k)$-Lipschitz in normalized coordinates;
thus coordinate error $\delta_x+\sigma_{k,n}/h_k$ produces regression error
$O(h_k\delta_x+\sigma_{k,n})$. Consequently the tangent stochastic term is
multiplied by $h_k^2$, and $h_k^2\log(en)$ stays bounded. Since
$(N_kh_k^{d_k})^{-1}=h_k^{2\alpha_k}$, the component-wise bias and variance
balance. The tangent estimator uses only covariates, so the response-noise
variance calculation remains valid conditional on the estimated geometry.

Two uniformity steps are essential in this argument. The localized covariate
window must contain only observations from the query component, even when
the separation scale changes with $n$. Next, the regression design event
must hold for every rank-$d_k$ projector in a neighborhood of the true
projector. This uniform event accommodates the dependence of
$\widehat\Pi_x$ on the regression covariates. The positive-spectrum
bound and the weight-sum bound then give variance
$O((N_kh_k^{d_k})^{-1})$ after conditioning on all covariates. Finally, the component-wise
failure probabilities and risk bounds are uniform, so averaging over the
query label preserves the component-wise constants.

\section{Transformer realization}\label{sec:approximation}\label{subsec:transformers}
\paragraph{Transformers.}
For $N$ tokens $Z\in\mathbb R^{N\times d_E}$ of dimension $d_E$ and $M$ heads
with $Q_m,K_m,V_m\in\mathbb R^{d_E\times d_E}$, define
\[
 \mathrm{MHA}(Z)=\sum_{m\in[M]}\operatorname{softmax}_{\rm row}
 \left(\frac{ZQ_m(ZK_m)^\top}{\sqrt{d_E}}\right)ZV_m.
\]
The row-wise FFN of width $d_{\rm FFN}$ is
$\mathrm{FFN}(Z)=\mathrm{ReLU}(ZW_1+\mathbf 1_Nb_1^\top)W_2+\mathbf 1_Nb_2^\top$,
where $W_1,W_2^\top\in\mathbb R^{d_E\times d_{\rm FFN}}$,
$b_1\in\mathbb R^{d_{\rm FFN}}$, $b_2\in\mathbb R^{d_E}$, and
$\mathbf 1_N$ is the all-ones vector. Let $\theta^{(\ell)}$ collect these
matrices and biases in block $\ell\in[L]$, and put
$\theta=(\theta^{(1)},\ldots,\theta^{(L)})$. Define
\[
 H=Z+\mathrm{MHA}_{\theta^{(\ell)}}(Z),\quad \mathrm{Block}_{\theta^{(\ell)}}(Z)=H+\mathrm{FFN}_{\theta^{(\ell)}}(H),\quad \mathrm{TF}_\theta=\mathrm{Block}_{\theta^{(L)}}\circ\cdots\circ \mathrm{Block}_{\theta^{(1)}}.
\]
For $B>0$, $\mathcal T_M(N,d_E,d_{\rm FFN},L,B)$ comprises these maps
$\mathbb R^{N\times d_E}\to\mathbb R^{N\times d_E}$ with every parameter entry
bounded in absolute value by $B$; write $\mathcal T(d_E,d_{\rm FFN},L,B)$
when $M,N$ are fixed.

\paragraph{Input and output interface.}
The scalar predictor is
\begin{equation}\label{eq:transformer-predictor}
 \widehat f_{\theta_n}^{\rm TF}(\mathfrak s)
 =\mathrm{Read}_n\circ\mathrm{TF}_{\theta_n}\circ\mathrm{Emb}_n(\mathfrak s).
\end{equation}
The fixed affine embedding packs each $(X_i,Y_i)$ into one sample row and
$(X_{n+1},0)$ into one query row, giving $N_n^{\rm end}=n+1$.
The construction preserves raw data and row-type markers; extra coordinates
store intermediate geometric and regression computations. The fixed readout
uses ReLU multiplication to approximate the chart-weighted sum and clips the
resulting scalar to $[-L_{\mathcal F},L_{\mathcal F}]$
(Appendix~\ref{app:emb-and-read}). The query response is used in the
pretraining loss.

\paragraph{From geometric summaries to an intrinsic regression system.}
A selector encoded in the comparator parameters evaluates $x=X_{n+1}$ to
obtain $d_x,p_x,s_x,h_x$ and $\pi_{k_x,n}^{-1/2}$. Its network is fixed by
the deterministic structural model and $n$; the query determines the
component state (Appendix~\ref{app:geometric-preconditional}). Using this
state, the geometric block aggregates localized polynomial moments of the
observed covariates, approximates the finite-grid projector fit, and produces
tangent and chart states. The regression block combines the resulting
coordinates and local weights with the observed responses to form chartwise
normal equations. A globally defined bounded ReLU decoder approximates their
inverse on well-conditioned inputs. These are internal stages of one
forward pass (Figure~\ref{fig:end-to-end}, Appendix~\ref{app:roadmap}).
The reduced regression dictionary has size
\[
 q_{\rm reg,\max}=\binom{d_{\max}+p_{\max}}{p_{\max}},
 \qquad p_{\max}=\lceil\alpha_{\max}\rceil-1,
\]
compared with the ambient dictionary size $\binom{D+p_{\max}}{p_{\max}}$.
Geometric calculations and network-size exponents retain their dependence on $D$.

\paragraph{Fixed-dimensional summaries.}
With $s_{\max}=\lceil\alpha_{\max}\rceil$, the geometric loss depends on the
in-context covariates through
\[
 m_{\rm tan}(x)=\frac1n\sum_{i\in[n]}
 \mathsf A(\|z_i\|_1)\Phi_{D,2s_{\max}}(z_i).
\]
Indeed, each $\|R_q(z)\|^2$ is a polynomial of degree at most $2s_{\max}$,
so $L_q(x)=h_x^2c_q^\top m_{\rm tan}(x)$ for a deterministic coefficient
$c_q$. All tangent candidates can therefore be evaluated using the same
summary vector. After charting,
the analogous regression summaries are
\[
 G_j=\frac1n\sum_{i\in[n]}\pi_{k_x,n}^{-1}W_{i,j}\phi_{i,j}\phi_{i,j}^\top,
 \qquad g_j=\frac1n\sum_{i\in[n]}\pi_{k_x,n}^{-1}W_{i,j}\phi_{i,j}Y_i.
\]
The factor $\pi_{k_x,n}^{-1}$ keeps the good-event Gram spectrum bounded below
by a fixed positive constant, and cancels in $G_j^{-1}g_j$.
Row-wise FFNs compute local weighted contributions in value coordinates.
Uniform softmax with binary marker gates and deterministic value scaling
performs exact query broadcast and sample averaging
(Lemma~\ref{lem:softmax-communication}). Compact localization therefore acts
through values. The summary dimensions are fixed with respect to $n$, although their
numerical ranges and approximation widths can grow. Each arithmetic
approximator acts on a single row or a fixed-dimensional summary,
yielding polynomial resource bounds as the prompt grows.

\paragraph{Why several charts represent one estimator.}
A rank-$d$ projector specifies a subspace; continuous frame choices are local
to coordinate charts. Using all charts avoids discontinuities from choosing
one frame globally. Let $P_T$ be an exact
rank-$d$ tangent projector, let $\mathfrak C_d$ be the collection of all size-$d$ subsets
of $[D]$, and let $E_J$ select the coordinates in $J$. Define
\[
 A_J=E_J^\top P_TE_J,\qquad
 \sum_{J\in\mathfrak C_d}\det A_J=1.
\]
The second identity is Cauchy--Binet: at least one coordinate choice sees
all tangent directions. The chart rule assigns positive weight only when
$\det A_J>\tau_d$, with $\tau_d=(4|\mathfrak C_d|)^{-1}$. Since the
eigenvalues of $A_J$ lie in $[0,1]$, positivity gives
$\lambda_{\min}(A_J)>\tau_d$. On such a chart,
\[
 U_J=P_TE_JA_J^{-1/2},\qquad U_J^\top U_J=I_d,
 \qquad U_JU_J^\top=P_T.
\]
Thus the chart supplies stable orthonormal coordinates, and
$U_JU_J^\top(X_i-x)/h_x$ is the same projected covariate for every
positive-weight chart. The reduced polynomial dictionaries are different
coordinate descriptions of the same polynomial space. On the regression
good event, their Gram matrices are positive definite, so their fitted
intercepts agree with the ambient pseudoinverse intercept. Their weighted
average therefore represents the same estimator on that event. Safe chart
maps extend the coordinate formulas to other inputs.

Only positive-weight charts require inverse approximation. A zero-weight
chart may have a singular Gram matrix, but its globally defined decoder
still returns a clipped value. If approximate weights place small mass on
such a chart, the resulting error is bounded by $L_{\mathcal F}$ times
that weight error. The decoder is evaluated on every branch, including those with zero weight.
Appendix~\ref{app:approximation-error} gives the precise comparison and
controls failures of the spectral and regression good events.

\begin{theorem}[Transformer approximation]\label{thm:transformer-approx}
Fix $A_0>0$ and a deterministic structural model satisfying
Assumptions~\ref{ass:geometry}--\ref{ass:feasible-local}. For all sufficiently
large $n$, there are deterministic dimensions and a parameter vector
$\widehat\theta_n$ defining a transformer
\[
 \mathrm{TF}_{\widehat\theta_n}\in
 \mathcal T_{M_0}(n+1,d_{E,n},d_{{\rm FFN},n},L_n,B_n)
\]
whose fixed-interface predictor satisfies
\[
 \sup_{f\in\mathcal H,\,P\in\mathcal P_\star}
 \mathbb E_{P,f}\bigl[
   \{\widehat f_{\widehat\theta_n}^{\rm TF}(\mathfrak s)
      -\widehat f_{\rm plug}^{\Pi}(X_{n+1})\}^2\bigr]
 \lesssim n^{-A_0}.
\]
Here $M_0=O(1)$, $L_n\le C_{\rm depth}\log(en)$, and
$d_{E,n}+d_{{\rm FFN},n}+B_n\le n^{C_{\rm size}}$, with
$C_{\rm depth},C_{\rm size}$ and implicit constants depending only on $A_0$
and fixed model and construction bounds (Appendix~\ref{app:accuracy-resource}).
The parameter vector is fixed by the structural model and $n$, and the same
vector applies uniformly over $f\in\mathcal H$ and $P\in\mathcal P_\star$.
\end{theorem}

Combining the two comparisons by the squared triangle inequality gives
\[
 \sup_{f\in\mathcal H,\,P\in\mathcal P_\star}
 \mathcal E_{P,f}(\widehat f_{\widehat\theta_n}^{\rm TF})
 \lesssim \mathfrak R_n+n^{-A_0}.
\]
Thus a sufficiently accurate implementation preserves the statistical rate
of the oracle target. This establishes a good comparator in the transformer
class before any pretraining analysis. Choosing a predictor from that class
using finitely many tasks is a separate statistical step, addressed next.

\paragraph{Accuracy, workspace, and structural information.}
The proof compiles ReLU modules into FFNs with residual connections using two polynomial-width
scratch banks, with raw observations and markers protected
(Lemma~\ref{lem:relu-transformer-compilation}). These banks store the intermediate module states. One communication head copies or averages an entire stored vector,
so the head count stays fixed.
Corollary~\ref{cor:accuracy-resource} makes the accuracy dependence explicit:
\begin{equation}\label{eq:accuracy-resource-main}
 L_n\le(c_{L,0}+c_{L,1}A_0)\log(en),\qquad
 d_{E,n}+d_{{\rm FFN},n}+B_n\le n^{c_{S,0}+c_{S,1}A_0},
\end{equation}
with coefficients fixed by the structural bounds and module choices;
the threshold in $n$ can depend on $A_0$.
The grid and selector are encoded in the comparator's weights. Their
construction uses the specified geometry and component parameters, and is
fixed before sampling any prompt. The same construction covers the admissible
within-component density and perturbation laws. Thus the comparator is shared
across tasks within that structural environment, while its encoded
construction may vary with the environment. Its resource bounds are given in
\eqref{eq:accuracy-resource-main}, and its guarantee is in mean square, with
clipping controlling rare spectral or design failures.

\section{In-context generalization}\label{sec:icl}
\paragraph{Training and target tasks.}\label{subsubsec:generation}
Fix $P\in\mathcal P_\star$ and a distribution $\rho_f$ on $\mathcal H$.
Each of $\Gamma$ independent training tasks draws $f^{(\gamma)}\sim\rho_f$
and generates $n+1$ independent latent/noise tuples from $P$ as in
Section~\ref{subsec:assumption}. Training observes
$\mathcal D_\Gamma^{\rm tr}
=\{(\mathfrak s^{(\gamma)},Y_{n+1}^{(\gamma)})\}_{\gamma=1}^{\Gamma}$.
An independent target task is generated in the same way from a fresh
$f\sim\rho_f$. Here $n$ measures information within a task; $\Gamma$ measures
the number of training tasks. Tasks share the structural model and sampling
law.

Let $\mathcal G_n^{\rm end}$ be the class of all scalar predictors
\eqref{eq:transformer-predictor} over the deterministic architecture envelope
in Theorem~\ref{thm:transformer-approx}, with fixed embedding and readout.
All parameters, including the query and key matrices, vary freely within
the prescribed bound. Define
\begin{equation*}
 R_{P,\rho_f}^\star(g)=\mathbb E_{f\sim\rho_f}\mathcal E_{P,f}(g),\qquad
 \widehat R_\Gamma(g)=\frac1\Gamma\sum_{\gamma\in[\Gamma]}
       \{Y_{n+1}^{(\gamma)}-g(\mathfrak s^{(\gamma)})\}^2.
\end{equation*}
Fix a deterministic tolerance $\varepsilon_{{\rm opt},\Gamma}\ge0$.
A measurable estimator $\widehat g_\Gamma\in\mathcal G_n^{\rm end}$ is a near
ERM if, almost surely,
\begin{equation}\label{eq:empirical-risk-minimizer}
 \widehat R_\Gamma(\widehat g_\Gamma)
 \le\inf_{g\in\mathcal G_n^{\rm end}}\widehat R_\Gamma(g)
       +\varepsilon_{{\rm opt},\Gamma}.
\end{equation}
The tolerance bounds the achieved empirical-risk gap and enters the
generalization bound explicitly, quantifying how imperfect optimization
affects the final statistical guarantee.

\begin{theorem}[Near-ERM generalization]\label{thm:icl-erm}
Under Assumptions~\ref{ass:geometry}--\ref{ass:feasible-local}, fix $A_0>0$
and the architecture of Theorem~\ref{thm:transformer-approx}. Every measurable
near ERM in \eqref{eq:empirical-risk-minimizer} satisfies
\begin{equation*}
 \mathbb E_{\mathcal D_\Gamma^{\rm tr}}R_{P,\rho_f}^\star(\widehat g_\Gamma)
 \lesssim \mathfrak R_n+n^{-A_0}
       +\frac{\mathfrak H_{n,\Gamma}^{\rm end}+1}{\Gamma}
       +\varepsilon_{{\rm opt},\Gamma},
\end{equation*}
where the entropy proxy defined in Appendix~\ref{app:end-to-end-entropy}
obeys $\mathfrak H_{n,\Gamma}^{\rm end}
\le C_{\rm ent}n^{c_{\rm ent}}\{1+\log(e\Gamma)\}$.
Constants are uniform over $P,\rho_f$ and admissible component sequences;
they may depend on $A_0$ and fixed statistical and architectural bounds.
\end{theorem}
The four terms separate intrinsic in-context error, implementation error,
finite-task complexity, and optimization tolerance. The proof uses a bounded
square-loss oracle inequality and parameter covering of the full softmax
class. Row-stochastic attention controls hidden-state growth, and the global
Lipschitz bound for softmax controls parameter sensitivity. Discretizing the
parameter cube then yields the entropy term, including the polynomial
embedding dimension (Lemma~\ref{lem:transformer-entropy}).
Conditional mean-zero query noise allows training against
$Y_{n+1}$ while evaluating against the latent target. See
Appendix~\ref{app:generalization-error}.

In particular, for a fixed predictor $g$ the noisy and latent risks satisfy
\[
 \mathbb E\{Y_{n+1}-g(\mathfrak s)\}^2
 =R_{P,\rho_f}^\star(g)+\mathbb E\epsilon_{n+1}^2.
\]
The same identity holds for the learned predictor after conditioning on the
independent training tasks. The empirical-process argument is centered at
$g_0(\mathfrak s)=\mathbb E[Y_{n+1}\mid\mathfrak s]$, with the oracle
inequality allowing a misspecified predictor class. The good comparator provided by
Theorem~\ref{thm:transformer-approx} controls approximation to the latent
regression value, while the covering term controls choosing a predictor from
the entire parameter-bounded class. A training procedure may select any
admissible parameter vector, so its generalization cost depends on the
capacity of that whole class.
Following the covering-based risk analysis of \citet{ching2026efficient},
we use a finite sup-norm net to reduce uniform control to finitely many
predictors. Since the variance of squared-loss differences relative to $g_0$ is
controlled by excess risk, Bernstein's inequality and a union bound over a net at resolution
$\asymp\Gamma^{-1}$ convert its log-cardinality into the
$(\mathfrak H_{n,\Gamma}^{\rm end}+1)/\Gamma$ near-ERM remainder
(Appendix~\ref{app:end-to-end-generalization}).

\begin{corollary}[Sufficient meta-sample regime]\label{cor:main-polynomial-meta-sample}
Choose $A_0\ge2\alpha_{\max}/(2\alpha_{\max}+1)$. If
\[
 \Gamma\gtrsim n^{c_{\rm ent}}\log(en)\mathfrak R_n^{-1},\qquad
 \varepsilon_{{\rm opt},\Gamma}\lesssim\mathfrak R_n,
\]
then $\mathbb E_{\mathcal D_\Gamma^{\rm tr}}
R_{P,\rho_f}^\star(\widehat g_\Gamma)\lesssim\mathfrak R_n$.
\end{corollary}
This gives a sufficient pretraining sample-complexity bound.
It also separates two uses of data. More independent training tasks can
reduce the cost of selecting an in-context procedure, whereas the new
function is still observed through only $n$ examples in its own prompt.
Corollary~\ref{cor:pretraining-minimax-lower} formalizes the distinction: it takes the
uniform prior on the heterogeneous Assouad family, draws a fresh function
index for each task, and conditions on all pretraining observations.
The target index and its prompt retain their original joint law, so the
same prior-average lower bound applies to every learned prediction rule.
The assertion allows knowledge of the hard task law while retaining
uncertainty about the new function. Optimality is therefore worst-case
over task distributions. This quantifier matters for interpreting the
benchmark: a fixed task distribution may be easier, as illustrated by a
point mass at the zero function, for which the identically zero predictor
has zero latent risk.

\section{Discussion}\label{sec:conclusion}
The benchmark in \eqref{eq:rate-Rn} connects heterogeneous local difficulty
to one geometry-first in-context estimator. A higher-order tangent fit,
stable reduced regression, and an architecture-compatible realization attain
this benchmark through a single forward pass. The resulting
near-ERM guarantee separates representation from statistical learning.

Our statistical formulation motivates extending transformer analysis to other
inference problems. \citet{martinez2026deep} establish convergence guarantees
for generative learning on heterogeneous stratified spaces. A further direction
is to derive sharp minimax rates for growing mixtures under suitable
distributional losses, through matching aggregate upper and lower bounds that
retain component dimensions, smoothness, masses, and effective sample sizes,
and to investigate whether
diffusion transformers attain these rates.

\label{main-text-end}

\ack{
Jaehee Seo was supported by the Next Generation Scholarship for Basic Studies (Type C) from
Seoul National University.}

\clearpage

\bibliographystyle{plainnat}
\bibliography{reference_cam_ready}

\clearpage

\appendix
\numberwithin{theorem}{section}
\numberwithin{lemma}{section}
\numberwithin{corollary}{section}
\numberwithin{definition}{section}
\numberwithin{proposition}{section}
\numberwithin{remark}{section}
\section*{Appendix}
\section{Construction overview}\label{app:roadmap}
\label{subsec:construction-overview}
Figure~\ref{fig:end-to-end} separates the fixed input/output interface from
computations inside the transformer. Its notation follows
Appendices~\ref{app:emb-and-read}--\ref{app:regression-solver}.
Displayed summaries are the targets of the approximate modules at query
$x=X_{n+1}$. Uniform softmax attention copies and averages numerical
registers; row-wise ReLU networks with residual connections compute the statistical modules.

\begin{figure}[H]
\centering
\begingroup
\setlength{\abovedisplayskip}{3pt}
\setlength{\belowdisplayskip}{3pt}
\setlength{\abovedisplayshortskip}{2pt}
\setlength{\belowdisplayshortskip}{2pt}
\begin{tikzpicture}[
  x=1pt,y=-1pt,
  >=Latex,
  every node/.style={font=\small},
  iface/.style={draw=black!70,line width=.65pt,rounded corners=3pt,
    fill=black!3,align=left,inner sep=6pt,outer sep=.5pt,
    text width=154pt,anchor=north},
  fixed/.style={iface,fill=black!8},
  inside/.style={iface,text width=174pt,fill=white},
  flow/.style={->,line width=.8pt,shorten <=2pt,shorten >=2pt},
  detail/.style={flow,dashed,draw=black!65},
  heading/.style={font=\small\bfseries,align=center,anchor=north,
    inner sep=0pt,text width=166pt}
]
\node[heading] at (83,0) {End-to-end interface};
\node[heading] at (294,0) {Inside $\mathrm{TF}_{\theta_n}$};

\node[fixed,minimum height=72pt] (emb) at (83,22) {
  \textbf{Fixed embedding $\mathrm{Emb}_n$}\\[4pt]
  $\mathfrak s=((X_i,Y_i)_{i=1}^n;X_{n+1})$\\[3pt]
  $H^{(0)}=\mathrm{Emb}_n(\mathfrak s)$\\[4pt]
  The observed query is $X_{n+1}$.
};
\node[iface,minimum height=106pt] (tokens) at (83,116) {
  \textbf{Initial token sequence}\\[4pt]
  $H^{(0)}\in\mathbb R^{(n+1)\times d_{E,n}}$\\[4pt]
  Sample row $i\le n$: $(X_i,Y_i)$.\\[2pt]
  Query row $n+1$: $(X_{n+1},0)$.\\[4pt]
  Row markers and zero-initialized workspace \emph{coordinates}.\\[2pt]
  Workspace lies within these $n+1$ rows.
};
\node[iface,minimum height=79pt] (final) at (83,282) {
  \textbf{Final transformer state}\\[4pt]
  $Z=\mathrm{TF}_{\theta_n}(H^{(0)})$\\[3pt]
  $Z_q:=Z_{n+1,\cdot}$\\[4pt]
  Reserved query coordinates contain chart weights $w_j$ and branch values $v_j$.
};
\node[fixed,minimum height=102pt] (read) at (83,383) {
  \textbf{Fixed scalar readout}\\[4pt]
  $\widehat f_{\theta_n}^{\rm TF}(\mathfrak s)=\mathrm{Read}_n(Z)$\\[3pt]
  $\displaystyle S_n=\sum_{j\in[N_{\rm ch}]}M_n(\bar w_j,\bar v_j)$\\[3pt]
  $\mathrm{Read}_n(Z)=(-L_{\mathcal F})\lor S_n\land L_{\mathcal F}$\\[4pt]
  Bars denote clipped coordinates. Fixed ReLU $M_n(w,v)$ approximates $wv$.
};

\node[inside,fill=black!3,minimum height=72pt] (struct) at (294,22) {
  \textbf{Encoded structural information}\\[4pt]
  Selector, tangent grid, chart atlas, and localization templates are fixed in the comparator parameters.\\[3pt]
  $r_\pi(x)=\pi_{k_x,n}^{-1/2}$.
};
\node[inside,minimum height=127pt] (geo) at (294,116) {
  \textbf{Geometric preconditioner}\\[4pt]
  From the query and covariate prompt:\\[3pt]
  $\displaystyle b_n^\circ(x)=
     \bigl(e_{d_x},e_{p_x},e_{s_x},h_x,r_\pi(x)\bigr)$\\[5pt]
  Localized moments and tangent fit:\\[2pt]
  $\bar P_x\ \longrightarrow\ P_{T,x}
       \ \longrightarrow\ \{\mathsf\omega_{j,x},U_{j,x}\}_j$\\[5pt]
  Reduced sample states:
  $\{\mathsf u_{i,j}(x),\mathsf W_{i,j}(x)\}_{i,j}$.\\[4pt]
  Response coordinates $Y_i$ are preserved.
};
\node[inside,minimum height=161pt] (reg) at (294,266) {
  \textbf{Chartwise regression solver}\\[4pt]
  For branch $j$, form the reduced features $\phi_{i,j}$ and normalized summaries:\\[4pt]
  $\displaystyle G_j=\frac1n\sum_{i\in[n]} r_\pi(x)^2\,
     \mathsf W_{i,j}\phi_{i,j}\phi_{i,j}^{\top}$\\[6pt]
  $\displaystyle g_j=\frac1n\sum_{i\in[n]} r_\pi(x)^2\,
     \mathsf W_{i,j}Y_i\phi_{i,j}$\\[6pt]
  $v_j=V_n(s_{x,j}^{\rm gen})
      \in[-L_{\mathcal F},L_{\mathcal F}]$.\\[5pt]
  A bounded ReLU decoder approximates the inverse-based intercept only on well-conditioned, positive-weight charts.
};
\node[align=left,inner sep=0pt,text width=178pt,anchor=north,
      font=\small] (scope) at (294,443) {
  \emph{Internal decomposition only.}\\[3pt]
  Tangent, chart, and regression states are computed within the transformer.
};

\draw[flow] (emb.south)--(tokens.north);
\draw[flow] (tokens.south)--node[left=3pt,font=\small,align=right]
  {$\mathrm{TF}_{\theta_n}$\\[-1pt]one forward pass}(final.north);
\draw[flow] (final.south)--(read.north);
\draw[flow] (struct.south)--(geo.north);
\draw[flow] (geo.south)--(reg.north);
\draw[detail] (tokens.east)--(geo.west);
\draw[detail] (reg.west)--(final.east);
\begin{scope}[on background layer]
  \node[draw=black!55,dashed,line width=.65pt,rounded corners=4pt,
        fit=(struct)(geo)(reg)(scope),inner sep=4pt] {};
\end{scope}
\end{tikzpicture}
\endgroup

\caption[End-to-end interface and internal computation]{End-to-end interface
and internal computation of the predictor in
Theorem~\ref{thm:transformer-approx}.
Solid arrows show successive transformations; dashed arrows associate the
initial and final token states with the internal modules.
The sequence has $n+1$ rows and polynomially many workspace coordinates.
The readout is a clipped scalar.}
\label{fig:end-to-end}
\label{fig:appendix-roadmap}
\end{figure}
\clearpage

\section{Extended related work}\label{app:related}
\paragraph{In-context statistical procedures.}
Viewing the prompt as a sample from a new task separates a transformer's
capacity to represent a statistical procedure from whether training finds
that procedure. Empirical work on simple function classes
\citep{garg2022can}, investigations of linear-model learning algorithms
\citep{akyurek2023what}, and constructive analyses of gradient-descent-like
updates and algorithm selection
\citep{von2023transformers,bai2023transformers} develop this perspective.
Our theorems use it for nonparametric tasks, with the optimization tolerance
kept separate from approximation and statistical error.

\paragraph{Euclidean and manifold nonparametric ICL.}
\citet{kim2024transformers} analyze nonparametric in-context estimation through
learned basis representations and distinguish pretraining from in-context
generalization. \citet{ching2026efficient} show that efficient transformer
implementations of local-polynomial regression can exploit arbitrary positive
H\"older smoothness in Euclidean covariates. Our analysis supplies usable intrinsic coordinates in an admissible growing
geometric environment and controls the interaction between coordinate
estimation and higher-order regression.

For related supervised regression on noisy manifolds,
\citet{shen2025transformers} analyze targets of the form
$g\circ\pi_{\mathcal M}$ with a single manifold and H\"older-continuous $g$
with exponent in $(0,1]$. Their discussion identifies projection approximation
as an obstacle to exploiting higher regularity. The structured-manifold ICL
model of \citet{shen2026understanding} assumes a fixed manifold, a uniform
covariate distribution, and H\"older exponents in $(0,1]$, and connects
attention to kernel prediction. In contrast, our benchmark explicitly retains
$(d_k,\alpha_k,n\pi_{k,n})$ for each component. A tangent fit through degree
$\lceil\alpha_k\rceil$ and a regression polynomial of degree
$\lceil\alpha_k\rceil-1$ yield the geometric propagation term
$h_k^2\mathbb E\delta_x^2$. These comparisons concern the stated regression settings, each with its own
perturbation model and structural assumptions. For Euclidean complexity,
\citet{ching2026efficient} provide sharper guarantees than our conservative
polynomial resource bounds.

\paragraph{Nonparametric regression and geometric inference.}
Local polynomial regression on manifold-supported data achieves rates driven
by intrinsic dimension \citep{bickel2007local}; local linear regression connects
this principle to tangent-space estimation \citep{cheng2013local}.
Deep estimators have also been analyzed on approximate manifolds
\citep{jiao2023deep}. Higher-order tangent and curvature estimation
\citep{Aamari2019Nonasymptotic}, and estimation of the reach
\citep{aamari2019estimating,aamari2023optimal}, motivate the separation between
local graph regularity and the global localization scale in our model.
The present construction uses these geometric ideas as part of a
prompt-dependent estimator.

\paragraph{Representation geometry and the modeling assumption.}
Image-dimension studies \citep{pope2021the} and evidence for a union-of-manifolds
viewpoint \citep{brown2023verifying} motivate heterogeneous low-dimensional
structure. Analyses of transformer embeddings document changes in anisotropy
and intrinsic dimension during training \citep{razzhigaev2024shape}.
The mixture represents this local heterogeneity. Bounded tubular
perturbations describe observations concentrated near an idealized latent
manifold, connecting the model to the approximate-manifold viewpoint
\citep{jiao2023deep,shen2025transformers}.
A fixed feature map producing $X$ and the affine token-packing map
$\mathrm{Emb}_n$ have different roles: only the latter is part of the formal
input interface. The covariate dimension $D$ is fixed, whereas the workspace
dimension $d_{E,n}$ of the constructed transformer may grow.

\paragraph{Quantitative implementation and generalization.}
ReLU approximation theory provides general tools for constructive networks
\citep{yarotsky2017error,petersen2018optimal}; program-style attention analyses
explain selection and aggregation \citep{weiss2021thinking}.
We give the required Lipschitz-preserving approximation and residual-block
compilation directly in Appendix~\ref{app:geo-primitive}, using polynomial
workspace for intermediate ReLU states. Existing covering analyses
\citep{edelman2022inductive,trauger2024sequence,havrilla2024understanding}
provide relevant context, but our entropy calculation is for the specific
bounded-parameter softmax class and polynomial workspace used here.
Lemma~\ref{lem:softmax-communication} implements exact copying and averaging
with uniform softmax weights and bounded value gates.

\clearpage

\section{Numerical illustration}\label{app:numerical}
We compare a trained softmax workspace predictor with a geometry-informed
local-polynomial benchmark at fixed context size and heterogeneous mixture
geometry, varying the number of pretraining prompts
$\Gamma\in\{100{,}000,200{,}000,500{,}000,1{,}000{,}000\}$. The experiment combines
heterogeneity in dimension, smoothness, sampling mass, shape, and design
density. We specify the data generator, empirical architecture, optimization,
bandwidth tuning, and evaluation protocol below.

\subsection{Data and task generation}
\paragraph{Heterogeneous mixture and geometry.}
We fix ambient dimension $D=64$, context size $n=40{,}000$, and $K=12$
components, one for each pair
$(d_k,\alpha_k)\in\{2,4,6,8\}\times\{0.75,1.5,3\}$.
The component probabilities and expected context counts are
\[
\begin{array}{c|ccc}
 &\alpha=0.75&\alpha=1.5&\alpha=3\\ \hline
 d=2&0.02\;(800)&0.03\;(1{,}200)&0.05\;(2{,}000)\\
 d=4&0.04\;(1{,}600)&0.06\;(2{,}400)&0.10\;(4{,}000)\\
 d=6&0.06\;(2{,}400)&0.09\;(3{,}600)&0.15\;(6{,}000)\\
 d=8&0.08\;(3{,}200)&0.12\;(4{,}800)&0.20\;(8{,}000).
\end{array}
\]
For indices $i=0,\ldots,3$ and $j=0,1,2$, write $k=3i+j+1$,
and assign anisotropy and density tilt by
$\eta_k=(0,0.15,0.30)_{(i+j)\bmod3}$ and
$\lambda_k=(0,0.4,0.7)_{(i+2j)\bmod3}$, with tuple entries indexed from zero.
Let $a_d=\operatorname{vol}(\mathbb S^d)$ denote the area of
$\mathbb S^d$, and define
\[
 R_d=(4/a_d)^{1/d},\qquad
 A_k=\operatorname{diag}\bigl(e^{-\eta_k+2\eta_k\ell/d_k}\bigr)_{\ell=0}^{d_k},
 \qquad
 \mathcal M_k=c_k+R_{d_k}Q_kA_k\mathbb S^{d_k}.
\]
The frames $Q_k\in\mathbb R^{64\times(d_k+1)}$ have orthonormal columns.
With $r_{\max}=\max_k R_{d_k}e^{\eta_k}$ and $s=2r_{\max}+1$, the
centers are the rows of
$(s/\sqrt2)(I_{12}-\boldsymbol1\boldsymbol1^\top/12)V^\top$,
where $V\in\mathbb R^{64\times12}$ also has orthonormal columns.
Thus the centers form a regular simplex of side $s$.
Gaussian reduced QR factorizations, with signs fixed by the diagonal of
the triangular factor, generate $V$ and $Q_k$ using NumPy
\texttt{SeedSequence} inputs $(2026,0)$ and $(2026,1,k-1)$, respectively.
This geometry is fixed across tasks and training budgets.

For component $k$, sample $U\in\mathbb S^{d_k}$ with density
$1+\lambda_k u_1$ relative to the uniform probability law, using
rejection sampling from normalized Gaussian directions with envelope
$1+\lambda_k$, and set $X^\star=c_k+R_{d_k}Q_kA_kU$.
Since $\det A_k=1$ and $R_{d_k}^{d_k}a_{d_k}=4$, the density relative
to manifold volume is
\[
 \varrho_k(c_k+R_{d_k}Q_kA_ku)
 =\frac{1+\lambda_k u_1}{4\|A_k^{-T}u\|},\qquad
 4e^{-\eta_k}\le\operatorname{vol}(\mathcal M_k)\le4e^{\eta_k}.
\]
Across the twelve components, $0.05556<\varrho_k<0.57370$, so
$c_\mu=1/20$ is a valid common density constant.
The reach is $R_{d_k}e^{-3\eta_k}$, with minimum $0.22938$;
the conservative separation lower bound is $1.10143$, and every latent
point has Euclidean norm at most $3.15035<4$.

\paragraph{Task functions.}
Each task independently draws one function per component and shares it
across all context and query points on that component. In sphere
coordinates, it has the form
\[
 \begin{aligned}
 f_k(c_k+R_{d_k}Q_kA_ku)
 &=b_k+\sum_{\ell\in[32]}a_{k\ell}
   \phi_{\alpha_k}\!\left(4(v_{k\ell}^\top u-t_{k\ell})\right)\\
 &\quad+\sum_{\ell\in[8]}c_{k\ell}\psi_{k\ell}(u),\qquad
 \phi_\alpha(z)=\frac{(z_+)^\alpha}{1+(z_+)^\alpha}.
 \end{aligned}
\]
Here $b_k\sim\operatorname{Unif}[-0.15,0.15]$,
$t_{k\ell}\sim\operatorname{Unif}[-0.35,0.35]$, and the ridge
directions $v_{k\ell}$ are uniform on $\mathbb S^{d_k}$.
Independently draw bump centers $w_{k\ell}$ uniformly on
$\mathbb S^{d_k}$ and widths $q_{k\ell}\sim\operatorname{Unif}[1,1.6]$;
with $z_{k\ell}(u)=\|u-w_{k\ell}\|^2/q_{k\ell}^2$, set
\[
 \psi_{k\ell}(u)=
 \begin{cases}
 \exp\{1-(1-z_{k\ell}(u))^{-1}\},&z_{k\ell}(u)<1,\\
 0,&z_{k\ell}(u)\ge1.
 \end{cases}
\]
The ridge and bump coefficient vectors have absolute sums $0.4$ and
$0.2$, respectively. For each vector, the first two coefficients receive
$45\%$ and $25\%$ of its absolute mass; the remaining $30\%$ is
allocated in proportion to independent $\operatorname{Unif}[0.5,1.5]$
draws. All coefficient signs are independent and equiprobable.
Consequently $|f_k|\le0.15+0.4+0.2=0.75$ deterministically.
The ridge knots yield the smoothness classes $C^{0,0.75}$,
$C^{1,0.5}$, and $C^{2,1}$ for the three values of $\alpha_k$,
while the bumps are smooth. On the fixed ellipsoid family these
functions admit a common finite H\"older bound, independent of the task
draw and training budget; the value bound $0.75$ is distinct from that
H\"older norm bound.

\paragraph{Observations and reproducibility.}
Each prompt draws $n+1$ independent labels from the mixture and latent
points from their corresponding component laws. Write $N_k=n\pi_{k,n}$
and set
\[
 h_k=N_k^{-1/(2\alpha_k+d_k)},\qquad
 \sigma_{k,n}=\frac{h_k^{\alpha_k\vee1}}{n},\qquad
 X=X^\star+\xi,\qquad Y=f_k(X^\star)+\epsilon.
\]
Conditionally on the component, $\xi$ is uniform on the ambient
$64$-dimensional Euclidean ball of radius $\sigma_{k,n}$:
$\xi=\sigma_{k,n}B^{1/64}G/\|G\|$ for independent
$B\sim\operatorname{Unif}[0,1]$ and $G\sim N(0,I_{64})$.
The independent response noise is
$\epsilon\sim\operatorname{Unif}[-0.1\sqrt3,0.1\sqrt3]$, with
conditional mean zero and variance $0.01$.
Evaluation targets the latent query value $f_k(X^\star)$.

Task randomness uses NumPy's \texttt{SeedSequence} with the ordered input
\[
 (23619,\text{split},\text{task index},\text{stream},\text{component}).
\]
Split identifiers $0,1,2$ denote training, validation, and testing;
stream identifiers $0,1,2,3,4$ generate labels, latent points, task
functions, covariate noise, and response noise, respectively. The
component entry is $k-1$ for component-specific streams and zero for
the label and response streams.
The $2{,}000$ test prompts have indices $3{,}000{,}017$ through
$3{,}002{,}016$ and are shared across methods and budgets.
The fixed bounded geometry and shrinking perturbation rule reflect the
asymptotic model; the full finite-sample constant conditions of
Assumption~\ref{ass:feasible-local} have not been certified at
$n=40{,}000$.

\subsection{Empirical architecture and optimization}
\paragraph{Softmax workspace predictor.}
The learned predictor has width $64$, four workspace blocks, four attention
heads per module, $32$ learned workspace tokens, and FFN width $256$.
Each block applies workspace-to-context cross-attention, workspace
self-attention, and a GELU FFN, with residual connections. Layer
normalization and dropout are disabled. For the observed query
$x=X_{n+1}$, the implemented preprocessing is
\[
 \widetilde X_i=X_i/6,\qquad \widetilde x=x/6,\qquad
 r_i=\widetilde X_i-\widetilde x,\qquad t_i=\|r_i\|_2^2.
\]
A learned affine--GELU--affine encoder maps the $130$-dimensional vector
$(\widetilde X_i,r_i,Y_i,t_i)$ through width $64$ to a context token.
A separate encoder of the same form maps $\widetilde x\in\mathbb R^{64}$
to a width-$64$ vector added to every learned workspace token.
Each attention head uses scaled dot-product softmax weights
$\operatorname{softmax}(q_a^\top k_j/\sqrt{16})$.
The context tokens remain fixed across the four blocks; attention from
the $32$ workspace tokens accesses all $40{,}000$ context tokens through
a $32\times40{,}000$ score matrix per head.
An affine readout of the first final workspace token followed by $\tanh$
produces a prediction in $[-1,1]$.
The model has $286{,}401$ trainable parameters. Its learned encoders and
GELU blocks define the empirical model studied here. The theoretical
construction uses the residual-ReLU architecture in
Section~\ref{subsec:transformers}.

\paragraph{Training and checkpoint selection.}
For each of the four budgets, we train the model from scratch with
training seed $0$.
Training traverses the indexed task pool
$\{0,\ldots,\Gamma-1\}$ exactly once, using a permutation with seed $101$.
Prompts are generated on demand. Gradient accumulation over $25$
microbatches of one prompt gives an effective batch size of $25$; at most
one additional microbatch is prefetched. The completed runs have
\[
\begin{array}{c|rrrr}
 \Gamma&100{,}000&200{,}000&500{,}000&1{,}000{,}000\\
 \text{optimizer updates}&4{,}000&8{,}000&20{,}000&40{,}000\\
 \text{final effective batch size}&25&25&25&25
\end{array}
\]
AdamW uses initial learning rate $3\times10^{-4}$ and weight decay
$10^{-4}$. The learning rate follows cosine annealing to zero over each
run's prescribed update budget, and the accumulated gradient norm is
clipped at $1$. Training minimizes squared error against the noisy query
response $Y_{n+1}$, and validation evaluates the same loss using the $96$ fixed
validation-split tasks with indices $0$--$95$, evaluated after the first
update, every $250$ updates, and at the final update. All four runs
complete the full schedule. Test evaluation uses the final \texttt{last.pt}
checkpoint after all prescribed updates. Training and transformer inference
use float32 on the PyTorch MPS backend. Increasing $\Gamma$
therefore changes both the number of distinct training prompts and the
amount of optimization, with the cosine schedule stretched accordingly.

\subsection{Geometry-informed local-polynomial benchmark}

LP-CV receives the query component's dimension $d_k$, smoothness
$\alpha_k$, mixture mass $\pi_{k,n}$, and true tangent basis $U_\star$ at the
latent query $X_{n+1}^\star$. For a candidate multiplier $c$, set
$b=c(n\pi_{k,n})^{-1/(2\alpha_k+d_k)}$ and
$u_i=U_\star^\top(X_i-X_{n+1})/b$.
The total polynomial degree is
$p_k=\lceil\alpha_k\rceil-1\in\{0,1,2\}$, with
$q_k=\binom{d_k+p_k}{p_k}$ features
$\phi_\nu(u)=u^\nu/\nu!$, $|\nu|\le p_k$.
The weighted fit uses all observed context responses, with
\[
 w_i=\mathsf A\!\left(\frac{\|X_i-X_{n+1}\|_1}{b}\right)
          (1-\|u_i\|_1)_+,
 \qquad
 \mathsf A(t)=1\land\{0\lor(2-2t)\}.
\]
Thus both ambient and tangent-coordinate cutoffs use the $\ell_1$ norm,
and the ambient cutoff scale is $1$. Common positive factors such as
$b^{-d_k}/n$ cancel from the normal equations. Before solving, positive
weights are divided by their maximum.
The implementation uses float64 normal equations and an eigendecomposition,
retaining Gram eigenvalues greater than $10^{-8}$ times the largest
eigenvalue. The resulting pseudoinverse fit supplies the intercept,
clipped to $[-1,1]$. If no point has positive weight, the fallback is the
clipped response of the nearest observed context point in Euclidean
distance. Context component labels supply contamination diagnostics;
all context points enter the weight rule and the nearest-point fallback.
This benchmark uses the true latent-query tangent while centering at the
observed query; it differs from the estimated-tangent construction in
Algorithm~\ref{alg:localpoly}.

Bandwidth selection uses the common grid
$\{2,4,6,8,10,12,16,24\}$ separately for all $12$ $(d_k,\alpha_k)$
strata. For stratum $j=0,\ldots,11$, ordered by dimension and then
smoothness, $100$ validation prompts have task indices
$2{,}000{,}003+100j$ through $2{,}000{,}102+100j$, with the query
component fixed to that stratum. These $1{,}200$ tuning prompts are
disjoint from the $96$ training-validation prompts and the test split.
A multiplier is admissible if its mean positive-weight count is at least
$2q_k$ and its Gram matrix has full numerical rank on at least $90\%$
of the stratum's tuning prompts. Among admissible candidates, LP-CV
minimizes mean squared error against the independent noisy query
response, breaking ties toward the smaller multiplier.
Every stratum has an admissible candidate;
the selected multipliers, fixed for all four model comparisons, are
\[
\begin{array}{c|rrr}
 &\alpha=0.75&\alpha=1.5&\alpha=3\\\hline
 d=2&16&10&6\\
 d=4&10&8&6\\
 d=6&10&12&10\\
 d=8&12&10&16
\end{array}
\]

\subsection{Test protocol and uncertainty}

All four final checkpoints and LP-CV use the same $2{,}000$ test-split
prompts, indexed $3{,}000{,}017$ through $3{,}002{,}016$.
Test query components follow the
specified mixture probabilities. Transformer inference processes one
prompt at a time. The test target is the latent regression value
$f(X_{n+1}^\star)$; training and tuning use the noisy query response.
The saved per-prompt
input/target hashes agree across all four evaluations and the LP-CV
reference; the error arithmetic and intervals have been checked by
recomputing them from the saved predictions.

Reported $95\%$ intervals are nonparametric percentile-bootstrap
intervals using $5{,}000$ replicates. Each replicate samples $2{,}000$
test-prompt indices with replacement; its mean squared error gives the
bootstrap risk, and the $2.5\%$ and $97.5\%$ quantiles give the interval.
The transformer and LP-CV risk intervals use seeds $3{,}000{,}017$ and
$3{,}000{,}018$, respectively. For the transformer/LP-CV risk ratio,
seed $3{,}000{,}019$ generates paired resamples: the same indices are
used for the numerator and denominator, and each replicate takes the
ratio of their mean squared errors. These intervals describe test-task
variation conditional on the fitted models and selected LP bandwidths;
they do not measure training-seed or bandwidth-selection uncertainty.

\subsection{Results and interpretation}

\begin{table}[ht]
\centering
\caption{Latent-target test MSE in units of $10^{-3}$ and risk relative to
geometry-informed LP-CV on the same $2{,}000$ prompts. Brackets give $95\%$
percentile-bootstrap intervals; risk-ratio intervals use paired resampling.
LP-CV has MSE $0.690\,[0.645,0.739]$ in the same units at every budget.}
\label{tab:numerical-budget}
\begin{tabular}{rcc}
\toprule
Pretraining budget $\Gamma$ & Softmax workspace MSE & Transformer / LP-CV\\
\midrule
$100{,}000$ & $14.455$ & $20.948$\\
 & $[13.619,15.342]$ & $[19.177,22.776]$\\
$200{,}000$ & $0.919$ & $1.332$\\
 & $[0.849,0.995]$ & $[1.222,1.461]$\\
$500{,}000$ & $0.538$ & $0.780$\\
 & $[0.503,0.575]$ & $[0.727,0.837]$\\
$1{,}000{,}000$ & $0.505$ & $0.733$\\
 & $[0.472,0.540]$ & $[0.686,0.782]$\\
\bottomrule
\end{tabular}
\end{table}

The test MSE decreases across the four budgets. At $\Gamma=500{,}000$,
the transformer has $21.98\%$ lower MSE than LP-CV and $41.45\%$ lower MSE
than the $\Gamma=200{,}000$ model. Its paired risk-ratio interval
$[0.727,0.837]$ lies below one. At $\Gamma=1{,}000{,}000$, the MSE is
$0.505\times10^{-3}$, which is $26.75\%$ lower than LP-CV and $6.10\%$
lower than the $\Gamma=500{,}000$ model. The paired risk ratio is $0.733$
with interval $[0.686,0.782]$, again below one.
These comparisons quantify test-prompt
uncertainty for the fitted seed-$0$ models.

\begin{table}[ht]
\centering
\caption{Component-wise point estimates at $\Gamma=1{,}000{,}000$. MSEs are in units
of $10^{-3}$; $n_{\rm test}$ is the number of queries from that component
among the $2{,}000$ shared test prompts. Ratios use unrounded MSEs.}
\label{tab:numerical-components}
\begin{tabular}{rrrrrr}
\toprule
$d_k$ & $\alpha_k$ & $n_{\rm test}$ & LP-CV MSE & Transformer MSE & Ratio\\
\midrule
2 & 0.75 & 41 & 0.606 & 0.448 & 0.739\\
2 & 1.5 & 60 & 0.376 & 0.310 & 0.825\\
2 & 3 & 97 & 0.545 & 0.480 & 0.880\\
\addlinespace[2pt]
4 & 0.75 & 72 & 0.435 & 0.371 & 0.853\\
4 & 1.5 & 106 & 0.632 & 0.538 & 0.851\\
4 & 3 & 217 & 0.845 & 0.418 & 0.495\\
\addlinespace[2pt]
6 & 0.75 & 117 & 0.692 & 0.474 & 0.684\\
6 & 1.5 & 175 & 0.882 & 0.404 & 0.458\\
6 & 3 & 317 & 0.605 & 0.546 & 0.903\\
\addlinespace[2pt]
8 & 0.75 & 159 & 0.667 & 0.515 & 0.772\\
8 & 1.5 & 252 & 0.648 & 0.433 & 0.667\\
8 & 3 & 387 & 0.780 & 0.679 & 0.871\\
\bottomrule
\end{tabular}
\end{table}

Table~\ref{tab:numerical-components} gives the component means.
At $\Gamma=1{,}000{,}000$, the transformer has a lower MSE point estimate
than LP-CV in all twelve components. These component-wise comparisons are
descriptive: query counts range from $41$ to $387$, and no component-wise
confidence intervals are reported.

This comparison measures the combined effect of increasing the number of
distinct pretraining prompts and the optimizer update budget. It illustrates
learned prediction at a fixed context size and heterogeneous geometry.
The theoretical guarantees concern the constructed residual-ReLU
architecture and its asymptotic scaling; the empirical results concern
the specified workspace architecture and task distribution.

Code is available at \url{https://github.com/seojaehee02/nonparametric_ICL}.

\clearpage

\section{Mathematical preliminaries}
\label{sec:preliminaries}\label{app:preliminaries}

This section fixes the notation and basic objects used throughout the paper.
We collect conventions for asymptotic notation, polynomial features, matrix norms, localization, and the H\"older and manifold regularity classes. The transformer class is defined in Section~\ref{subsec:transformers}.

\subsection{Notations}\label{subsec:notations}
For an event \(A\), \(\mathbb{I}\{A\}\), also written \(\mathbb{I}_A\),
denotes its indicator. The vector \(\mathbf 1_N\in\mathbb R^N\) has all entries equal to one.
Let \(\mathbb N_0:=\mathbb N\cup\{0\}\). Write \([m]:=\{1,\ldots,m\}\)
for \(m\in\mathbb N\), and \([m]_0:=\{0,\ldots,m\}\) for \(m\in\mathbb N_0\).
For \(a\in\mathbb R\), let
\(\lfloor a\rfloor\) and \(\lceil a\rceil\) denote the floor and ceiling of
\(a\). For \(a,b\in\mathbb R\), write
\(a\wedge b:=\min\{a,b\}\) and \(a\vee b:=\max\{a,b\}\). For non-negative
sequences \(a_n,b_n\), write \(a_n\lesssim b_n\) if
\(a_n\le Cb_n\) for a constant \(C>0\) independent of \(n\). Write
\(a_n\gtrsim b_n\) if \(b_n\lesssim a_n\), and
\(a_n\asymp b_n\) if both \(a_n\lesssim b_n\) and \(a_n\gtrsim b_n\) hold. For a metric space \((\mathcal X,d)\), \(x\in\mathcal X\), and \(r>0\), let \(B_{\mathcal X}(x,r):=\{x'\in\mathcal X:d(x,x')<r\}\). For a set \(A\subset\mathcal X\), write
\(d(x,A):=\inf_{a\in A}d(x,a)\) and \(d(A,B):=\inf_{x\in A}d(x,B)\). Here \(d(A,B)\) denotes the separation between \(A\) and \(B\).

\paragraph{Constants.}\label{app:constant-convention}
Throughout the Appendix, \(c,C>0\) denote generic constants whose values may
change between occurrences. Unless a different dependence is specified,
they depend only on the fixed model bounds in Section~\ref{subsec:assumption},
\(\kappa\), and fixed kernel, cutoff, and construction choices. All bounds
are uniform in \(n,K_n,k,P,f\) and admissible component sequences; varying
masses, reach, and separation enter through their displayed scales.
Additional dependence on fixed parameters of auxiliary statements or on
approximation and architecture parameters is stated locally. Named thresholds
and exponents retain their fixed values throughout each argument.

\smallskip
For a multi-index \(\nu=(\nu_1,\ldots,\nu_q)\in\mathbb N_0^q\), define
\(
|\nu|:=\sum_{j\in[q]}\nu_j\),
\(\nu!:=\prod_{j\in[q]}\nu_j!\), and \(u^\nu:=\prod_{j\in[q]}u_j^{\nu_j}\).
For \(p\ge0\), set
\(
\mathcal I_{q,p}:=\{\nu\in\mathbb N_0^q:|\nu|\le p\}\), \(q_{q,p}:=|\mathcal I_{q,p}|,
\)
and
\(
\Phi_{q,p}(u):=\left({u^\nu\over\nu!}\right)_{\nu\in\mathcal I_{q,p}}\).
The zero-multi-index coordinate vector is denoted by \(e_{0,q,p}\), or simply
\(e_0\).

\smallskip
For a linear map \(A\), define
\(
\|A\|_{\rm op}:=\sup_{\|v\|=1}\|Av\|.
\)
For a matrix \(A\), \(\|A\|_F\) denotes its Frobenius norm and 
\(\|A\|_{\max}:=\max_{i,j}|A_{ij}|\). For a symmetric
positive semidefinite matrix \(G\), let \(\lambda_{\min}^{+}(G)\) and
\(\lambda_{\max}^{+}(G)\) denote its smallest and largest positive eigenvalues.
We write \(G^\dagger\) for the Moore--Penrose pseudoinverse and
\(\operatorname{Im}(G)\) for the image of \(G\). For symmetric matrices,
\(A\preceq B\) means that \(B-A\) is positive semidefinite. 

\smallskip
We use the kernel \(\mathsf K(v):=(1-\|v\|_1)_+\). The cutoff \(\mathsf A\) is the
piecewise-linear function given by \(\mathsf A(t)=1\) for \(0\le t<1/2\),
\(\mathsf A(t)=2-2t\) for \(1/2\le t<1\), and \(\mathsf A(t)=0\) for \(t\ge1\).

\subsection{H\"older classes and manifolds}

\paragraph{H\"older class on Euclidean domains.}
Let \(U\subset\mathbb R^D\) be open. For \(\gamma\in(0,1]\), a function
\(g:U\to\mathbb R\) is called \emph{\(\gamma\)-H\"older continuous} if
\[
    [g]_{C^{0,\gamma}(U)}
    :=
    \sup_{\substack{u,v\in U\\u\ne v}}
    { |g(u)-g(v)|\over \|u-v\|^\gamma}
    <\infty .
\]
Let \(s\in\mathbb N_0\). A scalar-valued map \(g:U\to\mathbb R\) is called
\(C^{s,\gamma}\) if it has continuous derivatives up to order \(s\), all
derivatives of order at most \(s\) are bounded, and every derivative of order
\(s\) is \(\gamma\)-H\"older continuous. Equivalently, its
\(C^{s,\gamma}\)-norm
\[
\|g\|_{C^{s,\gamma}(U)}
:=
\max_{|\mu|\le s}\|D^\mu g\|_\infty
+
\max_{|\mu|=s}
[D^\mu g]_{C^{0,\gamma}(U)}
\]
is finite, where
\(
[D^\mu g]_{C^{0,\gamma}(U)}
:=
\sup_{\substack{u,v\in U\\u\ne v}}
\frac{|D^\mu g(u)-D^\mu g(v)|}{\|u-v\|^\gamma}.
\)
For vector-valued maps, the same definition is used component-wise, equivalently
with Euclidean norms and multilinear operator norms. For \(\alpha>0\), put
\(
s_\alpha:=\lceil \alpha\rceil-1\) and \(\gamma_\alpha:=\alpha-s_\alpha\in(0,1].
\)
We define the \(\alpha\)-H\"older class by
\(\mathcal H^{\alpha}(U):=C^{s_{\alpha},\gamma_{\alpha}}(U)\) with \(\|g\|_{\mathcal H^{\alpha}(U)}:=\|g\|_{C^{s_{\alpha},\gamma_{\alpha}}(U)}\).

\paragraph{Manifolds.}

A \(d\)-dimensional embedded \(C^{s_\beta,\gamma_\beta}\) submanifold of
\(\mathbb R^D\) is a set that can be represented locally, after a rigid change
of coordinates, as the graph of a \(C^{s_\beta,\gamma_\beta}\) map from
\(\mathbb R^d\) to \(\mathbb R^{D-d}\). We denote the tangent space and the normal space at \(x\in\mathcal M\) as \(T_{x}\mathcal M\) and \(T_{x}\mathcal M^\bot\). \(\Pi_{T_x\mathcal M}\) is the nearest-point projection map onto \(T_x\mathcal M\). Also, let \(d_{\mathcal M}\) denote the intrinsic length metric on \(\mathcal M\). Standard terminology about reach is in the sense of
\cite{federer1959curvature}.

We recall the reach of a closed set \(\mathcal M\), first introduced in \citet{federer1959curvature}. The medial axis of \(\mathcal M\) is defined as \[Med(\mathcal M):=\left\{y\in\mathbb R^D:\exists x_1\neq x_2\in \mathcal M \quad  \text{s.t.} \quad \|x_1-y\|=\|x_2-y\|=d(y,\mathcal M)\right\}.\]
Thereby, the medial axis of \(A\) is the set of points whose nearest-point projection onto \(\mathcal M\) is not uniquely defined. The \emph{reach} of \(\mathcal M\) is defined as
\[
\mathrm{reach}(\mathcal M)
:=\tau_\mathcal M=
\inf_{x\in \mathcal M,\ y\in Med(\mathcal M)}\|x-y\|.
\]
Figure~\ref{fig:app-reach-illustration} illustrates the medial axis and the reach of the closed set.

We now define the class of regular manifolds.

\begin{figure}[t]
    \centering
    \includegraphics[width=0.5\linewidth]{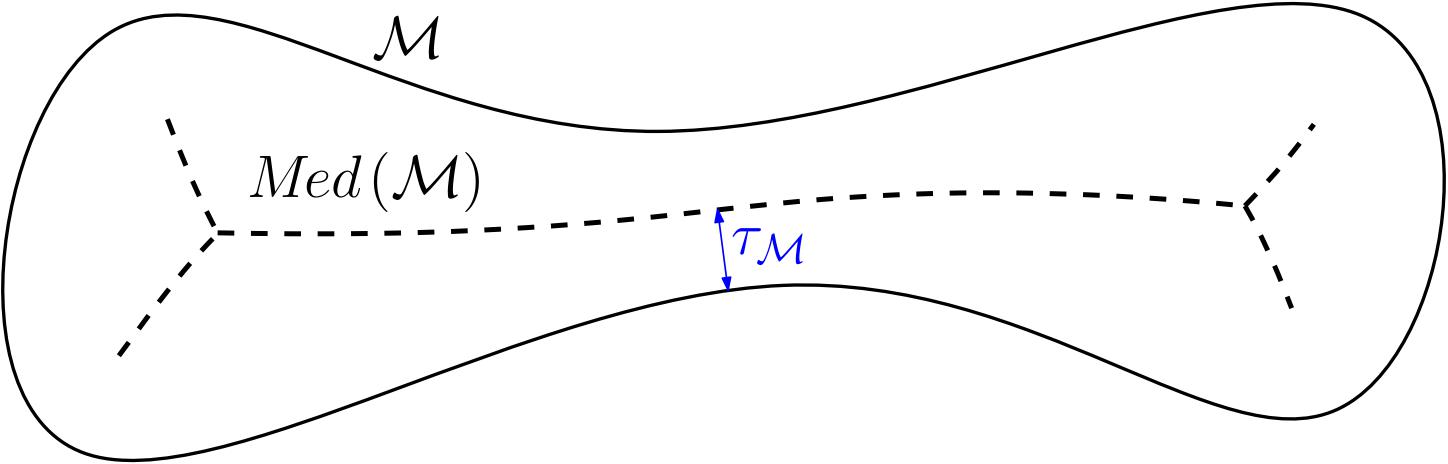}
    \caption{The medial axis $Med(\mathcal M)$ consists of points
    having multiple nearest points on $\mathcal M$.  The shortest distance from
    $\mathcal M$ to this axis is $\operatorname{reach}(\mathcal M)$.  A narrow bottleneck brings the
    medial axis closer to $\mathcal M$ and therefore reduces the reach.}
    \label{fig:app-reach-illustration}
\end{figure}

\begin{definition}[Regular embedded manifold class]
\label{def:regular-manifold-class}
Fix \(D\ge 2\). For \(d\in[D-1]\), \(\beta\ge 2\), \(L\ge 1\), and \(\tau>0\), let
\(r_L:=1/(4L)\). We write
\(\mathcal M\in\mathcal C^\beta_{d,L,\tau}(\mathbb R^D)\) if
\(\mathcal M\subset\mathbb R^D\) is a closed, connected, embedded \(d\)-dimensional
\(C^{s_\beta,\gamma_\beta}\) submanifold without boundary satisfying
\[
\operatorname{reach}(\mathcal M)\ge\tau,
\]
such that, for every
\(x\in \mathcal M\), there is a tangent-normal chart
\[
\Psi_x(v)=x+v+N_x(v),
\qquad
v\in B_{T_x\mathcal M}(0,r_L),
\]
parametrizing a relatively open neighborhood of \(x\) in \(\mathcal M\). Here
\[
N_x:B_{T_x\mathcal M}(0,r_L)\to T_x\mathcal M^\perp,
\qquad
N_x(0)=0,\qquad D_0N_x=0.
\]
The graph maps have uniformly bounded higher-order regularity:
\[
\max_{2\le j\le s_\beta}
\sup_{v\in B_{T_x\mathcal M}(0,r_L)}
\|D^jN_x(v)\|_{\rm op}
\le L,
\qquad
[D^{s_\beta}N_x]_{C^{0,\gamma_\beta}(B_{T_x\mathcal M}(0,r_L))}
\le L,
\]
where the first maximum is omitted when \(s_\beta=1\).
The chart radius \(r_L\) and regularity bound \(L\) do not depend on \(\tau\).
\end{definition}
Note that for a given \(\mathcal M\in\mathcal C_{d,L,\tau}^2(\mathbb R^D)\), 
\(\mathrm{reach}(\mathcal M)=\inf_{\substack{x,x'\in\mathcal M\\x\ne x'}}\frac{\|x-x'\|^2}{2d(x', x+T_x\mathcal M)}\) holds (see \citep[Theorem 4.18]{federer1959curvature}). Positive reach is a standard regularity condition in geometric inference \citep{aamari2019estimating, Aamari2019Nonasymptotic, aamari2023optimal}: Classical formulations often impose a fixed reach lower bound. The following lemma characterizes the local Euclidean behavior with positive reach. The proof is given in Appendix~\ref{app:manifold}.

\begin{lemma}
\label{lem:reach-local-facts-main}
Let \(\mathcal M\in\mathcal C^\beta_{d,L,\tau}(\mathbb R^D)\), where \(\tau>0\).
There exists \(c_{\rm geo}\in(0,1)\), depending only on \(D,d,L,\beta\), such that
the following holds.  For every \(x\in\mathcal M\), set
\(\rho:=c_{\rm geo}(\tau\wedge r_L)\).
Then the following local estimates hold.

\begin{enumerate}
\item[(i)]
For every \(r\in(0,\rho]\),
\(\operatorname{vol}_{\mathcal M}
    \bigl(\mathcal M\cap B_{\mathbb R^D}(x,r)\bigr)\asymp r^d\),
with comparison constants depending only on \(d\).

\item[(ii)]
For all \(y,z\in \mathcal M\cap B_{\mathbb R^D}(x,\rho)\),
\(\|y-z\|
    \le
    d_{\mathcal M}(y,z)
    \le 2\|y-z\|\).

\item[(iii)]
For all \(y\in \mathcal M\cap B_{\mathbb R^D}(x,\rho)\),
\(\|\Pi_{T_y\mathcal M}-\Pi_{T_x\mathcal M}\|_{\rm op}
    \le
    L\|y-x\|\).
\end{enumerate}
\end{lemma}

\paragraph{H\"older class on manifolds.}

We next recall H\"older class defined on regular embedded manifolds by pulling them back to tangent-normal charts. Since the admissible
charts have uniform \(C^\beta\) bounds and fixed radius, this gives the same smoothness class, up to equivalent norms, as any other
atlas with comparable uniform regularity constants. Thus the definition is stable under admissible choices of uniformly regular
charts.

\begin{definition}[H\"older functions on regular manifolds]
\label{def:holder-manifold}
Let \(\mathcal M\in\mathcal C^\beta_{d,L,\tau}(\mathbb R^D)\), and let
\(\Psi_x:B_{T_x\mathcal M}(0,r_L)\to \mathcal M\) be the tangent-normal chart from
Definition~\ref{def:regular-manifold-class}. For \(\alpha>0\) and
\(L_{\mathcal F}>0\) with \(\beta\ge \lceil \alpha\rceil+1\), define
\[
\mathcal H^\alpha(\mathcal M;L_{\mathcal F})
:=
\bigg\{
f:\mathcal M\to\mathbb R:
\sup_{x\in \mathcal M}
\|f\circ\Psi_x\|_{\mathcal H^\alpha(B_{T_x\mathcal M}(0,r_L))}
\le L_{\mathcal F}
\bigg\}.
\]
Here \(T_x\mathcal M\) is regarded as a Euclidean space with the inherited inner product.
\end{definition}

\section{Auxiliary details}
\label{app:geometric-assumption-details}

This appendix collects technical details for the geometric assumptions used in
the main text. We first prove the local consequences of positive reach stated
in Lemma~\ref{lem:reach-local-facts-main}. We then give a constant-explicit
version of Assumption~\ref{ass:feasible-local} and record its elementary
consequences for effective sample sizes, local bandwidths, and the growth of
\(K_n\).

\subsection{Proof of Lemma~\ref{lem:reach-local-facts-main}}
\label{app:manifold}

\begin{proof}
Class membership gives \(\operatorname{reach}(\mathcal M)\ge\tau\).
Fix \(x\in\mathcal M\), and let \(r_L:=1/(4L)\). Let
\[
    \Psi_x(u)=x+u+N_x(u),
    \qquad
    u\in B_{T_x\mathcal M}(0,r_L),
\]
be the tangent-normal chart from
Definition~\ref{def:regular-manifold-class}. Choose
\(c_{\rm geo}\in(0,1/4]\) sufficiently small and put
\(\rho=c_{\rm geo}(\tau\wedge r_L)\). We first verify
\[
    \mathcal M\cap B_{\mathbb R^D}(x,\rho)
    \subset
    \Psi_x\bigl(B_{T_x\mathcal M}(0,\rho)\bigr).
\]
For \(y\) in the left-hand side, set
\(u=\Pi_{T_x\mathcal M}(y-x)\), \(q=\Psi_x(u)\), and \(w=y-q\).
Chart regularity gives \(\|DN_x(u)\|_{\rm op}\le L\|u\|\le1/4\) and
\(\|N_x(u)\|\le L\|u\|^2/2\), so \(\|w\|<3\rho<\tau\).
Moreover, \(w\in T_x\mathcal M^\perp\), whereas \(T_q\mathcal M\) is
the graph of \(DN_x(u)\). Hence
\(d(w,T_q\mathcal M)\ge\|w\|/2\). The reach inequality
\citep[Theorem~4.18]{federer1959curvature} yields
\[
    \frac{\|w\|}{2}
    \le d(w,T_q\mathcal M)
    \le \frac{\|w\|^2}{2\tau}.
\]
Since \(\|w\|<\tau\), this forces \(w=0\), proving the inclusion.

In the rest of the proof, write
\[
    G_x:=N_x\big|_{B_{T_x\mathcal M}(0,\rho)},
    \qquad
    F_x(u):=x+u+G_x(u).
\]
Boundary values of \(G_x\) and its derivatives are inherited from \(N_x\),
which is defined on the larger ball \(B_{T_x\mathcal M}(0,r_L)\).
By the chart regularity and \(DG_x(0)=0\),
\[
    \|DG_x(u)\|_{\rm op}
    =
    \|DN_x(u)-DN_x(0)\|_{\rm op}
    \le
    L\|u\|,
    \qquad
    \|u\|\le \rho.
\]
Since \(\rho\le c_{\rm geo}r_L\) and \(r_L=(4L)^{-1}\),
\begin{equation}\label{eq:A.2}
    \sup_{\|u\|\le \rho}\|DG_x(u)\|_{\rm op}
    \le
    L\rho
    \le
    \frac{c_{\rm geo}}{4}
    <
    \frac14.
\end{equation}
We first prove the volume estimate. Write
\(\omega_d:=\operatorname{vol}_{\mathbb R^d}(B_{\mathbb R^d}(0,1))\).
The Jacobian of \(F_x\) is
\[
    J_x(u)
    =
    \sqrt{\det\{I+DG_x(u)^\top DG_x(u)\}}.
\]
By \eqref{eq:A.2},
\begin{equation}\label{eq:Jacobian-bound}
    1
    \le
    J_x(u)
    \le
    \left(1+\frac1{16}\right)^{d/2}
    =
    \left(\frac{17}{16}\right)^{d/2}.
\end{equation}
If \(F_x(u)\in \mathcal M\cap B_{\mathbb R^D}(x,r)\), then
\[
    \|u\|
    =
    \|\Pi_{T_x\mathcal M}(F_x(u)-x)\|
    \le
    \|F_x(u)-x\|
    <r.
\]
Thus
\[
    \mathcal M\cap B_{\mathbb R^D}(x,r)
    \subset
    F_x\bigl(B_{T_x\mathcal M}(0,r)\bigr).
\]
Using \eqref{eq:Jacobian-bound},
\[
    \operatorname{vol}_{\mathcal M}
    \bigl(\mathcal M\cap B_{\mathbb R^D}(x,r)\bigr)
    \le
    \omega_d\left(\frac{17}{16}\right)^{d/2}r^d.
\]

For the lower bound, if \(\|u\|<r/2\), then
\[
    \|G_x(u)\|
    \le
    \int_0^1 \|DG_x(tu)\|_{\rm op}\|u\|\,dt
    \le
    \frac L2\|u\|^2.
\]
Hence, since \(r\le \rho\) and \(L\rho\le 1/4\),
\[
    \|F_x(u)-x\|
    \le
    \|u\|+\|G_x(u)\|
    \le
    \frac r2+\frac L2\left(\frac r2\right)^2
    \le
    \frac r2+\frac r{32}
    <r.
\]
Therefore
\[
    F_x\bigl(B_{T_x\mathcal M}(0,r/2)\bigr)
    \subset
    \mathcal M\cap B_{\mathbb R^D}(x,r).
\]
Again using \eqref{eq:Jacobian-bound},
\[
    \operatorname{vol}_{\mathcal M}
    \bigl(\mathcal M\cap B_{\mathbb R^D}(x,r)\bigr)
    \ge
    \omega_d\left(\frac r2\right)^d.
\]
This proves (i). Let
\[
    y=F_x(u),
    \qquad
    z=F_x(v),
    \qquad
    u,v\in B_{T_x\mathcal M}(0,\rho).
\]
The ambient distance is always bounded above by the intrinsic distance on the manifold \(\mathcal M\):
\[
    \|y-z\|\le d_{\mathcal M}(y,z).
\]
For the reverse inequality, consider the lifted line segment
\[
    \gamma(t):=F_x((1-t)u+tv),
    \qquad
    t\in[0,1].
\]
By \eqref{eq:A.2},
\[
\operatorname{length}(\gamma)
    \le
    \int_0^1
    \sqrt{1+\|DG_x((1-t)u+tv)\|_{\rm op}^2}\,
    \|u-v\|\,dt  \le
    \left(\frac{17}{16}\right)^{1/2}\|u-v\|.
\]
Moreover,
\(u-v=\Pi_{T_x\mathcal M}(y-z)\) implies \(\|u-v\|\le\|y-z\|\). Hence
\[
    d_{\mathcal M}(y,z)
    \le
    \operatorname{length}(\gamma)
    \le
    \left(\frac{17}{16}\right)^{1/2}\|y-z\|
    \le
    2\|y-z\|.
\]
This proves (ii). Write \(y=F_x(u)\). Then \(T_y\mathcal M\) is the graph of
\[
    DG_x(u):T_x\mathcal M\to T_x\mathcal M^\perp.
\]
For a linear map \(A:T_x\mathcal M\to T_x\mathcal M^\perp\) with
\(\|A\|_{\rm op}\le1\), let
\(\operatorname{graph}(A)
    :=
    \{v+Av:v\in T_x\mathcal M\}
    \subset \mathbb R^D\) and \(\theta_{\max}\) is the largest canonical angle between \(\mathrm{graph}(A)\) and \(T_x\mathcal M\). Then we have
\[
    \|\Pi_{\operatorname{graph}(A)}-\Pi_{T_x\mathcal M}\|_{\rm op}
    =\sin \theta_{\max} = \sup_{v\in T_x\mathcal M\setminus \{0\}}\frac{\|Av\|}{\|v+Av\|} =\frac{\|A\|_{\rm op}}{\sqrt{1+\|A\|_{\rm op}^2}} \le \|A\|_{\rm op},
\]
where the first equality comes from linear algebra; see~\citep{golub2013matrix}.
Using this with \(A=DG_x(u)\), and then using the chart regularity,
\[
    \|\Pi_{T_y\mathcal M}-\Pi_{T_x\mathcal M}\|_{\rm op}
    \le
    \|DG_x(u)\|_{\rm op}  
    \le
    L\|u\| 
    \le
    L\|y-x\|.
\]
This proves (iii).
\end{proof}

\subsection{Constant-explicit version of Assumption~\ref{ass:feasible-local}}
\label{app:feasible-local-formal}
We use the fixed-$n$ aliases of the main text and retain explicit $n$
indices in $\tau_{0,n}$, $\delta_{0,n}$, $\pi_{k,n}$, $\sigma_{k,n}$,
$\mu_{k,n}$, and $\varrho_{k,n}$. Write
$\sigma_{\max,n}:=\max_{k\in[K_n]}\sigma_{k,n}$ and
$R_{\rm loc}=R_{{\rm loc},n}$.
Recall \(N_k=n\pi_{k,n}\), \(h_k=h_{k,n}=N_k^{-1/(2\alpha_k+d_k)}\), and
\(r_n=\max_k h_k\). Fix the exponent \(\kappa\in(0,1)\) from
Assumption~\ref{ass:feasible-local}. Set
\[
 r_{L_{\mathcal M}}=(4L_{\mathcal M})^{-1},\qquad
 R_{\rm loc}=c_{\rm geo}(\tau_{0,n}\wedge r_{L_{\mathcal M}}\wedge1),\qquad
 a_{\rm sc}=\frac{c_{\rm geo}}{64\sqrt D}.
\]
Choose \(c_{\rm geo}\in(0,1)\) sufficiently small that the local estimates
of Lemma~\ref{lem:reach-local-facts-main} hold up to radius \(2R_{\rm loc}\)
and the fixed small-perturbation requirements in
Appendix~\ref{app:tangent-local-poly} hold. In particular, the requirements
involving \(2a_{\rm sc}\) in the regression design proof must hold.
Such a choice is possible by reducing the uniform localization radius and
then imposing finitely many smallness conditions, all depending only on
fixed model bounds.

\begin{assumption*}[Formal restatement of Assumption~\ref{ass:feasible-local}]
\label{ass:feasible-local-formal}
For all sufficiently large \(n\), uniformly over \(k\in[K_n]\),
\[
 N_k\ge n^\kappa,\qquad
 \tau_{0,n}\wedge\delta_{0,n}\ge a_{\rm sc}^{-1}r_n,\qquad
 \sigma_{k,n}\le a_{\rm sc}h_k^{\alpha_k\vee1}.
\]
\end{assumption*}
Define
\[
 \rho_h:=\frac{\kappa}{2\alpha_{\max}+d_{\max}},\qquad
 \bar\rho_h:=\frac{1}{2\alpha_{\min}+1},\qquad
 \rho_{\rm tan}:=\frac{2\kappa\alpha_{\min}}{2\alpha_{\min}+d_{\max}}.
\]

\begin{lemma}
\label{lem:ass3-consequences}
Under Assumption~\ref{ass:feasible-local-formal}, for all sufficiently large
\(n\), uniformly over \(k\),
\begin{gather*}
 N_k\ge n^{\kappa},\qquad K_n\le n^{1-\kappa},\qquad
 n^{-\bar\rho_h}\le h_k\le r_n\le n^{-\rho_h},\\
 n\pi_{k,n}h_k^{d_k}\ge n^{\rho_{\rm tan}},\qquad
 (n\pi_{k,n}h_k^{d_k})^{-1}=h_k^{2\alpha_k},\\
 h_k\le R_{\rm loc},\qquad r_n+2\sigma_{\max,n}\le\delta_{0,n}/2,\qquad
 \frac{2\sigma_{k,n}}{h_k}\le2a_{\rm sc}\le\frac{1}{4\sqrt D}.
\end{gather*}
\end{lemma}
\begin{proof}
Since \(\sum_{k\in[K_n]}N_k=n\) and \(n^\kappa\le N_k\le n\),
\[
 K_n n^\kappa\le n,\qquad
 n^{-1/(2\alpha_{\min}+1)}
 \le N_k^{-1/(2\alpha_k+d_k)}
 \le n^{-\kappa/(2\alpha_{\max}+d_{\max})}.
\]
Furthermore,
\[
 N_kh_k^{d_k}=N_k^{2\alpha_k/(2\alpha_k+d_k)}
 \ge n^{\rho_{\rm tan}},\qquad
 (N_kh_k^{d_k})^{-1}=h_k^{2\alpha_k}.
\]
For large \(n\), \(h_k\le r_n\le1\). Hence
\[
 \sigma_{k,n}/h_k\le a_{\rm sc},\qquad
 \sigma_{\max,n}\le a_{\rm sc}r_n,\qquad
 r_n+2\sigma_{\max,n}\le(1+2a_{\rm sc})a_{\rm sc}\delta_{0,n}
 \le\delta_{0,n}/2.
\]
Finally, \(h_k\le r_n\le a_{\rm sc}\tau_{0,n}\le c_{\rm geo}\tau_{0,n}\), while
\(r_n\to0\) ensures \(r_n\le c_{\rm geo}(r_{L_{\mathcal M}}\wedge1)\).
Together these give \(h_k\le R_{\rm loc}\).
\end{proof}

\section{Minimax lower bound}
\label{app:minimax-lower}

\subsection{Submodel for lower bound}\label{subsec:submodel}
It suffices to prove the lower bound on a submodel, so we take
\(\xi\equiv0\). Fix a symmetric nonnegative
\(r_0\in C_c^\infty((-1,1))\) with \(\int r_0^2=1\), and set
\(q_0:=r_0^2\). Then \(q_0\) is a symmetric density supported on
\((-1,1)\), and
\[
    \int_{\mathbb R}
    \left\{
        {d\over du}\sqrt{q_0(u)}
    \right\}^2du
    =
    \int_{\mathbb R}r_0'(u)^2du
    <\infty .
\]
Let 
\[
    \sigma_{\rm lb}:=(B_\epsilon\wedge\sigma_Y)/2,\qquad
    q_{\rm lb}(u):=\sigma_{\rm lb}^{-1}q_0(u/\sigma_{\rm lb}).
\] 
In the submodel, \(\epsilon\) has density \(q_{\rm lb}\) and is
independent of \(X^\star\). Since \(q_0\) is symmetric and supported on
\((-1,1)\), we have 
\[
    \mathbb E[\epsilon]=0,\qquad
    |\epsilon|\le\sigma_{\rm lb}\le B_\epsilon,\ \mathrm{a.s.},\qquad
    \mathbb E[\epsilon^2]\le\sigma_{\rm lb}^2\le\sigma_Y^2.
\]
Thus this submodel satisfies the noise conditions. Under Assumption~\ref{ass:geometry} and non-empty condition of \(\mathcal P_\star\),
\[
    \inf_k\operatorname{reach}(\mathcal M_k)\ge\tau_{0,n},\qquad
    \inf_{k\ne \ell}d(\mathcal M_k,\mathcal M_\ell)\ge\delta_{0,n}.
\]
Throughout this section, the constants \(c,C\) follow the convention in
Section~\ref{subsec:notations}, with dependence restricted to the fixed
statistical model bounds and scale margins, \(q_0\), and the fixed bump
functions used below.

\subsection{Auxiliary lemmas}

We use the following standard \(L^2\)-form of Assouad's lemma; see
\citep[Lemma 2.12]{tsybakov2009introduction}.

\begin{lemma}
\label{lem:assouad-l2}
Let \(J\) be finite, \(\Theta=\{0,1\}^J\), and
\(f_\theta\in L^2(\nu)\). Let \(P_\theta\) be the corresponding observation
law and \(\theta^{(r)}\) denote a flip in coordinate \(r\). Suppose
\begin{equation}\label{eq:flip-bound}
 \|f_\theta-f_{\theta'}\|_{L^2(\nu)}^2
 \ge\sum_{r\in J}\rho_r^2\mathbb{I}\{\theta_r\ne\theta'_r\}.
\end{equation}
Then
\begin{align*}
 \inf_{\widetilde f}\sup_\theta
 \mathbb E_\theta\|\widetilde f-f_\theta\|_{L^2(\nu)}^2
 &\ge \inf_{\widetilde f}2^{-|J|}\sum_\theta
 \mathbb E_\theta\|\widetilde f-f_\theta\|_{L^2(\nu)}^2\\
 &\ge\frac18\sum_{r\in J}\rho_r^2
 \left\{1-\sup_{\theta:\theta_r=0}
 d_{\rm TV}(P_\theta,P_{\theta^{(r)}})\right\}.
\end{align*}
The infima include randomized measurable estimators with independent
auxiliary randomness.
\end{lemma}
\begin{proof}
The bound is immediate for estimators with infinite prior-average risk.
For any estimator with finite prior-average risk, choose a nearest \(f_\vartheta\) in the finite
family, with deterministic tie-breaking, and call its index
\(\widehat\theta\). The triangle inequality and \eqref{eq:flip-bound} give
\[
 \|\widetilde f-f_\theta\|_{L^2(\nu)}^2
 \ge\frac14\sum_r\rho_r^2\mathbb{I}\{\widehat\theta_r\ne\theta_r\}.
\]
For each \(r\), pairing \(\theta\) and \(\theta^{(r)}\) yields
\[
 2^{-|J|}\sum_\theta P_\theta(\widehat\theta_r\ne\theta_r)
 \ge\frac12\left\{1-\sup_{\theta:\theta_r=0}
 d_{\rm TV}(P_\theta,P_{\theta^{(r)}})\right\}.
\]
Sum these inequalities and take the infimum. Independent auxiliary
randomness can be conditioned on, or appended to both laws while preserving
their total variation distance.
\end{proof}

Next, we give a lemma bounding Hellinger distance by \(L^2\) distance. Let
\[
    \mathsf H^2(P,Q)
    :=
    \int
    \left(
        \sqrt{dP/d\lambda}-\sqrt{dQ/d\lambda}
    \right)^2d\lambda
\]
for two probability measures \(P,Q\ll\lambda\). 

\begin{lemma}
\label{lem:bounded-noise-hellinger}
Let \(Q_a\) have density \(q_{\rm lb}(y-a)\).
Then, for all \(a,b\in\mathbb R\),
\[
    \mathsf H^2(Q_a,Q_b)
    \le
    C(a-b)^2.
\]
Here \(C\) depends only on \(q_0\) and \(\sigma_{\rm lb}\).
Consequently, if \(X\sim\nu\), \(Y_f=f(X)+\epsilon\), and \(Y_g=g(X)+\epsilon\)
with \(\epsilon\sim q_{\rm lb}\) independent of \(X\), let
\(P_f^{(n)}\) and \(P_g^{(n)}\) denote the laws of \(n\) i.i.d.
covariate--response pairs. Then
\[
    \mathsf H^2(P_f^{(1)},P_g^{(1)})
    \le
    C\|f-g\|_{L^2(\nu)}^2,
    \qquad
    \mathsf H^2(P_f^{(n)},P_g^{(n)})
    \le
    Cn\|f-g\|_{L^2(\nu)}^2.
\]
\end{lemma}

\begin{proof}
Write \(r_{\rm lb}:=\sqrt{q_{\rm lb}}\). Since \(Q_a\) has density
\(q_{\rm lb}(y-a)\), its square-root density is \(r_{\rm lb}(y-a)\).
Hence
\[
\begin{aligned}
    \mathsf H^2(Q_a,Q_b)
    &=
    \int_{\mathbb R}
    \{r_{\rm lb}(y-a)-r_{\rm lb}(y-b)\}^2\,dy
    =
    \|r_{\rm lb}(\cdot-a)-r_{\rm lb}(\cdot-b)\|_2^2 .
\end{aligned}
\]
Set \(t:=a-b\) and note that
\[
    \|r_{\rm lb}(\cdot-a)-r_{\rm lb}(\cdot-b)\|_2
    =
    \|r_{\rm lb}(\cdot-t)-r_{\rm lb}\|_2.
\]
Since \(r_{\rm lb}\) is absolutely continuous and \(r'_{\rm lb}\in L^2\),
\[
    r_{\rm lb}(y-t)-r_{\rm lb}(y)
    =
    -\int_0^t r'_{\rm lb}(y-s)\,ds
    =
    -\operatorname{sgn}(t)
    \int_0^{|t|}
    r'_{\rm lb}(y-\operatorname{sgn}(t)s)\,ds .
\]
Minkowski's integral inequality gives
\[
    \|r_{\rm lb}(\cdot-t)-r_{\rm lb}(\cdot)\|_2
    \le
    \int_0^{|t|}
    \|r'_{\rm lb}(\cdot-\operatorname{sgn}(t)s)\|_2\,ds
    =
    |t|\|r'_{\rm lb}\|_2,
\]
so
\[
    \mathsf H^2(Q_a,Q_b)
    \le
    (a-b)^2\|r_{\rm lb}'\|_2^2 .
\]
Since \(q_{\rm lb}(u)=\sigma_{\rm lb}^{-1}q_0(u/\sigma_{\rm lb})\), we have
\[
    r_{\rm lb}(u)=\sigma_{\rm lb}^{-1/2}r_0(u/\sigma_{\rm lb})
\]
and
\[
    r_{\rm lb}'(u)
    =
    \sigma_{\rm lb}^{-3/2}r_0'(u/\sigma_{\rm lb}).
\]
Hence, by the change of variables \(v=u/\sigma_{\rm lb}\),
\[
\begin{aligned}
    \|r_{\rm lb}'\|_2^2
    &=
    \int_{\mathbb R}
    \sigma_{\rm lb}^{-3}
    \{r_0'(u/\sigma_{\rm lb})\}^2\,du
    =
    \sigma_{\rm lb}^{-2}
    \int_{\mathbb R}\{r_0'(v)\}^2\,dv
    \le C.
\end{aligned}
\]
This proves \(\mathsf H^2(Q_a,Q_b)\le C(a-b)^2\).

\smallskip
Now consider the one-sample joint laws. Let \(P_f^{(1)}\) and \(P_g^{(1)}\)
denote the laws of \((X,Y_f)\) and \((X,Y_g)\), respectively. Since
\(X\sim\nu\) and, conditionally on \(X=x\), \(Y_f=f(x)+\epsilon\) has density
\(q_{\rm lb}(y-f(x))\), the law \(P_f^{(1)}\) is absolutely continuous
with respect to \(\nu(dx)\,dy\), with density
\[
    p_f(x,y)=q_{\rm lb}(y-f(x)).
\]
Similarly,
\[
    p_g(x,y)=q_{\rm lb}(y-g(x)).
\]
Therefore
\[
\begin{aligned}
    \mathsf H^2(P_f^{(1)},P_g^{(1)})
    &=
    \int\!\!\int
    \{r_{\rm lb}(y-f(x))-r_{\rm lb}(y-g(x))\}^2
    \,dy\,\nu(dx) \\
    &=
    \int
    \mathsf H^2(Q_{f(x)},Q_{g(x)})\,\nu(dx)
    \le
    C\|f-g\|_{L^2(\nu)}^2 .
\end{aligned}
\]
Then we have
\[
    \mathsf H^2(P_f^{(n)},P_g^{(n)})
    \le
    n\mathsf H^2(P_f^{(1)},P_g^{(1)})
    \le
    Cn\|f-g\|_{L^2(\nu)}^2.
\]
\end{proof}

\begin{lemma}
\label{lem:geometric-bump-packing}
Let \(\mathcal M\subset\mathbb R^D\) be compact with
\(\mathcal M\in\mathcal C_{d,L_{\mathcal M},\tau}^{\beta}(\mathbb R^D)\), where
\(\tau>0\). Suppose
\(d\in [d_{\max}]\), \(\alpha\in[\alpha_{\min},\alpha_{\max}]\),
\(\beta\in[2,\beta_{\max}]\), and \(\beta\ge\lceil\alpha\rceil+1\). There
exists \(a_{\rm geo}>0\), depending only on
\(D,d_{\max},\beta_{\max},L_{\mathcal M}\), such that, whenever
\(h\in (0,a_{\rm geo}(\tau\wedge1))\), there are points
\(\mathcal X_h=\{x_1,\ldots,x_m\}\subset\mathcal M\) satisfying
\begin{equation}
    m
    \ge
    c\operatorname{vol}_{\mathcal M}(\mathcal M)h^{-d},
    \label{eq:geometric-bump-packing-m-bound}
\end{equation}
and there are functions
\(\varphi_1,\ldots,\varphi_m:\mathcal M\to\mathbb R\) with pairwise
disjoint supports, i.e.,
\[
    \operatorname{supp}(\varphi_i)
    \cap
    \operatorname{supp}(\varphi_j)
    =
    \emptyset
    \qquad
    \text{for }i\neq j,
\]
satisfying
\begin{equation*}
    \operatorname{supp}(\varphi_j)
    \subset
    \mathcal M\cap B_{\mathbb R^D}(x_j,h),
\end{equation*}
and further satisfy
\begin{equation}
    \left\|
    \sum_{j\in[m]} \vartheta_j h^\alpha\varphi_j
    \right\|_{\mathcal H^\alpha(\mathcal M)}
    \le
    C
    \quad
    \text{for all }\vartheta\in\{0,1\}^m,
    \label{eq:geometric-bump-packing-norm-bound}
\end{equation}
and
\begin{equation*}
    ch^d
    \le
    \int_{\mathcal M}\varphi_j^2\,d\operatorname{vol}_{\mathcal M}
    \le
    Ch^d .
\end{equation*}
The constants \(c,C\) depend only on
\(D,d_{\max},\alpha_{\min},\alpha_{\max},\beta_{\max},L_{\mathcal M}\) and on the
fixed bump functions.
\end{lemma}

\begin{proof}
Let \(c_{\rm geo}\in(0,1)\) be the localization constant in
Lemma~\ref{lem:reach-local-facts-main}, chosen uniformly over
\(d\in [d_{\max}]\), \(\beta\in[2,\beta_{\max}]\), and \(L=L_{\mathcal M}\). Put
\[
    r_{L_{\mathcal M}}:=(4L_{\mathcal M})^{-1},\qquad
    a_{\rm geo}:=c_{\rm geo}(r_{L_{\mathcal M}}\wedge1)/16.
\]
If
\(h\le a_{\rm geo}(\tau\wedge1)\), then
\(4h\le c_{\rm geo}(\tau\wedge r_{L_{\mathcal M}})\) and \(h\le r_{L_{\mathcal M}}\).

\smallskip
Let \(\mathcal X_h=\{x_1,\ldots,x_m\}\) be a maximal \(4h\)-separated subset
of \(\mathcal M\), with separation measured in ambient Euclidean distance,
that is, the following is satisfied:
\begin{enumerate}
    \item[(a)] for any \(j\neq j'\in[m]\),
    \(\|x_j-x_{j'}\|\ge4h\).
    \item[(b)] for any \(x\in\mathcal M\), there is \(j\in[m]\) such that
    \(\|x-x_j\|<4h\).
\end{enumerate}
Note that (b) implies
\[
    \mathcal M
    \subset
    \bigcup_{j\in[m]}B_{\mathbb R^D}(x_j,4h).
\]
Since
\(4h\le c_{\rm geo}(\tau\wedge r_{L_{\mathcal M}})\),
Lemma~\ref{lem:reach-local-facts-main}(i)
gives
\[
    \operatorname{vol}_{\mathcal M}(\mathcal M)
    \le
    \sum_{j\in[m]}
    \operatorname{vol}_{\mathcal M}
    \bigl(\mathcal M\cap B_{\mathbb R^D}(x_j,4h)\bigr)
    \le
    Cmh^d.
\]
Hence,
\[
    m
    \ge
    c\operatorname{vol}_{\mathcal M}(\mathcal M)h^{-d}.
\]

\smallskip
Fix a nonzero bump
\(\eta_d\in C_c^\infty(B_{\mathbb R^d}(0,1/4))\). For each \(j\), let
\[
    \Psi_j(u)=x_j+u+N_j(u),
    \qquad
    u\in B_{T_{x_j}\mathcal M}(0,r_{L_{\mathcal M}}),
\]
be the tangent-normal chart at \(x_j\). Identifying
\(T_{x_j}\mathcal M\) with
\(\mathbb R^d\) by an arbitrary orthonormal basis, define
\[
    \varphi_j(\Psi_j(u)):=\eta_d(u/h)
    \qquad
    \text{for }\|u\|<h/4,
\]
and set \(\varphi_j=0\) outside.

\smallskip
If \(\varphi_j(\Psi_j(u))\ne0\), then \(\|u\|\le h/4\). Since
\(DN_j(0)=0\)
and \(\|DN_j(v)\|_{\rm op}\le L_{\mathcal M}\|v\|\),
\[
    \|N_j(u)\|
    \le
    \int_0^1\|DN_j(tu)\|_{\rm op}\|u\|\,dt
    \le
    {L_{\mathcal M}\over2}\|u\|^2.
\]
Thus
\[
    \|\Psi_j(u)-x_j\|
    \le
    \|u\|+\|N_j(u)\|
    \le
    {h\over4}
    +
    {L_{\mathcal M}\over2}\left({h\over4}\right)^2
    <h,
\]
because \(h\le r_{L_{\mathcal M}}\). So
\[
    \operatorname{supp}(\varphi_j)
    \subset
    \mathcal M\cap B_{\mathbb R^D}(x_j,h).
\]
The centers are \(4h\)-separated, so these supports are pairwise disjoint. Let
\[
    J_j(u)
    :=
    \sqrt{\det\{I+DN_j(u)^\top DN_j(u)\}}
\]
be the Jacobian in \(j\)th chart. On the support of \(\varphi_j\),
\(\|DN_j(u)\|_{\rm op}\le L_{\mathcal M}\|u\|\le1/16\), and hence
\[
    J_j(u)\in
    \left[
        1,\left(\frac{17}{16}\right)^{d/2}
    \right]
\]
by \eqref{eq:Jacobian-bound}.
Therefore, by the change of variables \(u=hv\),
\[
    \int_{\mathcal M}\varphi_j^2\,d\operatorname{vol}_{\mathcal M}
    =
    \int
    \eta_d(u/h)^2J_j(u)\,du
    =
    h^d
    \int
    \eta_d(v)^2J_j(hv)\,dv .
\]
Since \(\eta_d\) is fixed and nonzero,
\[
    ch^d
    \le
    \int_{\mathcal M}\varphi_j^2\,d\operatorname{vol}_{\mathcal M}
    \le
    Ch^d .
\]
Let $F_\vartheta
    :=
    \sum_{j\in[m]}\vartheta_j h^\alpha\varphi_j$ for 
$\vartheta\in\{0,1\}^m$. In the chart centered at \(x_j\),
\[
    h^\alpha\varphi_j(\Psi_j(u))
    =
    h^\alpha\eta_d(u/h).
\]
For every multi-index \(\mu\) with \(|\mu|=r\le s_\alpha\), the chain rule
gives
\[
    D^\mu\{h^\alpha\eta_d(u/h)\}
    =
    h^{\alpha-r}(D^\mu\eta_d)(u/h).
\]
Hence all derivatives of order \(r\le s_\alpha\) are uniformly bounded for
\(0<h\le1\), since \(r<\alpha\). For the top H\"older seminorm, if
\(|\mu|=s_\alpha\), then
\[
[D^\mu\{h^\alpha\eta_d(\cdot/h)\}]_{C^{0,\gamma_\alpha}}
\le
h^{\alpha-s_\alpha}h^{-\gamma_\alpha}
[D^\mu\eta_d]_{C^{0,\gamma_\alpha}}
=
[D^\mu\eta_d]_{C^{0,\gamma_\alpha}},
\]
because \(\alpha=s_\alpha+\gamma_\alpha\). Thus the factor \(h^\alpha\)
exactly compensates for the derivative scaling \(h^{-s_\alpha}\) and the
H\"older scaling \(h^{-\gamma_\alpha}\).

In any admissible chart \(\Psi_x\), the transition to the \(j\)-th chart
on their overlap is
\[
    \Psi_j^{-1}(\Psi_x(v))
    =\Pi_{T_{x_j}\mathcal M}(\Psi_x(v)-x_j).
\]
Its derivatives through order \(s_\alpha+1\) are uniformly bounded because
\(\beta\ge s_\alpha+2\) and the chart radius and derivative bounds are
fixed independently of \(\tau\). Each bump vanishes near the boundary of
its defining chart, so its extension by zero has the same regularity.
The chain rule therefore gives derivative bounds
\(Ch^{\alpha-r}\) for each pullback \(h^\alpha\varphi_j\circ\Psi_x\),
for \(0\le r\le s_\alpha+1\). Since the bump supports are disjoint,
these bounds also hold for \(g_x:=F_\vartheta\circ\Psi_x\):
\[
    \|D^r g_x\|_\infty\le Ch^{\alpha-r},
    \qquad 0\le r\le s_\alpha+1.
\]
The chart domain is convex. Thus the bounds for orders \(s_\alpha\) and
\(s_\alpha+1\) imply, for any two points \(u,v\) in that domain,
\[
    \|D^{s_\alpha}g_x(u)-D^{s_\alpha}g_x(v)\|
    \le C\min\{h^{\gamma_\alpha},
                  h^{\gamma_\alpha-1}\|u-v\|\}
    \le C\|u-v\|^{\gamma_\alpha}.
\]
Together with \(h\le1\), this proves
\[
    \|F_\vartheta\|_{\mathcal H^\alpha(\mathcal M)}\le C
\]
uniformly over \(\vartheta\in\{0,1\}^m\), all admissible charts, and
\(h\), with constants independent of \(\tau\).
\end{proof}

\subsection{Proof of Theorem~\ref{thm:minimax-lower}}

\begin{proof}
It suffices to prove the lower bound on a submodel fixed in
Section~\ref{subsec:submodel}. Let
\[
    \nu^\circ:=\sum_{k\in[K_n]}\pi_{k,n}\mu_{k,n}^\circ,
\]
where $d\mu_{k,n}^\circ=\varrho_{k,n}^\circ\,d\operatorname{vol}_{\mathcal M_k}$
and $c_\mu\le\varrho_{k,n}^\circ\le c_\mu^{-1}$. Since
\[
    1
    =
    \int_{\mathcal M_k}
    \varrho_{k,n}^\circ\,d\operatorname{vol}_{\mathcal M_k}
    \le
    c_\mu^{-1}\operatorname{vol}_{\mathcal M_k}(\mathcal M_k),
\]
we have
$\operatorname{vol}_{\mathcal M_k}(\mathcal M_k)\ge c_\mu$.
Recall
\[
    N_k:=n\pi_{k,n},\qquad
    h_{k,n}:=N_k^{-1/(2\alpha_k+d_k)},\qquad
    r_n:=\max_{k\in[K_n]}h_{k,n}.
\]
By Lemma~\ref{lem:ass3-consequences},
\(r_n\to0\), and by Assumption~\ref{ass:feasible-local-formal},
\[
    r_n\le a_{\rm sc}(\tau_{0,n}\wedge\delta_{0,n}).
\]

\smallskip
Let \(a_{\rm geo}\) be the constant in
Lemma~\ref{lem:geometric-bump-packing}. Define
\[
    a_0:=1\wedge a_{\rm geo},\qquad
    \bar h_k:=a_0h_{k,n}.
\]
Since \(a_{\rm sc}\le1\), for all sufficiently large \(n\),
\[
    \bar h_k
    \le
    a_0r_n
    \le
    a_0a_{\rm sc}\tau_{0,n}
    \le
    a_{\rm geo}\tau_{0,n}.
\]
Also, since \(r_n\to0\), \(\bar h_k\le a_{\rm geo}\) for all sufficiently
large \(n\). Hence
\[
    \bar h_k
    \le
    a_{\rm geo}(\tau_{0,n}\wedge1)
    \le
    a_{\rm geo}
    \{\operatorname{reach}(\mathcal M_k)\wedge1\}.
\]
Thus Lemma~\ref{lem:geometric-bump-packing} applies to every component
\(\mathcal M_k\in\mathcal C_{d_k,L_{\mathcal M},\tau_{0,n}}^{\beta_k}
(\mathbb R^D)\) with \(\tau=\tau_{0,n}\) and \(h=\bar h_k\).

For each \(k\), let
\[
    \{(x_{k,j},\varphi_{k,j}):j\in[m_k]\}
\]
be the corresponding bump packing. Since
\(\operatorname{vol}_{\mathcal M_k}(\mathcal M_k)\ge c_\mu\), the packing
bound in \eqref{eq:geometric-bump-packing-m-bound} gives
\begin{equation}
    m_k
    \ge
    c
    \operatorname{vol}_{\mathcal M_k}(\mathcal M_k)
    \bar h_k^{-d_k}
    \ge
    c\bar h_k^{-d_k}.
    \label{eq:minimax-lower-m-bound}
\end{equation}
The supports are disjoint within each component by construction. Across
different components, they are also disjoint because
\[
    d(\mathcal M_k,\mathcal M_\ell)\ge\delta_{0,n}>0
    \qquad
    \text{for }k\ne\ell.
\]
Choose a fixed \(\lambda_{\rm lb}\in(0,1]\), depending only on the fixed
model constants, \(q_0\), and the bump functions, sufficiently small for
the H\"older and testing bounds below; their constants are independent of
\(\lambda_{\rm lb}\).
For
\[
    \mathcal{J}_n
    :=
    \{(k,j):k\in[K_n],\,j\in[m_k]\}
\]
and
\(\theta\in\{0,1\}^{|\mathcal{J}_n|}\), define, for
\(x\in\mathcal M_k\),
\[
    f_\theta(x)
    :=
    \sum_{j\in[m_k]}
    \theta_{k,j}
    \lambda_{\rm lb}
    \bar h_k^{\alpha_k}
    \varphi_{k,j}(x).
\]
By \eqref{eq:geometric-bump-packing-norm-bound} in
Lemma~\ref{lem:geometric-bump-packing},
\[
    \left\|
    \sum_{j\in[m_k]}
    \theta_{k,j}\bar h_k^{\alpha_k}\varphi_{k,j}
    \right\|_{\mathcal H^{\alpha_k}(\mathcal M_k)}
    \le
    C.
\]
The first requirement, \(C\lambda_{\rm lb}\le L_{\mathcal F}\), therefore ensures
$f_\theta
    \in
    \mathcal H(\boldsymbol\alpha,\mathscr M;L_{\mathcal F})$ for every
\(\theta\in\{0,1\}^{|\mathcal{J}_n|}\).

\smallskip
Let \(P_\theta^{(n)}\) be the distribution of the training sample under
\[
    X_i^\star\stackrel{\rm i.i.d.}{\sim}\nu^\circ,\qquad
    X_i=X_i^\star,\qquad
    Y_i=f_\theta(X_i^\star)+\epsilon_i.
\]
For
\(\theta,\theta'\in\{0,1\}^{|\mathcal{J}_n|}\), the pairwise disjointness
of the supports implies that, pointwise on \(\mathcal M_k\),
\[
\begin{aligned}
    |f_\theta-f_{\theta'}|^2
    &=
    \lambda_{\rm lb}^2\bar h_k^{2\alpha_k}
    \left|
        \sum_{j\in[m_k]}
        (\theta_{k,j}-\theta'_{k,j})
        \varphi_{k,j}
    \right|^2 \\
    &=
    \lambda_{\rm lb}^2\bar h_k^{2\alpha_k}
    \sum_{j\in[m_k]}
    (\theta_{k,j}-\theta'_{k,j})^2
    \varphi_{k,j}^2 \\
    &=
    \lambda_{\rm lb}^2\bar h_k^{2\alpha_k}
    \sum_{j\in[m_k]}
    \mathbb{I}\{\theta_{k,j}\ne\theta'_{k,j}\}
    \varphi_{k,j}^2.
\end{aligned}
\]
Moreover, the density bounds
\(c_\mu\le\varrho_{k,n}^\circ\le c_\mu^{-1}\) and
Lemma~\ref{lem:geometric-bump-packing} give
\[
    c\bar h_k^{d_k}
    \le
    \int_{\mathcal M_k}
    \varphi_{k,j}(x)^2
    \varrho_{k,n}^\circ(x)
    \,d\operatorname{vol}_{\mathcal M_k}(x) 
    \le
    C\bar h_k^{d_k}.
\]
Consequently,
\[
\begin{aligned}
    \|f_\theta-f_{\theta'}\|_{L^2(\nu^\circ)}^2
    &=
    \sum_{k\in[K_n]}\pi_{k,n}
    \int_{\mathcal M_k}
    |f_\theta-f_{\theta'}|^2
    \varrho_{k,n}^\circ\,d\operatorname{vol}_{\mathcal M_k} \\
    &=
    \sum_{k\in[K_n]}
    \sum_{j\in[m_k]}
    \lambda_{\rm lb}^2
    \pi_{k,n}
    \bar h_k^{2\alpha_k}
    \mathbb{I}\{\theta_{k,j}\ne\theta'_{k,j}\}
    \int_{\mathcal M_k}
    \varphi_{k,j}^2
    \varrho_{k,n}^\circ
    \,d\operatorname{vol}_{\mathcal M_k} \\
    &\ge
    \sum_{(k,j)\in\mathcal J_n}
    c\lambda_{\rm lb}^2
    \pi_{k,n}
    \bar h_k^{2\alpha_k+d_k}
    \mathbb{I}\{\theta_{k,j}\ne\theta'_{k,j}\}.
\end{aligned}
\]
Fix a uniform lower constant \(c>0\) in the last bound. Then
Lemma~\ref{lem:assouad-l2} applies with
\[
    \rho_{k,j}^2
    :=
    c\lambda_{\rm lb}^2
    \pi_{k,n}
    \bar h_k^{2\alpha_k+d_k}.
\]
Let \(\theta^{(k,j)}\) be obtained from \(\theta\) by flipping only the
coordinate \((k,j)\). Then
\[
    f_\theta-f_{\theta^{(k,j)}}
    =
    \pm
    \lambda_{\rm lb}
    \bar h_k^{\alpha_k}
    \varphi_{k,j}
\]
on the corresponding bump support, and the difference is zero elsewhere.
By Lemma~\ref{lem:bounded-noise-hellinger} and the density upper bound,
\[
\begin{aligned}
    \mathsf H^2(P_\theta^{(n)},P_{\theta^{(k,j)}}^{(n)})
    &\le
    Cn
    \|f_\theta-f_{\theta^{(k,j)}}\|_{L^2(\nu^\circ)}^2 \\
    &\le
    Cn\lambda_{\rm lb}^2
    \pi_{k,n}
    \bar h_k^{2\alpha_k+d_k} \\
    &=
    C\lambda_{\rm lb}^2
    a_0^{2\alpha_k+d_k}.
\end{aligned}
\]
Since \(a_0\le1\), the second requirement, \(C\lambda_{\rm lb}^2\le1/16\),
gives
\[
    \mathsf H^2(P_\theta^{(n)},P_{\theta^{(k,j)}}^{(n)})
    \le
    {1\over16}.
\]
Using \(d_{\rm TV}(P,Q)\le\mathsf H(P,Q)\), we get
\[
    d_{\rm TV}(P_\theta^{(n)},P_{\theta^{(k,j)}}^{(n)})
    \le
    {1\over4}.
\]
Therefore Lemma~\ref{lem:assouad-l2} gives
\[
\begin{aligned}
    \inf_{\widetilde f}
    \sup_{\theta\in\{0,1\}^{|\mathcal{J}_n|}}
    \mathbb E_\theta
    \|\widetilde f-f_\theta\|_{L^2(\nu^\circ)}^2
    &\ge
    {1\over8}
    \sum_{(k,j)\in\mathcal J_n}
    \rho_{k,j}^2
    \left(1-{1\over4}\right) \\
    &=
    {3\over32}
    \sum_{(k,j)\in\mathcal J_n}
    \rho_{k,j}^2 \\
    &=
    {3\over32}
    c\lambda_{\rm lb}^2
    \sum_{k\in[K_n]}
    m_k\pi_{k,n}
    \bar h_k^{2\alpha_k+d_k}.
\end{aligned}
\]
By \eqref{eq:minimax-lower-m-bound},
\[
    m_k
    \ge
    c\bar h_k^{-d_k},
\]
and hence
\[
    m_k\pi_{k,n}\bar h_k^{2\alpha_k+d_k}
    \ge
    c\pi_{k,n}
    \bar h_k^{-d_k}
    \bar h_k^{2\alpha_k+d_k}
    =
    c\pi_{k,n}\bar h_k^{2\alpha_k}.
\]
It follows that
\[
\begin{aligned}
    \inf_{\widetilde f}
    \sup_{\theta\in\{0,1\}^{|\mathcal{J}_n|}}
    \mathbb E_\theta
    \|\widetilde f-f_\theta\|_{L^2(\nu^\circ)}^2
    &\ge
    c
    \sum_{k\in[K_n]}
    \pi_{k,n}\bar h_k^{2\alpha_k}.
\end{aligned}
\]
Since \(\bar h_k=a_0N_k^{-1/(2\alpha_k+d_k)}\),
\[
\begin{aligned}
    \sum_{k\in[K_n]}\pi_{k,n}\bar h_k^{2\alpha_k}
    &=
    \sum_{k\in[K_n]}
    \pi_{k,n}
    a_0^{2\alpha_k}
    N_k^{-2\alpha_k/(2\alpha_k+d_k)} \\
    &\ge
    a_0^{2\alpha_{\max}}
    \sum_{k\in[K_n]}
    \pi_{k,n}
    (n\pi_{k,n})^{-2\alpha_k/(2\alpha_k+d_k)} \\
    &=
    a_0^{2\alpha_{\max}}
    \mathfrak R_n(\boldsymbol\alpha,\mathbf d,\boldsymbol\pi).
\end{aligned}
\]

The same calculation uses the middle (uniform-prior) bound in
Lemma~\ref{lem:assouad-l2} and therefore also proves
\begin{equation}\label{eq:assouad-prior-risk}
 \inf_g 2^{-|\mathcal J_n|}\sum_\theta
 \mathbb E_{P^\circ,f_\theta}
 \{g(\mathfrak s)-f_\theta(X_{n+1}^\star)\}^2
 \ge c\mathfrak R_n.
\end{equation}
Here \(P^\circ\) is the fixed latent design and bounded-noise law of the
submodel.

Finally, on this submodel,
\[
    X_{n+1}=X_{n+1}^\star\sim\nu^\circ
\]
independently of the training sample. Therefore, conditional on
\((X_i,Y_i)_{i\in[n]}\), the prediction risk against
\(f_\theta(X_{n+1}^\star)\) is exactly the \(L^2(\nu^\circ)\)-risk of the
function
\[
    x
    \mapsto
    \widehat f_n((X_i,Y_i)_{i\in[n]},x).
\]
Thus the Assouad lower bound above applies to the original prediction
problem. Since the finite family
\[
    \{f_\theta:\theta\in\{0,1\}^{|\mathcal{J}_n|}\},
\]
together with the fixed design and noise law constructed above, is a
submodel of the full model class, the proof is complete.
\end{proof}

\subsection{Lower bound with independent pretraining tasks}
\label{app:pretraining-lower}

\begin{corollary}
\label{cor:pretraining-minimax-lower}
Under the hypotheses of Theorem~\ref{thm:minimax-lower}, for every sufficiently
large \(n\) there are \(P^\circ\in\mathcal P_\star\) and a finite-support
task distribution \(\rho_{f,n}^\circ\) such that, for every number of
independent training tasks \(\Gamma\ge0\),
\[
 \inf_{\mathcal A}\mathbb E_{P^\circ,\rho_{f,n}^\circ}
 \left[\left\{\mathcal A(\mathcal D_\Gamma^{\rm tr},\mathfrak s)
             -f(X_{n+1}^\star)\right\}^2\right]
 \ge c\mathfrak R_n.
\]
The tasks are generated as in Section~\ref{subsubsec:generation}; the
infimum ranges over all measurable learning rules, also allowing independent
randomization. The rule may be given \(P^\circ\) and \(\rho_{f,n}^\circ\).
Consequently the same lower bound holds for
\(\inf_{\mathcal A}\sup_{P,\rho_f}\) in this pretraining experiment.
\end{corollary}
\begin{proof}
Use the submodel and finite family from the proof above and set
\[
 \rho_{f,n}^\circ=2^{-|\mathcal J_n|}\sum_{\theta\in\{0,1\}^{\mathcal J_n}}
 \delta_{f_\theta}.
\]
The index \(\theta_0\) of the target task is independent of all training-task
indices and observations. Thus \(\mathcal D_\Gamma^{\rm tr}\) is independent
of the pair \((\theta_0,\mathfrak s)\). Conditional on almost every realized
training dataset \(D\), the map \(g_D(s)=\mathcal A(D,s)\) is a measurable
prompt predictor, and the conditional target law is unchanged. Applying
\eqref{eq:assouad-prior-risk} gives
\[
 \mathbb E[\{g_D(\mathfrak s)-f_{\theta_0}(X_{n+1}^\star)\}^2
          \mid\mathcal D_\Gamma^{\rm tr}=D]\ge c\mathfrak R_n.
\]
Average over \(D\) and any independent randomization. The hard prior is
allowed to depend on \(n\) and is fixed independently of the realized training data.
\end{proof}
This is a worst-case statement over task distributions. For a fixed
\(\rho_f\), the lower bound depends on that distribution; for example a
point mass at the zero function has zero latent prediction risk.

\section{Tangent local polynomial estimator}
\label{app:tangent-local-poly}

The upper-bound proof proceeds from local geometry to regression risk.
Local purity restricts each window to one component; polynomial coercivity
and concentration over the finite candidate grid then control the tangent
fit. Lemma~\ref{lem:tangent-projector-good} converts this control into a
rounded projector estimate. Lemma~\ref{lem:regression-design-stability}
establishes a single Gram event uniformly over nearby projectors, so it
also applies to the estimated projector. The final bias--variance argument
combines these two events to obtain the componentwise and aggregate rates.

\subsection{Standing notation and scale consequences}
\label{app:tangent-local-poly-estimator}

We use the notation from the main text and
Appendix~\ref{app:feasible-local-formal}. All spectral projectors use a
fixed measurable tie-breaking rule, including outside the good events. Let \(\Pi_{(d)}(B)\) denote the projector onto
the top \(d\)-dimensional eigenspace of a symmetric matrix \(B\). For
\(k\in[K_n]\), recall
\[
\begin{gathered}
 p_k:=\lceil\alpha_k\rceil-1,\qquad s_k:=\lceil\alpha_k\rceil,\qquad
 h_k:=h_{k,n}=N_k^{-1/(2\alpha_k+d_k)},\\
 p_{\max}:=\lceil\alpha_{\max}\rceil-1,\qquad
 s_{\max}:=\lceil\alpha_{\max}\rceil.
\end{gathered}
\]
The regression polynomial degree is \(p_k\), whereas the tangent graph fit uses
one additional degree \(s_k=p_k+1\). The degree \(s_k\) tangent fit is used
because the \(C^{s_k,1}\) regularity supplied by \(\beta_k\ge s_k+1\) yields,
after the rescaling \(u=h_kt\) and division by \(h_k\), a graph Taylor
remainder of order \(h_k^{s_k}\) on bounded \(t\)-sets.

For \(d\in [d_{\max}]\) and \(s\in [s_{\max}]\), let
\(\mathscr P_d\) be the set of rank-\(d\) orthogonal projectors in
\(\mathbb R^D\). The tensors below are understood as bounded symmetric multilinear maps
\(T_\ell:(\mathbb R^D)^\ell\to\mathbb R^D\). Define
\[
\Theta_{d,s}
:=
\left\{
(a,\Pi,T_2,\ldots,T_s):
\|a\|\le 1,
\ \Pi\in\mathscr P_d,
\ \|T_\ell\|_{\rm op}\le 2L_{\mathcal M},
\ 2\le \ell\le s
\right\}.
\]
Here, \(a\in\mathbb R^D\). The tensors are symmetric in their arguments.
This restriction preserves the diagonal evaluations
\(T_\ell(v,\ldots,v)\) entering \(R_q\) and the operator-norm bound, and
the Taylor tensors used below are symmetric. If \(s=1\), the tensor part is absent. For
\(q=(a,\Pi,T_2,\ldots,T_s)\in\Theta_{d,s}\), recall
\[
    R_q(z)
    :=
    a+(I-\Pi)z-\sum_{\ell=2}^s
    T_\ell\{(\Pi z)^{\otimes\ell}\}.
\]
The metric on \(\Theta_{d,s}\) is
\[
d_\Theta(q,q')
:=
\|a-a'\|+\|\Pi-\Pi'\|_F+
\sum_{\ell=2}^s\|T_\ell-T'_\ell\|_{\rm op}.
\]
\begin{lemma}[Product nets for tangent candidates]
\label{lem:tangent-product-net}
For \(d\in[D-1]\), \(s\in\mathbb N\), \(L_{\mathcal M}\ge1\), and
\(\eta\in(0,1]\), the space \((\Theta_{d,s},d_\Theta)\) has a finite
\(\eta\)-net \(\mathcal Q_{d,s}(\eta)\subset\Theta_{d,s}\) such that
\begin{equation}\label{eq:tangent-product-covering}
\begin{aligned}
 |\mathcal Q_{d,s}(\eta)|
 &\le C\eta^{-m_\Theta(d,s,D)},\\
 m_\Theta(d,s,D)
 &:=D+d(D-d)+D\sum_{\ell=2}^s\binom{D+\ell-1}{\ell}.
\end{aligned}
\end{equation}
The constant \(C\) depends only on \(D,d,s,L_{\mathcal M}\). All net centers belong
to their respective candidate factors, and the nets can be fixed
deterministically.
\end{lemma}
\begin{proof}
We first record a covering bound for a ball of radius \(R>0\) in an
\(m\)-dimensional normed vector space. A set of centers in that ball whose
pairwise distances exceed \(t>0\) has disjoint open balls of radius
\(t/2\). These balls lie in the ball of radius \(R+t/2\). Comparing
Lebesgue volumes in any linear coordinates bounds the number of centers by
\((1+2R/t)^m\). A greedy selection therefore terminates in a maximal such
set, which is a \(t\)-net with centers in the original ball. In particular,
the offset ball \(\{a\in\mathbb R^D:\|a\|\le1\}\) has a \(t\)-net of
size at most \((1+2/t)^D\).

For the projector factor, put \(M_{D,d}:=\binom Dd\). Given
\(\Pi\in\mathscr P_d\), choose \(U\in\mathbb R^{D\times d}\) with
\(U^\top U=I_d\) and \(\Pi=UU^\top\). For each size-\(d\) subset
\(J\subset[D]\), let \(U_J\) be the square matrix of rows indexed by
\(J\), in increasing order. Cauchy--Binet gives
\[
 \sum_{|J|=d}\det(U_J)^2=\det(U^\top U)=1.
\]
Thus some \(J\) satisfies \(|\det U_J|\ge M_{D,d}^{-1/2}\).
All singular values of \(U_J\) are at most one, so
\(\sigma_{\min}(U_J)\ge|\det U_J|\) and
\(\|U_J^{-1}\|_{\rm op}\le M_{D,d}^{1/2}\). Define
\[
 A:=U_{J^c}U_J^{-1}\in\mathbb R^{(D-d)\times d},
 \qquad
 R_{D,d}:=\sqrt{dM_{D,d}}.
\]
Since \(\|U_{J^c}\|_{\rm op}\le1\), we have
\(\|A\|_F\le\sqrt d\|A\|_{\rm op}\le R_{D,d}\).
Let \(S_J\) be the permutation matrix restoring the original coordinate
order from the order \((J,J^c)\), and set
\[
 X_A:=\begin{pmatrix}I_d\\A\end{pmatrix},\qquad
 H_A:=(I_d+A^\top A)^{-1},\qquad
 P_J(A):=S_JX_AH_AX_A^\top S_J^\top.
\]
The columns of \(U\) and \(S_JX_A\) have the same span, so
\(\Pi=P_J(A)\). For every matrix \(A\), \(P_J(A)\) is itself a
rank-\(d\) orthogonal projector.

On the Frobenius ball \(\|A\|_F,\|B\|_F\le R:=R_{D,d}\), put
\(S:=\sqrt{1+R^2}\). Then
\(\|X_A\|_{\rm op},\|X_B\|_{\rm op}\le S\),
\(\|H_A\|_{\rm op},\|H_B\|_{\rm op}\le1\), and the inverse identity gives
\[
 \|H_A-H_B\|_F
 =\|H_A(B^\top B-A^\top A)H_B\|_F
 \le2R\|A-B\|_F.
\]
Expanding \(X_AH_AX_A^\top-X_BH_BX_B^\top\) by changing one factor
at a time therefore yields
\[
 \|P_J(A)-P_J(B)\|_F
 \le\Lambda_{D,d}\|A-B\|_F,
 \qquad
 \Lambda_{D,d}:=2S+2RS^2.
\]
Apply the ball-covering bound to this \(d(D-d)\)-dimensional Frobenius
ball at radius \(t/\Lambda_{D,d}\), and map its centers through each
\(P_J\). Taking the union over the \(M_{D,d}\) coordinate choices gives
a \(t\)-net in \(\mathscr P_d\), with centers in \(\mathscr P_d\), of
size at most
\[
 M_{D,d}\left(1+\frac{2R_{D,d}\Lambda_{D,d}}{t}\right)^{d(D-d)}.
\]

For each \(\ell\ge2\), a symmetric multilinear map
\(T:(\mathbb R^D)^\ell\to\mathbb R^D\) is determined by its values on
unordered \(\ell\)-tuples of coordinate vectors. There are
\(\binom{D+\ell-1}{\ell}\) such tuples and \(D\) output coordinates,
so this vector space has dimension
\(m_\ell:=D\binom{D+\ell-1}{\ell}\).
The multilinear operator norm is a norm on this space: its vanishing
forces all these coordinate values to vanish. Its ball
\(\{T:\|T\|_{\rm op}\le2L_{\mathcal M}\}\) therefore has a
\(t\)-net, in that same norm and with centers in the ball, of size at most
\((1+4L_{\mathcal M}/t)^{m_\ell}\) by the preceding volume argument.

There are \(s+1\) factors: the offset, the projector, and \(s-1\)
tensor factors. Choose the factor nets above at radius
\(t:=\eta/(s+1)\) and take their Cartesian product. The sum metric gives
total approximation error at most \((s+1)t=\eta\), and all product
centers belong to \(\Theta_{d,s}\). Its cardinality is at most
\[
\begin{aligned}
 &
 (1+2/t)^D\,
 M_{D,d}\left(1+\frac{2R_{D,d}\Lambda_{D,d}}t\right)^{d(D-d)}\times
 \prod_{\ell=2}^s
 \left(1+\frac{4L_{\mathcal M}}t\right)^{m_\ell}.
\end{aligned}
\]
Substituting \(t=\eta/(s+1)\) and using \(\eta\le1\) proves
\eqref{eq:tangent-product-covering}. For \(s=1\), the tensor product is
empty and the same argument uses the two remaining factors. Fixing the
factor nets once for each prescribed radius makes the construction
deterministic.
\end{proof}

We fix the tangent grid and near-minimizer tolerance by
\[
 \eta_{{\rm grid},n}:=n^{-1},\qquad
 \Delta_{{\rm tan},n}:=n^{-3},\qquad
 \mathcal Q_{d,s,n}:=\mathcal Q_{d,s}(n^{-1}),
\]
where the deterministic product net is the one fixed in
Lemma~\ref{lem:tangent-product-net}. Consequently, every
\(q\in\Theta_{d,s}\) has \(q^\sharp\in\mathcal Q_{d,s,n}\) with
\(d_\Theta(q,q^\sharp)\le\eta_{{\rm grid},n}\), and
\begin{equation}\label{eq:N-grid}
 N_{{\rm grid},d,s,n}:=|\mathcal Q_{d,s,n}|
 \le Cn^{m_\Theta(d,s,D)}.
\end{equation}
For \(q\in\mathcal Q_{d,s,n}\), let \(\Pi_q\) denote its projector component.

\smallskip
Let \(c_{\rm Ber}\in(0,1]\) be the fixed Bernstein exponent used in the
concentration bounds below. Define
\[
    \bar\rho_h:=\frac{1}{2\alpha_{\min}+1},
    \qquad
    a_{{\rm tan},n}^{\rm or}
    :=
    \frac{2\alpha_{\max}\bar\rho_h+3}{c_{\rm Ber}}\log(en).
\]
Unless explicitly stated otherwise, Appendix~\ref{app:tangent-local-poly}
uses the proof-level tail parameter
\[
    a_{{\rm tan},n}:=a_{{\rm tan},n}^{\rm or}.
\]
This quantity is used only in concentration bounds.

Also recall
\[
    r_{L_{\mathcal M}}:=(4L_{\mathcal M})^{-1},
    \qquad
    R_{\rm loc}:=c_{\rm geo}(\tau_{0,n}\wedge r_{L_{\mathcal M}}\wedge1),
\]
where \(c_{\rm geo}\) is the constant fixed in
Appendix~\ref{app:feasible-local-formal}. By Lemma~\ref{lem:ass3-consequences}, uniformly over \(k\in[K_n]\), for all
sufficiently large \(n\),
\begin{equation}\label{eq:tlp-basic-scale-consequences}
    n\pi_{k,n}h_k^{d_k}\ge c n^{\rho_{\rm tan}},
    \qquad
    h_k\le R_{\rm loc},
    \qquad
    \sup_{\ell\in[K_n]}h_\ell+2\sigma_{\max,n}
    \le {\delta_{0,n}\over2},
\end{equation}
and
\begin{equation}\label{eq:tlp-rate-consequences}
    \sigma_{k,n}\le a_{\rm sc}h_k^{\alpha_k\vee1},
    \qquad
    \{n\pi_{k,n}h_k^{d_k}\}^{-1}=h_k^{2\alpha_k}.
\end{equation}
We also use, when converting polynomial or exponential remainders into powers
of \(h_k\), the bounds
\[
    n^{-\bar\rho_h}\le h_k\le r_n\le n^{-\rho_h}.
\]

\paragraph{Local notation after conditioning on the query.}
Whenever we condition on
\[
    (Z_{n+1},X^\star_{n+1},\xi_{n+1})=(k,x^\star,\xi_{\rm qry}),
\]
we write
\(\mathbb P_{k,x^\star,\xi_{\rm qry}}\),
\(\mathbb E_{k,x^\star,\xi_{\rm qry}}\),
and use the query-level notation
\begin{equation*}
x:=x^\star+\xi_{\rm qry},\quad
d_x:=d_k,\quad
p_x:=p_k,\quad
s_x:=s_k,\quad
h_x:=h_k,\quad
\Pi^\star:=\Pi_{T_{x^\star}\mathcal M_k}.
\end{equation*}
Let \(U_\star\in\mathbb R^{D\times d_x}\) be a fixed measurable orthonormal
basis of \(T_{x^\star}\mathcal M_k\), so that
\[
    U_\star^\top U_\star=I_{d_x},
    \qquad
    U_\star U_\star^\top=\Pi^\star.
\]
Let \(N_{x^\star}\) be the tangent-normal graph map at \(x^\star\), and define
its coordinate version by
\[
    G_{x^\star}(u):=N_{x^\star}(U_\star u),
    \qquad u\in B_{\mathbb R^{d_x}}(0,r_{L_{\mathcal M}}).
\]
Put $N_{{\rm grid},k,n}:=N_{{\rm grid},d_k,s_k,n}$. For \(i\in[n]\), define
\[
    z_i:=\frac{X_i-x}{h_x},
    \qquad
    \mathsf A_i:=\mathsf A\!\left({\|X_i-x\|_1\over h_x}\right).
\]
Also define
\begin{equation}\label{eq:tangent-error-scales}
b_{\rm tan}:=h_x^{s_x}+{\sigma_{k,n}\over h_x},
\qquad
\zeta_{\rm tan}:=
\left({\log N_{{\rm grid},k,n}+a_{{\rm tan},n}
\over n\pi_{k,n}h_x^{d_x}}\right)^{1/2},
\qquad
\epsilon_{\rm tan}:=b_{\rm tan}+\zeta_{\rm tan}.
\end{equation}
The symbol \(\zeta_{\rm tan}\) is used for the stochastic tangent fluctuation
scale, to avoid confusion with the integer polynomial degree \(s_x\).

\medskip
For \(q\in\Theta_{d_x,s_x}\), define the population and empirical tangent losses
\begin{equation*}
L(q;x):=
\mathbb E_{k,x^\star,\xi_{\rm qry}}
\left[
 h_x^2\mathsf A_i\|R_q(z_i)\|^2
\right],
\qquad
L_n(q;x):={h_x^2\over n}\sum_{i\in[n]}
\mathsf A_i\|R_q(z_i)\|^2.
\end{equation*}
For \(q\in\mathcal Q_{d_x,s_x,n}\), set
\[
    L_q(x):=L_n(q;x),
    \qquad
    L_{\min}(x):=\min_{q\in\mathcal Q_{d_x,s_x,n}}L_q(x),
\]
\[
    \omega_q(x)
    :=
    {\{\Delta_{{\rm tan},n}-(L_q(x)-L_{\min}(x))\}_+
    \over
    \sum_{q'\in\mathcal Q_{d_x,s_x,n}}
    \{\Delta_{{\rm tan},n}-(L_{q'}(x)-L_{\min}(x))\}_+}.
\]
The denominator is at least \(\Delta_{{\rm tan},n}\), because any minimizer of
\(L_q(x)\) over the finite grid contributes exactly \(\Delta_{{\rm tan},n}\). The
corresponding averaged projector and tangent estimator are
\[
    \bar P_x:=\sum_{q\in\mathcal Q_{d_x,s_x,n}}\omega_q(x)\Pi_q,
    \qquad
    \widehat\Pi_x:=\Pi_{(d_x)}(\bar P_x).
\]
For a rank-\(d_x\) projector \(\Pi\), define the regression coordinates,
weights, features, and Gram matrix by
\begin{equation*}
v_i^\Pi:=\Pi z_i,
\quad
W_i^\Pi:=h_x^{-d_x}\mathsf A_i\mathsf K(v_i^\Pi),
\quad
\phi_i^\Pi:=\Phi_{D,p_x}(v_i^\Pi),
\quad
G^\Pi:=\frac{1}{n}\sum_{i\in[n]}W_i^\Pi\phi_i^\Pi(\phi_i^\Pi)^\top.
\end{equation*}
Write
\[
    v_i^\circ:=v_i^{\Pi^\star},
    \qquad
    W_i^\circ:=W_i^{\Pi^\star},
    \qquad
    e_0:=e_{0,D,p_x}.
\]
When an orthonormal basis \(U\) of a rank-\(d\) subspace is used and
\(v=Uu\), the kernel remains the ambient kernel \(\mathsf K(v)\). Thus, for
example, \(\mathsf K(v_i^\circ)=\mathsf K(U_\star U_\star^\top z_i)\).

We use the generic-constant convention in Section~\ref{subsec:notations}.
Any additional dependence on fixed local parameters is stated in the
auxiliary lemmas below.

\subsection{Auxiliary lemmas}

We use the following convention. A probability kernel
\(\{\Lambda_t:t\in T\}\) from a measurable space \(T\) to \(\mathbb R^m\)
means that \(\Lambda_t\) is a probability measure on
\(\mathbb R^m\) for every \(t\in T\), and that
\(t\mapsto \Lambda_t(A)\) is measurable for every Borel set
\(A\subset\mathbb R^m\). If \(B\subset\mathbb R^m\) is Borel and we say that
the kernel is supported on \(B\), we mean \(\Lambda_t(B)=1\) for every
\(t\in T\).
Conditional laws are chosen in versions supported on their stipulated
closed perturbation balls for every conditioning value: an arbitrary version
may be replaced by a point mass at zero on the measurable null set where the
support restriction fails. Integrals against the design law are unchanged.
Pushforward kernels below inherit the stated pointwise support bounds.

\begin{lemma}
\label{lem:poly-feature-facts}
Uniformly over \(d\in[d_{\max}]\) and
\(p\in\{0,\ldots,p_{\max}\}\), the following hold.

\begin{enumerate}
\item[(i)] If \(U\in\mathbb R^{D\times d}\) has orthonormal columns, then there
exists \(B_{U,p}\in\mathbb R^{q_{D,p}\times q_{d,p}}\) such that
\[
\Phi_{D,p}(Uu)=B_{U,p}\Phi_{d,p}(u),
\qquad
 e_{0,D,p}=B_{U,p}e_{0,d,p}.
\]
Moreover, all singular values of \(B_{U,p}\) are bounded above and
below away from zero by fixed constants. Let
\[
    \psi_i:=\Phi_{d,p}(u_i),
    \qquad
    H:=\frac{1}{n}\sum_{i\in[n]}W_i\psi_i\psi_i^\top .
\]
If \(H\) satisfies \(c\pi I\preceq H\preceq c^{-1}\pi I\) for some
\(\pi>0\) and fixed \(c\in(0,1)\), and if
\(\phi_i=\Phi_{D,p}(Uu_i)\), then
\[
G:=\frac{1}{n}\sum_{i\in[n]}W_i\phi_i\phi_i^\top=B_{U,p}HB_{U,p}^\top
\]
satisfies
\[
\lambda_{\min}^+(G)\ge c'\pi,
\qquad
\|G\|_{\rm op}\le {1\over c'}\pi,
\qquad
 e_{0,D,p}\in\operatorname{Im}(G),
\qquad
 e_{0,D,p}^\top G^\dagger G=e_{0,D,p}^\top
\]
for some \(c'\in(0,1)\) depending only on \(c,D,d_{\max},p_{\max}\).

\item[(ii)] Fix \(R>0\). There exist constants
\(c,C,\eta_{\rm poly}>0\), depending only on
\(R,d_{\max},p_{\max}\), such that for every measurable map
\(e:\overline{B_{\mathbb R^d}(0,R)}\to\mathbb R^d\) satisfying
\[
    \|e\|_\infty
    :=
    \sup_{\|t\|\le R}\|e(t)\|
    \le\eta_{\rm poly},
\]
and every \(a\in\mathbb R^{q_{d,p}}\),
\[
c\|a\|^2
\le
\int_{\overline{B_{\mathbb R^d}(0,R)}}
\{a^\top\Phi_{d,p}(t+e(t))\}^2dt
\le
C\|a\|^2 .
\]
The same bounds hold in the following averaged form. Let
\[
    \Lambda:
    \overline{B_{\mathbb R^d}(0,R)}
    \times \mathcal B(\mathbb R^d)
    \to [0,1]
\]
be a probability kernel satisfying
\[
    \Lambda_t\!\left(
    \overline{B_{\mathbb R^d}(0,\eta_{\rm poly})}
    \right)=1
    \qquad
    \text{for every }t\in \overline{B_{\mathbb R^d}(0,R)}.
\]
Then, for every \(a\in\mathbb R^{q_{d,p}}\),
\[
c\|a\|^2
\le
\int_{\overline{B_{\mathbb R^d}(0,R)}}\int_{\mathbb R^d}
\{a^\top\Phi_{d,p}(t+u)\}^2\,d\Lambda_t(u)\,dt
\le
C\|a\|^2 .
\]

\item[(iii)] Fix \(c>0\). If \(U\in\mathbb R^{D\times d}\) has orthonormal columns,
\(m\in\mathcal H^\alpha(B_{\mathbb R^d}(0,ch);L_{\mathcal F})\),
\(p=\lceil\alpha\rceil-1\), and \(0<h\le1\), identify \(m\) with its unique
continuous extension to \(\overline{B_{\mathbb R^d}(0,ch)}\). Then there is
\(w_h\in\mathbb R^{q_{D,p}}\) such that, for \(\|u\|\le ch\),
\[
|m(u)-\Phi_{D,p}(Uu/h)^\top w_h|\le Ch^\alpha,
\qquad
 e_{0,D,p}^\top w_h=m(0).
\]
Moreover, for every fixed \(R_0<\infty\),
\(z\mapsto\Phi_{D,p}(z)^\top w_h\) is \(Ch\)-Lipschitz on
\(B_{\mathbb R^D}(0,R_0)\). The constants in this part may additionally depend
on \(c\), and the Lipschitz constant may also depend on \(R_0\).
\end{enumerate}
\end{lemma}

\begin{proof}
For part \((i)\), restricting each ambient monomial \(z^\nu/\nu!\) to the
subspace \(z=Uu\) gives a polynomial in \(u\) of degree at most \(p\). Collecting
its coefficients in the intrinsic monomial basis defines \(B_{U,p}\) and gives
\(\Phi_{D,p}(Uu)=B_{U,p}\Phi_{d,p}(u)\). Also note that
\(e_{0,D,p}=B_{U,p}e_{0,d,p}\).

The map \(U\mapsto B_{U,p}\) is continuous on the compact Stiefel manifold
\(\{A\in\mathbb R^{D\times d}:A^\top A=I_d\}\). Conversely, since
\(u=U^\top(Uu)\), each intrinsic coordinate \(u_j\) is a linear function of the
ambient coordinate \(z=Uu\). Hence every intrinsic monomial \(u^\mu/\mu!\) can be
written as a linear combination of ambient monomials \(z^\nu/\nu!\),
\(|\nu|\le p\), restricted to \(z=Uu\). Thus there exists \(C_{U,p}\) such that
\[
    \Phi_{d,p}(u)=C_{U,p}\Phi_{D,p}(Uu).
\]
Consequently \(C_{U,p}B_{U,p}=I_{q_{d,p}}\), so \(B_{U,p}\) has full column
rank. Since only finitely many \((d,p)\) are allowed, and by the compactness of Stiefel manifolds, all singular values of
\(B_{U,p}\) are uniformly bounded above and below.

Now let \(\phi_i=B_{U,p}\psi_i\), so
\[
G=B_{U,p}HB_{U,p}^\top .
\]
The bounds
\(\lambda_{\min}^+(G)\ge c'\pi\) and
\(\|G\|_{\rm op}\le {1\over c'}\pi\) directly follow from the spectral bounds
on \(H\) and the uniform singular-value bounds for \(B_{U,p}\). 

 Since \(H\) is
positive definite, \(\operatorname{Im}(G)=\operatorname{Im}(B_{U,p})\), and the
identity \(e_{0,D,p}=B_{U,p}e_{0,d,p}\) gives
\(e_{0,D,p}\in\operatorname{Im}(G)\). For a symmetric positive
semidefinite matrix \(G\), \(GG^\dagger\) is the orthogonal projector onto
\(\operatorname{Im}(G)\), so
\(e_{0,D,p}^\top G^\dagger G=e_{0,D,p}^\top\). By the symmetric property of the Moore-Penrose inverse, we obtain the desired result.

For part \((ii)\), fix \((d,p)\) and set
\[
G_{0,d,p}:=
\int_{B_{\mathbb R^d}(0,R)}
\Phi_{d,p}(t)\Phi_{d,p}(t)^\top\,dt .
\]
For \(a\ne0\), \(t\mapsto a^\top\Phi_{d,p}(t)\) is a nonzero polynomial and
cannot vanish on an open ball. Hence \(G_{0,d,p}\) is positive definite. Since
\((d,p)\) ranges over a finite set, its smallest eigenvalue is uniformly bounded
below, and its largest eigenvalue is uniformly bounded above.

If \(\|e\|_\infty\le1\), then for \(t \in B_{\mathbb R^d}(0,R)\), \(t+e(t)\in B_{\mathbb R^d}(0,R+1)\). On this fixed
ball, the entries of \(\Phi_{d,p}\Phi_{d,p}^\top\) are uniformly Lipschitz over
the allowed finite collection of \((d,p)\). Thus
\[
\left\|
\int \Phi_{d,p}(t+e(t))\Phi_{d,p}(t+e(t))^\top dt
-G_{0,d,p}
\right\|_{\rm op}
\le C\|e\|_\infty .
\]
Choosing \(\eta_{\rm poly}\) small enough gives the lower bound; the upper bound
follows from boundedness of the dictionary on \(B_{\mathbb R^d}(0,R+1)\).

For the averaged assertion, the same perturbation estimate gives
\[
\begin{aligned}
&\left\|
\int_{\overline{B_{\mathbb R^d}(0,R)}}\int_{\mathbb R^d}
\Bigl\{
\Phi_{d,p}(t+u)\Phi_{d,p}(t+u)^\top
-
\Phi_{d,p}(t)\Phi_{d,p}(t)^\top
\Bigr\}
\,d\Lambda_t(u)\,dt
\right\|_{\rm op}  \\
&\qquad\le C\eta_{\rm poly},
\end{aligned}
\]
because \(\Lambda_t\) is supported on
\(\overline{B_{\mathbb R^d}(0,\eta_{\rm poly})}\). After decreasing
\(\eta_{\rm poly}\) if necessary, the same lower and upper eigenvalue bounds
follow.

For part \((iii)\), the H\"older bound makes \(m\) uniformly continuous on
its open ball, giving the asserted unique continuous extension. Let \(T_m\)
be the Taylor polynomial of \(m\) at zero of degree \(p=\lceil\alpha\rceil-1\).
The H\"older remainder bound, extended to the boundary by continuity, gives
\[
|m(u)-T_m(u)|\le Ch^\alpha,
\qquad \|u\|\le ch.
\]
Define the ambient polynomial \(P_h(z):=T_m(hU^\top z)\). Then
\(P_h(z)=\Phi_{D,p}(z)^\top w_h\) for some \(w_h\). For \(z=Uu/h\),
\[
    \Phi_{D,p}(Uu/h)^\top w_h=T_m(u).
\]
Also \(e_{0,D,p}^\top w_h=P_h(0)=m(0)\). Finally, on any fixed ball
\(\|z\|\le R_0\), each nonconstant term in \(P_h\) contains a factor
\(h^{|\mu|}\) with \(|\mu|\ge1\), and the derivatives of \(m\) are bounded by
\(L_{\mathcal F}\). Hence \(\|\nabla_zP_h(z)\|\lesssim h\), with the constant depending on \(R_0\), proving the
Lipschitz claim.
\end{proof}

The next lemma identifies the tangent projector from a polynomial graph loss,
even after a small deterministic shift and a small random perturbation.

\begin{lemma}[Projector identifiability]
\label{lem:projector-identifiability}
Fix \(R>0\) and \(B_\Gamma>0\). There are constants
\(\eta_{\rm id},C>0\), depending on
\(D,d_{\max},s_{\max},R,B_\Gamma,L_{\mathcal M}\), such that the following holds uniformly
over \(d\in[d_{\max}]\) and \(s\in[s_{\max}]\). Let
\(\Pi^\circ=U_\circ U_\circ^\top\), where
\(U_\circ\in\mathbb R^{D\times d}\) has orthonormal columns, and let
\[
\Gamma_0(t):=U_\circ t+\sum_{\ell=2}^sV_\ell\{t^{\otimes \ell}\},
\qquad
\|V_\ell\|_{\rm op}\le B_\Gamma ,
\]
with the sum interpreted as absent when \(s=1\). For
\(\eta\in[0,\eta_{\rm id}]\) and \(\|b\|\le\eta\), let
\[
    \Lambda:
    B_{\mathbb R^d}(0,R)\times\mathcal B(\mathbb R^D)\to[0,1]
\]
be a probability kernel such that
\[
    \Lambda_t\bigl(\{e\in\mathbb R^D:\|e\|\le\eta\}\bigr)=1
    \qquad
    \text{for every }t\in B_{\mathbb R^d}(0,R).
\]
Then, for all \(q,q^\circ\in\Theta_{d,s}\) whose \(q^\circ\)-projector is
\(\Pi^\circ\),
\[
\begin{aligned}
 \|\Pi-\Pi^\circ\|_{\rm op}
 &\le C\{\mathcal D(q,q^\circ)+\eta\},\\
 \mathcal D(q,q^\circ)^2
 &:=\int_{B_{\mathbb R^d}(0,R)}\!\!\int_{\mathbb R^D}
 \|R_q(\Gamma_0(t)+b+e)-R_{q^\circ}(\Gamma_0(t)+b+e)\|^2
 \,d\Lambda_t(e)\,dt.
\end{aligned}
\]
where \(\Pi\) denotes the projector component of \(q\).
\end{lemma}

\begin{proof}
Write
\[
F_0(t):=R_q(\Gamma_0(t))-R_{q^\circ}(\Gamma_0(t)),
\qquad t\in B_{\mathbb R^d}(0,R).
\]
The map \(F_0\) is a \(\mathbb R^D\)-valued polynomial in \(t\). Its degree is
bounded by \(s_{\max}^2\). Hence all such maps \(F_0\) belong to a fixed
finite-dimensional polynomial space.

Coefficient extraction is a bounded linear map on this fixed-dimensional
polynomial space. Consequently,
\begin{equation}\label{eq:projector-id-coeff-bound}
\|D_0F_0\|_{\rm op}
\le
C
\left\{
\int_{B_{\mathbb R^d}(0,R)}\|F_0(t)\|^2\,dt
\right\}^{1/2}.
\end{equation}
Since \(\Gamma_0(t)=U_\circ t+O(\|t\|^2)\), the terms
\(T_\ell\{(\Pi z)^{\otimes\ell}\}\), \(\ell\ge2\), do not contribute to the
linear coefficient. Thus the degree-one coefficient of \(R_q(\Gamma_0(t))\) is
\((I-\Pi)U_\circ\), while that of \(R_{q^\circ}(\Gamma_0(t))\) is
\((I-\Pi^\circ)U_\circ=0\). Therefore
\[
D_0F_0=(I-\Pi)U_\circ=(\Pi^\circ-\Pi)U_\circ.
\]
For rank-\(d\) orthogonal projectors, the principal-angle identity gives
\[
\|(\Pi^\circ-\Pi)U_\circ\|_{\rm op}=\|\Pi-\Pi^\circ\|_{\rm op}.
\]
Combining this with \eqref{eq:projector-id-coeff-bound} yields
\begin{equation}\label{eq:projector-id-unperturbed}
\|\Pi-\Pi^\circ\|_{\rm op}
\le
C
\left\{
\int_{B_{\mathbb R^d}(0,R)}\|F_0(t)\|^2\,dt
\right\}^{1/2}.
\end{equation}

Define
\[
F_{b,e}(t):=
R_q(\Gamma_0(t)+b+e)-R_{q^\circ}(\Gamma_0(t)+b+e).
\]
For all \(q,q^\circ\in\Theta_{d,s}\), all \(\|t\|\le R\), and all
\(\|b\|,\|e\|\le\eta_{\rm id}\),
\begin{equation}\label{eq:projector-id-lipschitz}
\|F_{b,e}(t)-F_0(t)\|
\le C(\|b\|+\|e\|).
\end{equation}
Indeed, the relevant arguments \(\Gamma_0(t)+u\) remain in a fixed Euclidean
ball depending only on \(R,B_\Gamma,s_{\max}\), and the coefficients of
\(q,q^\circ\) range over the compact parameter set \(\Theta_{d,s}\). The
Jacobian in \(z\) of \(R_q(z)-R_{q^\circ}(z)\) is therefore uniformly bounded on
this ball.

Using \eqref{eq:projector-id-lipschitz}, Minkowski's inequality, and the fact
that each \(\Lambda_t\) is a probability measure, we get
\[
\left\{
\int_{B_{\mathbb R^d}(0,R)}\|F_0(t)\|^2dt
\right\}^{1/2}
\le
\left\{
\int_{B_{\mathbb R^d}(0,R)}\!\!\int_{\mathbb R^D}
\|F_{b,e}(t)\|^2d\Lambda_t(e)dt
\right\}^{1/2}
+C\eta .
\]
Substitution into \eqref{eq:projector-id-unperturbed} proves the lemma.
\end{proof}

\subsection{Tangent projector estimation}

In this paper, we suppose $n$ is sufficiently large and use the results in Lemma~\ref{lem:ass3-consequences}: the tangent-projection argument uses the window conditions
\begin{equation}
\label{eq:tangent-window-conditions}
h_x\le1 \wedge  R_{\rm loc},
\quad
h_x+\sigma_{k,n}+\sigma_{\max,n}<\delta_{0,n},
\quad
{\sigma_{k,n}\over h_x}
\le \frac{1}{2\sqrt D}
\le  {1\over2}
\end{equation}
and the grid conditions
\begin{equation}\label{eq:tangent-grid-conditions}
\log N_{{\rm grid},k,n}+a_{{\rm tan},n}\le cn\pi_{k,n}h_x^{d_x},
\quad
\epsilon_{\rm tan}\le c,
\quad
\eta_{{\rm grid},n}\le c\epsilon_{\rm tan}^2,
\quad
\Delta_{{\rm tan},n}\le c\pi_{k,n}h_x^{d_x+2}\epsilon_{\rm tan}^2 .
\end{equation}
Here \(N_{{\rm grid},k,n}\) is defined in \eqref{eq:N-grid},
\(\epsilon_{\rm tan}\) is defined in \eqref{eq:tangent-error-scales}, and
\(c>0\) is a fixed small constant. These conditions are verified uniformly over
components in the proof of the oracle upper bound below.

\begin{lemma}
\label{lem:tlp-local-purity}
For sufficiently large $n$, if \(\|X_i-x\|_1\le h_x\), then
\(Z_i=k\). Moreover, on this event,
\[
    \|X_i^\star-x^\star\|\le h_x+2\sigma_{k,n}\le 2R_{\rm loc}.
\]
\end{lemma}

\begin{proof}
If \(Z_i=\ell\ne k\) and \(\|X_i-x\|_1\le h_x\), then
\[
    d(\mathcal M_\ell,\mathcal M_k)
    \le
    \|X_i^\star-x^\star\|
    \le
    h_x+\|\xi_i\|+\|\xi_{\rm qry}\|
    \le
    h_x+\sigma_{\max,n}+\sigma_{k,n}
    <\delta_{0,n},
\]
contradicting Assumption~\ref{ass:geometry}. Hence \(Z_i=k\). The second claim
then follows from
\[
\|X_i^\star-x^\star\|
\le h_x+\|\xi_i\|+\|\xi_{\rm qry}\|
\le h_x+2\sigma_{k,n}
\le 2h_x
\le 2R_{\rm loc}.
\]
\end{proof}

\begin{lemma}[Tangent estimation event]
\label{lem:tangent-projector-good}
Under \eqref{eq:tangent-window-conditions}--\eqref{eq:tangent-grid-conditions},
there is an event \(\mathcal E_{k,n}^{\rm tan}(x)\) such that
\[
\mathbb P_{k,x^\star,\xi_{\rm qry}}
\{(\mathcal E_{k,n}^{\rm tan}(x))^c\}
\le
C\exp(-c_{\rm Ber}a_{{\rm tan},n}),
\]
and on \(\mathcal E_{k,n}^{\rm tan}(x)\),
\[
    \|\bar P_x-\Pi^\star\|_{\rm op}
    \le
    C\epsilon_{\rm tan}.
\]
Moreover, if the small constant \(c\) in
\eqref{eq:tangent-grid-conditions} is chosen sufficiently small, then on the
same event \(\mathcal E_{k,n}^{\rm tan}(x)\),
\[
\lambda_{d_x}(\bar P_x)-\lambda_{d_x+1}(\bar P_x)\ge {3\over4},
\qquad
\|\widehat\Pi_x-\Pi^\star\|_{\rm op}\le C\epsilon_{\rm tan}.
\]
Consequently,
\[
\mathbb E_{k,x^\star,\xi_{\rm qry}}
\|\widehat\Pi_x-\Pi^\star\|_{\rm op}^2
\le
C\left\{
 h_x^{2s_x}+{\sigma_{k,n}^2\over h_x^2}
 +
 {\log N_{{\rm grid},k,n}+a_{{\rm tan},n}\over n\pi_{k,n}h_x^{d_x}}
 +\exp(-c_{\rm Ber}a_{{\rm tan},n})
\right\}.
\]
\end{lemma}

\begin{proof}
By Lemma~\ref{lem:tlp-local-purity}, the tangent window contains only
observations from component \(k\), and all latent points in the window lie in
the tangent-normal chart at \(x^\star\).

For such observations, write
\[
    X_i^\star=x^\star+U_\star u_i+G_{x^\star}(u_i),
    \qquad
    t_i:=u_i/h_x.
\]

On \(\{\mathsf A_i>0\}\), \(\|X_i-x\|_1\le h_x\). By
Lemma~\ref{lem:tlp-local-purity},
\[
    \|X_i^\star-x^\star\|\le h_x+2\sigma_{k,n}\le Ch_x .
\]
Since \(G_{x^\star}(u_i)\in T_{x^\star}\mathcal M_k^\perp\), 
\[
    \|u_i\|
    =
    \|U_\star^\top(X_i^\star-x^\star)\|
    \le
    \|X_i^\star-x^\star\|
    \le Ch_x .
\]

Hence $\|t_i\|\le R_{\rm tan}$
for a constant \(R_{\rm tan}<\infty\) depending only on the fixed model
constants. The constant \(c_{\rm geo}\) in \(R_{\rm loc}\) is chosen at the
outset so that \(h_xR_{\rm tan}\le r_{L_{\mathcal M}}\) whenever \(h_x\le R_{\rm loc}\).
Thus the following Taylor expansion is taken inside the tangent-normal chart
domain.

\medskip

For \(2\le \ell\le s_x\), regard \(D^\ell G_{x^\star}(0)\) as a symmetric
\(\ell\)-linear map from \((\mathbb R^{d_x})^\ell\) to \(\mathbb R^D\), and
define
\[
    V_{\ell,h_x}
    :=
    {h_x^{\ell-1}\over \ell!}\,
    D^\ell G_{x^\star}(0).
\]
Since \(G_{x^\star}(u)=N_{x^\star}(U_\star u)\), \(U_\star\) is an isometry
from \(\mathbb R^{d_x}\) onto \(T_{x^\star}\mathcal M_k\), and the
tangent-normal graph satisfies
\[
    \|D^\ell N_{x^\star}(0)\|_{\rm op}\le L_{\mathcal M},
    \qquad 2\le \ell\le s_x,
\]
we have
$\|D^\ell G_{x^\star}(0)\|_{\rm op}\le L_{\mathcal M}$, and therefore
\[
    \|V_{\ell,h_x}\|_{\rm op}
    \le
    {L_{\mathcal M}\over \ell!}h_x^{\ell-1}
    \le L_{\mathcal M} .
\]
The \(C^{s_x,1}\) bound on \(G_{x^\star}\) gives, uniformly over
\(\|t\|\le R_{\rm tan}\),
\begin{equation}\label{eq:tangent-graph-taylor}
    h_x^{-1}G_{x^\star}(h_xt)
    =
    \sum_{\ell=2}^{s_x}V_{\ell,h_x}\{t^{\otimes\ell}\}
    +\mathcal R_{h_x}(t),
    \qquad
    \|\mathcal R_{h_x}(t)\|
    \le C h_x^{s_x}.
\end{equation}
Set $\Gamma_{h_x}(t)
    :=
    U_\star t+
    \sum_{\ell=2}^{s_x}V_{\ell,h_x}\{t^{\otimes\ell}\}$.
Then
\begin{equation*}
    z_i
    =
    \Gamma_{h_x}(t_i)-\xi_{\rm qry}/h_x
    +e_{h_x}(t_i,\xi_i),
    \qquad
    \|e_{h_x}(t_i,\xi_i)\|\le C\left(h_x^{s_x}+{\sigma_{k,n}\over h_x}\right).
\end{equation*}
Here \(e_{h_x}(t,\xi)\in\mathbb R^D\) denotes the sum of the graph Taylor
remainder and the rescaled sample noise \(\xi/h_x\). Also
\(\|\xi_{\rm qry}/h_x\|\le1/2\) by
\eqref{eq:tangent-window-conditions}.

\medskip
Define multilinear maps \(T_\ell^\circ:(\mathbb R^D)^\ell\to\mathbb R^D\) by
\[
    T_\ell^\circ(w_1,\ldots,w_\ell)
    :=
    V_{\ell,h_x}(U_\star^\top w_1,
    \ldots,U_\star^\top w_\ell).
\]
Since \(\|U_\star^\top\|_{\rm op}=1\),
$\|T_\ell^\circ\|_{\rm op}
    \le
    \|V_{\ell,h_x}\|_{\rm op}
    \le L_{\mathcal M}
    \le 2L_{\mathcal M}$.
Moreover,
$T_\ell^\circ\{(U_\star t)^{\otimes\ell}\}
    =
    V_{\ell,h_x}\{t^{\otimes\ell}\}$.
Put
\begin{equation*}
 q^\circ=(a^\circ,\Pi^\star,T_2^\circ,\ldots,T_{s_x}^\circ),\qquad
 a^\circ=\frac{(I-\Pi^\star)\xi_{\rm qry}}{h_x}.
\end{equation*}
Since \(\|a^\circ\|\le\sigma_{k,n}/h_x\le1/2\), this belongs to
\(\Theta_{d_x,s_x}\). On \(\{\mathsf A_i>0\}\),
\[
 \Pi^\star z_i=U_\star t_i+
              \frac{\Pi^\star(\xi_i-\xi_{\rm qry})}{h_x},\qquad
 \left\|\frac{\Pi^\star(\xi_i-\xi_{\rm qry})}{h_x}\right\|
 \le\frac{2\sigma_{k,n}}{h_x}.
\]
The normal query shift in \((I-\Pi^\star)z_i\) is cancelled by \(a^\circ\).
The sample perturbation contributes at most \(\sigma_{k,n}/h_x\), and the
polynomial arguments differ from \(U_\star t_i\) by at most
\(2\sigma_{k,n}/h_x\). Uniform Lipschitz bounds for the fixed-degree polynomial
terms on the localized window and \eqref{eq:tangent-graph-taylor} yield
\begin{equation}\label{eq:tangent-oracle-residual}
 \mathsf A_i>0\quad\Longrightarrow\quad
 \|R_{q^\circ}(z_i)\|\le C\left(h_x^{s_x}+\frac{\sigma_{k,n}}{h_x}\right)
 =Cb_{\rm tan}.
\end{equation}

Given this \(q^\circ\), define the loss-induced discrepancy
\begin{equation*}
d_{h_x}(q,q^\circ)^2
:=
{1\over \pi_{k,n}h_x^{d_x}}
\mathbb E_{k,x^\star,\xi_{\rm qry}}\!\left[
\mathbb{I}\{Z_i=k\}\mathsf A_i
\|R_q(z_i)-R_{q^\circ}(z_i)\|^2
\right].
\end{equation*}

Fix a small
constant \(r_{\rm id}>0\), depending only on the fixed regularity constants, and
set
\[
    y_h(t):=x^\star+U_\star h_xt+G_{x^\star}(h_xt),
    \qquad \|t\|\le r_{\rm id}.
\]
For \(\|t\|\le r_{\rm id}\) and
\(\|\xi\|,\|\xi_{\rm qry}\|\le\sigma_{k,n}\),
\[
{\|y_h(t)+\xi-(x^\star+\xi_{\rm qry})\|_1\over h_x}
\le
C\left(r_{\rm id}+h_x r_{\rm id}^2+{\sigma_{k,n}\over h_x}\right).
\]
The constants \(r_{\rm id}\) and \(a_{\rm sc}\) are fixed at the outset so that
the right-hand side is at most a number strictly smaller than one. Since
\(\mathsf A\) is bounded below on a fixed compact subinterval of \([0,1)\),
\[
    \mathsf A\!\left(
    {\|y_h(t)+\xi-(x^\star+\xi_{\rm qry})\|_1\over h_x}
    \right)
    \ge c .
\]

Define the pushforward measure \(\nu_h\) on
\(B_{\mathbb R^{d_x}}(0,r_{\rm id})\) by
\[
    \nu_h(A)
    :=
    \mathbb P\{X_i^\star\in y_h(A)\mid Z_i=k\}
\]
for Borel sets \(A\subset B_{\mathbb R^{d_x}}(0,r_{\rm id})\). The
change of variables \(u=h_xt\) in the tangent-normal chart gives
\[
    \nu_h(A)
    =
    \int_A
    \varrho_{k,n}(y_h(t))J_{x^\star}(h_xt)h_x^{d_x}\,dt,
\]
where
\[
    J_{x^\star}(u)
    :=
    \sqrt{\det\{I+DG_{x^\star}(u)^\top DG_{x^\star}(u)\}}.
\]
The density lower and upper bounds and the local Jacobian bounds imply
that there are constants \(0<c<C<\infty\), independent of
\(n,k,x^\star,\xi_{\rm qry}\), such that for every such Borel set \(A\),
\begin{equation}\label{eq:tangent-chart-measure-bounds}
    c h_x^{d_x}|A|
    \le
    \nu_h(A)
    \le
    C h_x^{d_x}|A|.
\end{equation}

Let \(b=-\xi_{\rm qry}/h_x\). Since the underlying spaces are standard Borel,
choose a regular conditional distribution of \(\xi_i\) given
\((X_i^\star,Z_i)\), supported on
\(\{u\in\mathbb R^D:\|u\|\le\sigma_{k,n}\}\) at every \((y,k)\) with
\(y\in\mathcal M_k\). Such a version is obtained by assigning the point mass
at zero on the \(\mu_{k,n}\)-null set where the support condition fails.
For each \(t\in B_{\mathbb R^{d_x}}(0,r_{\rm id})\), let
\(\Lambda_t\) be the pushforward of this conditional distribution at
\((y_h(t),k)\) under the map
$\xi\mapsto e_{h_x}(t,\xi)$.
Choose \(C_1\ge1\) large enough so that
\[
    \|b\|\le C_1b_{\rm tan},
    \qquad
    \Lambda_t\bigl(\overline{B_{\mathbb R^D}(0,C_1b_{\rm tan})}\bigr)=1
    \quad
    \text{for every }t.
\]
Then
$\Lambda:
    B_{\mathbb R^{d_x}}(0,r_{\rm id})
    \times\mathcal B(\mathbb R^D)\to[0,1]$ is a probability kernel. Therefore, using
\(\mathsf A_i\ge c\), \(\mathbb P(Z_i=k)=\pi_{k,n}\), and
\eqref{eq:tangent-chart-measure-bounds},
\begin{equation}\label{eq:tangent-dh-mass-lower}
\begin{aligned}
d_{h_x}(q,q^\circ)^2
&\ge
{c\over h_x^{d_x}}
\int_{B_{\mathbb R^{d_x}}(0,r_{\rm id})}\int_{\mathbb R^D}
\|R_q(\Gamma_{h_x}(t)+b+e)
      -R_{q^\circ}(\Gamma_{h_x}(t)+b+e)\|^2
\,d\Lambda_t(e)\,d\nu_h(t) \\
&\ge
c
\int_{B_{\mathbb R^{d_x}}(0,r_{\rm id})}\int_{\mathbb R^D}
\|R_q(\Gamma_{h_x}(t)+b+e)
      -R_{q^\circ}(\Gamma_{h_x}(t)+b+e)\|^2
\,d\Lambda_t(e)\,dt .
\end{aligned}
\end{equation}
Set \(\eta_0:=C_1b_{\rm tan}\). Since
\(b_{\rm tan}\le\epsilon_{\rm tan}\), the smallness condition in
\eqref{eq:tangent-grid-conditions} is chosen so that
\(\eta_0\le\eta_{\rm id}\). Since \(\|V_{\ell,h_x}\|_{\rm op}\le L_{\mathcal M}\),
Lemma~\ref{lem:projector-identifiability} applies with
\[
    R=r_{\rm id},
    \qquad
    B_\Gamma=L_{\mathcal M},
    \qquad
    \Gamma_0=\Gamma_{h_x},
    \qquad
    \eta=\eta_0 .
\]
Together with \eqref{eq:tangent-dh-mass-lower}, it gives
\begin{equation}\label{eq:tangent-projector-id-dh}
\|\Pi-\Pi^\star\|_{\rm op}
\le C\{d_{h_x}(q,q^\circ)+b_{\rm tan}\}.
\end{equation}

Local purity, the density upper bound, and the local volume bound in
Lemma~\ref{lem:reach-local-facts-main} give
\[
 \mathbb E_{k,x^\star,\xi_{\rm qry}}\mathsf A_i
 \le \mathbb P_{k,x^\star,\xi_{\rm qry}}\{\mathsf A_i>0\}
 \le C\pi_{k,n}h_x^{d_x}.
\]
The population quadratic loss is coercive in the same discrepancy. Expanding
squares and using this mass bound, \eqref{eq:tangent-oracle-residual},
Cauchy's inequality, and Young's inequality,
\begin{equation}\label{eq:tangent-pop-coercivity}
    L(q;x)-L(q^\circ;x)
    \ge
    c\pi_{k,n}h_x^{d_x+2}d_{h_x}(q,q^\circ)^2
    -
    C\pi_{k,n}h_x^{d_x+2}b_{\rm tan}^2 .
\end{equation}

Conditioning on the query, \(q^\circ\) is fixed. For each grid point \(q\), the
summands in \(L_n(q;x)-L_n(q^\circ;x)\) are uniformly bounded by \(Ch_x^2\).
Moreover, on \(\{\mathsf A_i>0\}\), both \(R_q(z_i)\) and
\(R_{q^\circ}(z_i)\) are uniformly bounded, and
\[
\left|
\|R_q(z_i)\|^2-\|R_{q^\circ}(z_i)\|^2
\right|
\le
C\|R_q(z_i)-R_{q^\circ}(z_i)\|.
\]
Together with the definition of \(d_{h_x}\) and
\eqref{eq:tangent-oracle-residual}, this gives the variance bound
\[
    \operatorname{Var}_{k,x^\star,\xi_{\rm qry}}
    \left[
    h_x^2\mathsf A_i
    \{\|R_q(z_i)\|^2-\|R_{q^\circ}(z_i)\|^2\}
    \right]
    \le
    C\pi_{k,n}h_x^{d_x+4}
    \{d_{h_x}(q,q^\circ)+b_{\rm tan}\}^2.
\]
Apply Bernstein's inequality with
\(u=\log N_{{\rm grid},k,n}+a_{{\rm tan},n}\), enlarging the fixed
threshold constant as needed. A union bound over the
\(N_{{\rm grid},k,n}\) candidates uses
\(N_{{\rm grid},k,n}e^{-u}=e^{-a_{{\rm tan},n}}\)
and gives an event on which, uniformly over
\(q\in\mathcal Q_{d_x,s_x,n}\),
\begin{equation}\label{eq:tangent-loss-concentration}
\left|
\{L_n(q;x)-L_n(q^\circ;x)\}
-
\{L(q;x)-L(q^\circ;x)\}
\right|
\le
C\pi_{k,n}h_x^{d_x+2}
\{(d_{h_x}(q,q^\circ)+b_{\rm tan})\zeta_{\rm tan}+\zeta_{\rm tan}^2\}.
\end{equation}
The same Bernstein argument gives
\begin{equation}\label{eq:tangent-local-count-bound}
    {1\over n}\sum_{i\in[n]}\mathsf A_i
    \le
    C\pi_{k,n}h_x^{d_x}.
\end{equation}
Let \(\mathcal E_{k,n}^{\rm tan}(x)\) be the intersection of the events in
\eqref{eq:tangent-loss-concentration} and
\eqref{eq:tangent-local-count-bound}. Under
\eqref{eq:tangent-grid-conditions},
\[
\mathbb P_{k,x^\star,\xi_{\rm qry}}
\{(\mathcal E_{k,n}^{\rm tan}(x))^c\}
\le
C\exp(-c_{\rm Ber}a_{{\rm tan},n}).
\]

Let \(q^\sharp\) be a grid point within \(\eta_{{\rm grid},n}\) of \(q^\circ\).
Since \(q\mapsto R_q(z)\) is Lipschitz uniformly on the fixed tangent window,
\[
    \|R_{q^\sharp}(z_i)-R_{q^\circ}(z_i)\|
    \le C\eta_{{\rm grid},n}
    \qquad\text{whenever }\mathsf A_i>0.
\]
Using \eqref{eq:tangent-oracle-residual} and
\eqref{eq:tangent-local-count-bound},
\[
\begin{aligned}
    L_n(q^\sharp;x)-L_n(q^\circ;x)
    &\le
    C h_x^2
    {1\over n}\sum_{i\in[n]}\mathsf A_i
    \{b_{\rm tan}\eta_{{\rm grid},n}+\eta_{{\rm grid},n}^2\} \\
    &\le
    C\pi_{k,n}h_x^{d_x+2}
    \{b_{\rm tan}\eta_{{\rm grid},n}+\eta_{{\rm grid},n}^2\}
    \le
    C\pi_{k,n}h_x^{d_x+2}\epsilon_{\rm tan}^2,
\end{aligned}
\]
where the last inequality uses
\(\eta_{{\rm grid},n}\le c\epsilon_{\rm tan}^2\) and
\(\epsilon_{\rm tan}\le c\).

If \(q\) is active, \(\omega_q(x)>0\), then
\(L_n(q;x)\le L_{\min}(x)+\Delta_{{\rm tan},n}\). Since
\(L_{\min}(x)\le L_n(q^\sharp;x)\), the preceding display and the condition on
\(\Delta_{{\rm tan},n}\) imply
\begin{equation}\label{eq:tangent-active-empirical-excess}
    L_n(q;x)-L_n(q^\circ;x)
    \le
    C\pi_{k,n}h_x^{d_x+2}\epsilon_{\rm tan}^2 .
\end{equation}

Put \(D_q:=d_{h_x}(q,q^\circ)\) and
\(\mathfrak m_x:=\pi_{k,n}h_x^{d_x+2}\). Combining
\eqref{eq:tangent-active-empirical-excess},
\eqref{eq:tangent-loss-concentration}, and
\eqref{eq:tangent-pop-coercivity} gives
\[
    \mathfrak m_xD_q^2
    \lesssim
    \mathfrak m_x\{b_{\rm tan}^2+\epsilon_{\rm tan}^2
    +(D_q+b_{\rm tan})\zeta_{\rm tan}+\zeta_{\rm tan}^2\}.
\]
Since \(b_{\rm tan}\le\epsilon_{\rm tan}\) and
\(\zeta_{\rm tan}\le\epsilon_{\rm tan}\), Young's inequality yields
\(D_q\le C\epsilon_{\rm tan}\). Hence, by
\eqref{eq:tangent-projector-id-dh},
\[
    \|\Pi_q-\Pi^\star\|_{\rm op}\le C\epsilon_{\rm tan}
\]
for every active grid point \(q\). Averaging over the nonnegative weights gives, on \(\mathcal E_{k,n}^{\rm tan}(x)\), 
\[
    \|\bar P_x-\Pi^\star\|_{\rm op}\le C_0\epsilon_{\rm tan}.
\]

\medskip
Since \(\Pi^\star\) is a rank-\(d_x\) projector, Weyl's inequality gives
\[
\lambda_{d_x}(\bar P_x)\ge1-C_0\epsilon_{\rm tan},
\qquad
\lambda_{d_x+1}(\bar P_x)\le C_0\epsilon_{\rm tan}.
\]
Choosing the small constant in \eqref{eq:tangent-grid-conditions} so that
\(C_0\epsilon_{\rm tan}\le1/8\), we get the spectral gap lower bound
\(3/4\). The Davis--Kahan sin-\(\Theta\) theorem
\citep[Theorem~1]{yu2015useful}, originally due to
\citet{davis1970rotation}, then gives
\[
    \|\widehat\Pi_x-\Pi^\star\|_{\rm op}
    \le
    C\|\bar P_x-\Pi^\star\|_{\rm op}
    \le
    C\epsilon_{\rm tan}.
\]

For the second moment, on \(\mathcal E_{k,n}^{\rm tan}(x)\) the squared error is
at most \(C\epsilon_{\rm tan}^2\), while on the complement it is at most one.
The displayed probability bound therefore gives
\[
\mathbb E_{k,x^\star,\xi_{\rm qry}}
\|\widehat\Pi_x-\Pi^\star\|_{\rm op}^2
\lesssim
\epsilon_{\rm tan}^2+\exp(-c_{\rm Ber}a_{{\rm tan},n}),
\]
which is the asserted bound after expanding \(\epsilon_{\rm tan}\).
\end{proof}

\subsection{Regression stability}

The regression design argument also uses the window condition in
\eqref{eq:tangent-window-conditions}.

\begin{lemma}[Uniform regression stability]
\label{lem:regression-design-stability}
Under the standing assumptions, including Assumption~\ref{ass:feasible-local},
and \eqref{eq:tangent-window-conditions}, there exist \(\delta_\star>0\) and
an event \(\mathcal E_{k,n}^{\rm reg}(x)\) such that
\[
\mathbb P_{k,x^\star,\xi_{\rm qry}}
\{(\mathcal E_{k,n}^{\rm reg}(x))^c\}
\le C\exp(-cn\pi_{k,n}h_x^{d_x}),
\]
and on \(\mathcal E_{k,n}^{\rm reg}(x)\), for every rank-\(d_x\) projector
\(\Pi\) with \(\|\Pi-\Pi^\star\|_{\rm op}\le\delta_\star\),
\[
\lambda_{\min}^+(G^\Pi)\ge c\pi_{k,n},
\qquad
\|G^\Pi\|_{\rm op}\le C\pi_{k,n},
\qquad
\|(G^\Pi)^\dagger\|_{\rm op}\le C\pi_{k,n}^{-1},
\]
\[
 e_0\in\operatorname{Im}(G^\Pi),
\qquad
 e_0^\top(G^\Pi)^\dagger G^\Pi=e_0^\top,
\]
and
\[
{1\over n}\sum_{i\in[n]}W_i^\Pi\le C\pi_{k,n},
\qquad
{1\over n}\sum_{i\in[n]}(W_i^\Pi)^2\le C\pi_{k,n}h_x^{-d_x}.
\]
Moreover, on \(\{\|X_i-x\|_1\le h_x\}\),
\[
\|v_i^\Pi-v_i^\circ\|\le C\|\Pi-\Pi^\star\|_{\rm op},
\qquad
|W_i^\Pi-W_i^\circ|\le Ch_x^{-d_x}\|\Pi-\Pi^\star\|_{\rm op}.
\]
\end{lemma}

\begin{proof}
Here and below in this proof, \(\mathbb E\) denotes
\(\mathbb E_{k,x^\star,\xi_{\rm qry}}\). By
Lemma~\ref{lem:tlp-local-purity}, the regression window contains only component
\(k\) and is contained in the local chart at \(x^\star\). Fix \(r_{\rm in}>0\)
sufficiently small, depending only on the fixed model constants. With this
choice, let \(\eta_{\rm poly}\) be the constant in
Lemma~\ref{lem:poly-feature-facts}\textnormal{(ii)}. The scale constant
\(a_{\rm sc}\) is chosen at the outset so that \(2a_{\rm sc}\le\eta_{\rm poly}\)
and so that the cutoff and kernel lower bounds below hold.

For \(t\in B_{\mathbb R^{d_x}}(0,r_{\rm in})\), write
\[
    y_h(t):=x^\star+U_\star h_xt+G_{x^\star}(h_xt).
\]
The tangent-normal change of variables gives, for every measurable
\(A\subset B_{\mathbb R^{d_x}}(0,r_{\rm in})\),
\begin{equation*}
\mathbb P\{X_i^\star\in y_h(A)\mid Z_i=k\}
=
\int_A
\varrho_{k,n}(y_h(t))J_{x^\star}(h_xt)h_x^{d_x}\,dt,
\end{equation*}
where
$J_{x^\star}(u)
    :=
    \sqrt{\det\{I+DG_{x^\star}(u)^\top DG_{x^\star}(u)\}}$.
The density lower and upper bounds and the local Jacobian bounds imply
that there are constants \(0<c<C<\infty\), independent of
\(n,k,x^\star,\xi_{\rm qry}\), such that for every Borel
\(A\subset B_{\mathbb R^{d_x}}(0,r_{\rm in})\),
\begin{equation*}
    c h_x^{d_x}|A|
    \le
    \mathbb P\{X_i^\star\in y_h(A)\mid Z_i=k\}
    \le
    C h_x^{d_x}|A|.
\end{equation*}

\medskip
For \(X_i^\star=y_h(t)\), define
$u_i^\circ:=U_\star^\top z_i$. Since \(G_{x^\star}(h_xt)\) is normal to
\(T_{x^\star}\mathcal M_k\),
\begin{equation*}
    u_i^\circ
    =
    t+e_h(t,\xi_i),
    \qquad
    e_h(t,\xi_i)
    :=
    {U_\star^\top(\xi_i-\xi_{\rm qry})\over h_x}.
\end{equation*}
Moreover
\[
    \|e_h(t,\xi_i)\|
    \le {2\sigma_{k,n}\over h_x}
    \le 2a_{\rm sc}h_x^{(\alpha_k\vee1)-1}
    \le 2a_{\rm sc}.
\]

For \(\|t\|\le r_{\rm in}\) and
\(\|\xi\|,\|\xi_{\rm qry}\|\le\sigma_{k,n}\),
\[
{\|y_h(t)+\xi-(x^\star+\xi_{\rm qry})\|_1\over h_x}
\le
C\left(r_{\rm in}+h_x r_{\rm in}^2+{\sigma_{k,n}\over h_x}\right),
\]
and, using \(u_i^\circ=t+U_\star^\top(\xi-\xi_{\rm qry})/h_x\),
\[
    \|U_\star u_i^\circ\|_1
    \le
    C\left(r_{\rm in}+{\sigma_{k,n}\over h_x}\right).
\]
The constants \(r_{\rm in}\) and \(a_{\rm sc}\) are fixed so that both
right-hand sides are bounded by constants on which \(\mathsf A\) and
\(\mathsf K\) are bounded below. Hence
\[
    \mathsf A_i\ge c,
    \qquad
    \mathsf K(v_i^\circ)=\mathsf K(U_\star u_i^\circ)\ge c.
\]

Set \(\psi_i^\circ:=\Phi_{d_x,p_x}(u_i^\circ)\), and define the reduced
true-tangent Gram matrix
\[
H^\circ:=
{1\over n}\sum_{i\in[n]}
\mathbb{I}\{Z_i=k\}
 h_x^{-d_x}\mathsf A_i\mathsf K(U_\star u_i^\circ)
\psi_i^\circ(\psi_i^\circ)^\top .
\]

Use the supported version of the regular conditional distribution of
\(\xi_i\) given \((X_i^\star,Z_i)\) chosen in the tangent-projector proof. For each
\(t\in B_{\mathbb R^{d_x}}(0,r_{\rm in})\), let \(\Lambda_t\) be the pushforward
of this conditional distribution at \((y_h(t),k)\) under the map
$\xi\mapsto {U_\star^\top(\xi-\xi_{\rm qry})\over h_x}$.
Then $\Lambda:
    B_{\mathbb R^{d_x}}(0,r_{\rm in})
    \times\mathcal B(\mathbb R^{d_x})
    \to[0,1]$ is a probability kernel. Moreover,
\[
    \Lambda_t\!\left(
    \overline{B_{\mathbb R^{d_x}}(0,\eta_{\rm poly})}
    \right)=1
    \qquad
    \text{for every }t\in B_{\mathbb R^{d_x}}(0,r_{\rm in}),
\]
because
\[
    \left\|{U_\star^\top(\xi-\xi_{\rm qry})\over h_x}\right\|
    \le {\|\xi\|+\|\xi_{\rm qry}\|\over h_x}
    \le {2\sigma_{k,n}\over h_x}
    \le 2a_{\rm sc}
    \le \eta_{\rm poly}.
\]
Using the lower local mass bound, the lower bounds on \(\mathsf A_i\) and
\(\mathsf K\), and Lemma~\ref{lem:poly-feature-facts}\textnormal{(ii)} applied
to the kernel \(\Lambda\), for every \(a\in\mathbb R^{q_{d_x,p_x}}\),
\[
    a^\top\mathbb EH^\circ a
    \ge
    c\pi_{k,n}
    \int_{B_{\mathbb R^{d_x}}(0,r_{\rm in})}\int_{\mathbb R^{d_x}}
    \{a^\top\Phi_{d_x,p_x}(t+u)\}^2
    d\Lambda_t(u)dt
    \ge c\pi_{k,n}\|a\|^2.
\]

The upper bound is obtained on the full regression window. If
\(\mathsf A_i>0\), then \(\|X_i-x\|_1\le h_x\). By
Lemma~\ref{lem:tlp-local-purity}, \(Z_i=k\) on this window and
\(\|X_i^\star-x^\star\|\le Ch_x\). Hence, in the local chart,
\(\|u_i^\circ\|\le C\), uniformly over all observations contributing to
\(H^\circ\). Since \(\mathsf A\), \(\mathsf K\), and
\(\Phi_{d_x,p_x}\) are uniformly bounded on this fixed coordinate window, and
since the local mass upper bound gives
\[
    \mathbb P_{k,x^\star,\xi_{\rm qry}}\{\mathsf A_i>0,\ Z_i=k\}
    \le C\pi_{k,n} h_x^{d_x},
\]
we have
$a^\top\mathbb EH^\circ a\le C\pi_{k,n}\|a\|^2$. Thus
\begin{equation*}
    c\pi_{k,n}I\preceq \mathbb EH^\circ\preceq C\pi_{k,n}I.
\end{equation*}

\medskip
Define
$C_i^\circ:=\mathbb{I}\{Z_i=k\}h_x^{-d_x}\mathsf A_i$. Let \(\mathcal E_{k,n}^{\rm reg}(x)\) be the event on which
\begin{equation}\label{eq:reg-empirical-true-gram-event}
 c\pi_{k,n}I\preceq H^\circ\preceq C\pi_{k,n}I,
 \qquad
 {1\over n}\sum_{i\in[n]} C_i^\circ\le C\pi_{k,n},
 \qquad
 {1\over n}\sum_{i\in[n]} (C_i^\circ)^2\le C\pi_{k,n}h_x^{-d_x}.
\end{equation}
Since \(q_{d_x,p_x}\) is uniformly bounded, entrywise Bernstein bounds imply
operator-norm concentration and give
\[
\mathbb P_{k,x^\star,\xi_{\rm qry}}
\{(\mathcal E_{k,n}^{\rm reg}(x))^c\}
\le
C\exp(-cn\pi_{k,n}h_x^{d_x}).
\]

On this event we extend the bounds deterministically from \(\Pi^\star\) to all
nearby rank-\(d_x\) projectors. If
\(\|\Pi-\Pi^\star\|_{\rm op}\le\delta_\star\), then on
\(\{\|X_i-x\|_1\le h_x\}\),
\[
    \|v_i^\Pi-v_i^\circ\|
    =
    \|(\Pi-\Pi^\star)z_i\|
    \le
    C\|\Pi-\Pi^\star\|_{\rm op},
\]
and, since \(\mathsf K\) is Lipschitz,
\[
    |W_i^\Pi-W_i^\circ|
    \le
    Ch_x^{-d_x}\|\Pi-\Pi^\star\|_{\rm op}.
\]

Choose an orthonormal basis \(U_\Pi\) of \(\operatorname{Im}(\Pi)\) aligned with
\(U_\star\). There exists an orthogonal matrix \(R_\Pi\in\mathbb R^{d_x\times d_x}\)
such that
\[
    \|U_\Pi-U_\star R_\Pi\|_{\rm op}
    \le
    C\|\Pi-\Pi^\star\|_{\rm op}.
\]
Put \(u_i^\Pi:=U_\Pi^\top z_i\), and define
\[
    H^\Pi
    :=
    {1\over n}\sum_{i\in[n]}
    W_i^\Pi
    \Phi_{d_x,p_x}(u_i^\Pi)
    \Phi_{d_x,p_x}(u_i^\Pi)^\top .
\]
Let \(S_{R_\Pi,p_x}\) be the fixed finite-dimensional rotation matrix satisfying
\[
    \Phi_{d_x,p_x}(R_\Pi^\top u)=S_{R_\Pi,p_x}\Phi_{d_x,p_x}(u).
\]
Its singular values are uniformly bounded above and below over all orthogonal
\(R_\Pi\). Therefore, with
$\widetilde H^\circ_\Pi
    :=
    S_{R_\Pi,p_x}H^\circ S_{R_\Pi,p_x}^\top$,
\eqref{eq:reg-empirical-true-gram-event} implies
$c\pi_{k,n} I
    \preceq
    \widetilde H^\circ_\Pi
    \preceq
    C\pi_{k,n} I$.

\medskip
Set
$\widetilde\psi_i^\circ
    :=
    \Phi_{d_x,p_x}(R_\Pi^\top u_i^\circ)$.
By local purity, \(\mathbb{I}\{Z_i=k\}W_i^\circ=W_i^\circ\). Hence
\[
    \widetilde H^\circ_\Pi
    =
    {1\over n}\sum_{i\in[n]}
    W_i^\circ
    \widetilde\psi_i^\circ(\widetilde\psi_i^\circ)^\top .
\]
On the support of the weights,
\[
    \|u_i^\Pi-R_\Pi^\top u_i^\circ\|
    =
    \|(U_\Pi-U_\star R_\Pi)^\top z_i\|
    \le
    C\|\Pi-\Pi^\star\|_{\rm op}.
\]
Since
\[
W_i^\Pi-W_i^\circ
=
h_x^{-d_x}\mathsf A_i
\{\mathsf K(v_i^\Pi)-\mathsf K(v_i^\circ)\},
\]
the Lipschitz property of \(\mathsf K\), local purity on \(\{\mathsf A_i>0\}\),
and the bound \(\|v_i^\Pi-v_i^\circ\|\le C\|\Pi-\Pi^\star\|_{\rm op}\) imply
\[
    |W_i^\Pi-W_i^\circ|
    \le
    C C_i^\circ\|\Pi-\Pi^\star\|_{\rm op}.
\]
Using the boundedness and Lipschitzness of the polynomial dictionary,
\[
\left\|
W_i^\Pi\Phi_{d_x,p_x}(u_i^\Pi)\Phi_{d_x,p_x}(u_i^\Pi)^\top
-
W_i^\circ\widetilde\psi_i^\circ(\widetilde\psi_i^\circ)^\top
\right\|_{\rm op}\le
C C_i^\circ\|\Pi-\Pi^\star\|_{\rm op}.
\]
Averaging over \(i\) and using
\({n}^{-1}\sum_{i\in[n]} C_i^\circ\le C\pi_{k,n}\), we obtain
\begin{equation}\label{eq:H-Pi-perturbation}
    \|H^\Pi-\widetilde H^\circ_\Pi\|_{\rm op}
    \le
    C\pi_{k,n}\|\Pi-\Pi^\star\|_{\rm op}.
\end{equation}
Taking \(\delta_\star>0\) sufficiently small and using Weyl's inequality yields
\[
    c\pi_{k,n}I\preceq H^\Pi\preceq C\pi_{k,n}I.
\]

\medskip

Since \(v_i^\Pi=U_\Pi u_i^\Pi\) and
$\phi_i^\Pi
    =
    \Phi_{D,p_x}(U_\Pi u_i^\Pi)
    =
    B_{U_\Pi,p_x}\Phi_{d_x,p_x}(u_i^\Pi)$,
Lemma~\ref{lem:poly-feature-facts}\textnormal{(i)} gives the spectral,
pseudoinverse, and image identities for \(G^\Pi\).

\medskip
Finally,
$W_i^\Pi
    \le
    Ch_x^{-d_x}\mathbb{I}\{\|X_i-x\|_1\le h_x\}$.
By local purity, \(\mathsf A_i=\mathbb{I}\{Z_i=k\}\mathsf A_i\) on the support
of the weights. Hence
\[
    W_i^\Pi\le C C_i^\circ,
    \qquad
    (W_i^\Pi)^2\le C h_x^{-d_x}C_i^\circ .
\]
The displayed bounds for \(C_i^\circ\) therefore give
\[
{1\over n}\sum_{i\in[n]} W_i^\Pi\le C\pi_{k,n},
\qquad
{1\over n}\sum_{i\in[n]} (W_i^\Pi)^2\le C\pi_{k,n}h_x^{-d_x}.
\]
\end{proof}

\subsection{Proof of Theorem~\ref{thm:oracle-upper}}
\label{app:tangent-local-poly-upper}

\begin{proof}
The scale consequences
\eqref{eq:tlp-basic-scale-consequences}--\eqref{eq:tlp-rate-consequences}, together with Lemma~\ref{lem:ass3-consequences},
verify the window hypotheses \eqref{eq:tangent-window-conditions}. We next
verify the grid conditions \eqref{eq:tangent-grid-conditions}.
The selected product grids satisfy \eqref{eq:N-grid}. Since
\(d\in[d_{\max}]\) and \(s\in[s_{\max}]\), their exponents and
prefactors are bounded in terms of fixed model parameters, and hence
\[
    \log N_{{\rm grid},k,n}\le C\log(en).
\]
Also \(a_{{\rm tan},n}\lesssim\log(en)\), while
\[
    n\pi_{k,n}h_x^{d_x}\ge c n^{\rho_{\rm tan}}.
\]
Hence
\[
    \log N_{{\rm grid},k,n}+a_{{\rm tan},n}
    \le
    cn\pi_{k,n}h_x^{d_x}
\]
for all sufficiently large \(n\). The condition \(\epsilon_{\rm tan}\le c\)
follows from
\[
    \sigma_{k,n}/h_x\le a_{\rm sc}h_x^{(\alpha_k\vee1)-1},
\]
the small fixed choice of \(a_{\rm sc}\), and the uniform convergence of
\(\zeta_{\rm tan}\) to zero.

Next,
\[
\zeta_{\rm tan}^2
=
{\log N_{{\rm grid},k,n}+a_{{\rm tan},n}\over n\pi_{k,n}h_x^{d_x}}
\ge
c{\log(en)\over n},
\]
because \(n\pi_{k,n}h_x^{d_x}\le n\) and
\(a_{{\rm tan},n}\asymp\log(en)\). Therefore
\[
    \eta_{{\rm grid},n}=n^{-1}\le c\epsilon_{\rm tan}^2
\]
for all sufficiently large \(n\). Moreover,
\[
\pi_{k,n}h_x^{d_x+2}\epsilon_{\rm tan}^2
\ge
\pi_{k,n}h_x^{d_x+2}\zeta_{\rm tan}^2
=
{h_x^2\{\log N_{{\rm grid},k,n}+a_{{\rm tan},n}\}\over n}
\ge
c n^{-1-2\bar\rho_h}\log(en).
\]
Since \(\bar\rho_h<1\),
\[
    \Delta_{{\rm tan},n}=n^{-3}
    \le
    c\pi_{k,n}h_x^{d_x+2}\epsilon_{\rm tan}^2
\]
for all sufficiently large \(n\). Thus
\eqref{eq:tangent-grid-conditions} holds uniformly over \(k\).

Fix \(P,f\), and condition on
\[
    (Z_{n+1},X^\star_{n+1},\xi_{n+1})=(k,x^\star,\xi_{\rm qry}).
\]
Use the local notation from Subsection~\ref{app:tangent-local-poly-estimator},
and set
\[
    \delta_x:=\|\widehat\Pi_x-\Pi^\star\|_{\rm op}.
\]
After reducing the small constant in \eqref{eq:tangent-grid-conditions},
Lemma~\ref{lem:tangent-projector-good} implies
\[
    \mathcal E_{k,n}^{\rm tan}(x)\subset\{\delta_x\le\delta_\star\}.
\]
Define
\[
    \mathcal E_{k,n}(x)
    :=
    \mathcal E_{k,n}^{\rm reg}(x)\cap\{\delta_x\le\delta_\star\}.
\]
The preceding inclusion implies
\(\{\delta_x>\delta_\star\}\subset(\mathcal E_{k,n}^{\rm tan}(x))^c\). Therefore,
by Lemmas~\ref{lem:tangent-projector-good} and
\ref{lem:regression-design-stability},
\[
\mathbb P_{k,x^\star,\xi_{\rm qry}}
\{\mathcal E_{k,n}(x)^c\}
\lesssim
\exp(-cn\pi_{k,n}h_x^{d_x})
+\exp(-c_{\rm Ber}a_{{\rm tan},n}).
\]
For every fixed \(M>0\), with an implicit constant that may depend on \(M\),
\[
    \exp(-cn\pi_{k,n}h_x^{d_x})
    =
    \exp(-ch_x^{-2\alpha_k})
    \lesssim h_x^M .
\]
Taking \(M=2\alpha_{\max}+1\) and using \(h_x\le r_n\to0\) gives
\[
    \exp(-cn\pi_{k,n}h_x^{d_x})
    \le
    C h_x^{2\alpha_k+1}
    =o(h_x^{2\alpha_k}).
\]
Moreover, by the definition of \(a_{{\rm tan},n}\) and
\(h_x\ge n^{-\bar\rho_h}\),
\[
    \exp(-c_{\rm Ber}a_{{\rm tan},n})
    \le
    n^{-2\alpha_{\max}\bar\rho_h-3}
    \le
    n^{-3}h_x^{2\alpha_k}
    =o(h_x^{2\alpha_k}).
\]
Both little-\(o\) bounds are uniform over components, admissible models,
and conditioned query values. Thus
\begin{equation}\label{eq:main-good-event-failure}
    \mathbb P_{k,x^\star,\xi_{\rm qry}}
    \{\mathcal E_{k,n}(x)^c\}
    =o(h_x^{2\alpha_k}).
\end{equation}

On \(\mathcal E_{k,n}(x)\), write
\[
    \widehat v_i:=v_i^{\widehat\Pi_x},
    \qquad
    \widehat W_i:=W_i^{\widehat\Pi_x},
    \qquad
    \widehat\phi_i:=\phi_i^{\widehat\Pi_x},
    \qquad
    \widehat G_x:=G^{\widehat\Pi_x}.
\]
Define
\[
    \widehat b_x
    :=
    {1\over n}\sum_{i\in[n]}\widehat W_i\widehat\phi_iY_i,
    \qquad
    \widehat w_{\rm plug}^\Pi(x)
    :=
    \widehat G_x^\dagger \widehat b_x.
\]
The truncated plug-in estimator is
\[
    \widehat f_{\rm plug}^\Pi(x)
    :=
    (-L_{\mathcal F})\vee
    \{e_0^\top\widehat w_{\rm plug}^\Pi(x)\}
    \wedge L_{\mathcal F}.
\]
Lemma~\ref{lem:regression-design-stability} gives
\[
\lambda_{\min}^+(\widehat G_x)\ge c\pi_{k,n},
\qquad
\|(\widehat G_x)^\dagger\|_{\rm op}\le C\pi_{k,n}^{-1},
\qquad
 e_0^\top(\widehat G_x)^\dagger\widehat G_x=e_0^\top,
\]
as well as the required weight-sum bounds.

Since the truncation map \(t\mapsto (-L_{\mathcal F})\vee t\wedge
L_{\mathcal F}\) is \(1\)-Lipschitz and
\(f(x^\star)\in[-L_{\mathcal F},L_{\mathcal F}]\), it suffices on
\(\mathcal E_{k,n}(x)\) to control the untruncated intercept error.

For observations with \(\widehat W_i>0\), Lemma~\ref{lem:tlp-local-purity}
gives \(Z_i=k\), and we may write
\[
    X_i^\star=x^\star+U_\star u_i+G_{x^\star}(u_i).
\]
On the regression support,
$\|u_i\|\le Ch_x$ and $\|G_{x^\star}(u_i)\|\le Ch_x^2$. The pullback
\[
    m_{x^\star}(u)
    :=
    f(x^\star+U_\star u+G_{x^\star}(u))
\]
is uniformly \(\alpha_k\)-H\"older on a ball of radius \(Ch_x\). By
Lemma~\ref{lem:poly-feature-facts}\textnormal{(iii)}, there is
\(w_x^0\in\mathbb R^{q_{D,p_x}}\) such that
\begin{equation*}
    e_0^\top w_x^0=f(x^\star),
    \qquad
    \left|
    m_{x^\star}(u_i)-\Phi_{D,p_x}(U_\star u_i/h_x)^\top w_x^0
    \right|
    \le
    Ch_x^{\alpha_k}.
\end{equation*}

Moreover,
\begin{equation*}
\widehat v_i-{U_\star u_i\over h_x}
=
{(\widehat\Pi_x-\Pi^\star)(X_i-x)\over h_x}
+{\Pi^\star\xi_i\over h_x}
-{\Pi^\star\xi_{\rm qry}\over h_x}.
\end{equation*}
For arbitrary bounded covariate perturbations, the localized window gives
\begin{equation*}
 \left\|\widehat v_i-\frac{U_\star u_i}{h_x}\right\|
 \le C\left(\delta_x+\frac{\sigma_{k,n}}{h_x}\right).
\end{equation*}
The comparison polynomial in Lemma~\ref{lem:poly-feature-facts}(iii) is
\(Ch_x\)-Lipschitz in the normalized coordinates; when \(p_x=0\) it is
constant. Thus
\begin{equation*}
 f(X_i^\star)=\widehat\phi_i^\top w_x^0+r_i,\qquad
 |r_i|\le C(h_x^{\alpha_k}+h_x\delta_x+\sigma_{k,n}).
\end{equation*}
The small-perturbation condition implies
\(\sigma_{k,n}\lesssim h_x^{\alpha_k\vee1}\le h_x^{\alpha_k}\).

Write
\begin{equation*}
\widehat b_x
=
\widehat G_x w_x^0+R_x+E_x,
\end{equation*}
where
\[
R_x:={1\over n}\sum_{i\in[n]}\widehat W_i\widehat\phi_i r_i,
\qquad
E_x:={1\over n}\sum_{i\in[n]}\widehat W_i\widehat\phi_i\epsilon_i .
\]
The range identity yields
\begin{equation*}
 e_0^\top\widehat w_{\rm plug}^\Pi(x)-f(x^\star)
=
 e_0^\top(\widehat G_x)^\dagger R_x
+
 e_0^\top(\widehat G_x)^\dagger E_x.
\end{equation*}
Using \(\|(\widehat G_x)^\dagger\|_{\rm op}\le C\pi_{k,n}^{-1}\), boundedness of
the features, and \({n}^{-1}\sum_{i\in[n]}\widehat W_i\le C\pi_{k,n}\),
\[
    |e_0^\top(\widehat G_x)^\dagger R_x|
    \le
    C(h_x^{\alpha_k}+h_x\delta_x+\sigma_{k,n}).
\]

For the noise term, condition on the sigma-field generated by
\((Z_i,X_i^\star,\xi_i)_{i\in[n]}\) and the query variables. The event
\(\mathcal E_{k,n}(x)\), the weights, and the features are measurable with
respect to this sigma-field. The response
noises are conditionally mean-zero and independent.

If \(\widehat W_i>0\), then \(\mathsf A_i>0\), and hence
\(\|X_i-x\|_1\le h_x\). Therefore \(\|z_i\|\le1\), and since
\(\widehat\Pi_x\) is an orthogonal projector, $\|\widehat v_i\|\le1$.
Because \(p_x\le p_{\max}\), the polynomial dictionary is bounded on the unit
ball, and therefore
$\|\widehat\phi_i\|\le C$.
Therefore,
\begin{equation*}
\mathbb E_{k,x^\star,\xi_{\rm qry}}
\left[
|e_0^\top(\widehat G_x)^\dagger E_x|^2
\mathbb{I}_{\mathcal E_{k,n}(x)}
\right]
\le
{C\over n\pi_{k,n}h_x^{d_x}}.
\end{equation*}
Indeed, on \(\mathcal E_{k,n}(x)\),
\(\|(\widehat G_x)^\dagger\|_{\rm op}\le C\pi_{k,n}^{-1}\), the features are
uniformly bounded, and
\[
{1\over n}\sum_{i\in[n]}\widehat W_i^2\le C\pi_{k,n}h_x^{-d_x}.
\]

By Lemma~\ref{lem:tangent-projector-good},
\begin{equation*}
\mathbb E_{k,x^\star,\xi_{\rm qry}}\delta_x^2
\le
C\left\{
 h_x^{2s_x}+{\sigma_{k,n}^2\over h_x^2}
 +
 {\log N_{{\rm grid},k,n}+a_{{\rm tan},n}\over n\pi_{k,n}h_x^{d_x}}
 +e^{-c_{\rm Ber}a_{{\rm tan},n}}
\right\}.
\end{equation*}
Using
\[
    \sigma_{k,n}\lesssim h_x^{\alpha_k\vee1},
    \qquad
    (n\pi_{k,n}h_x^{d_x})^{-1}=h_x^{2\alpha_k},
    \qquad
    \log N_{{\rm grid},k,n}+a_{{\rm tan},n}\lesssim \log(en),
\]
and \(h_x\le r_n\le n^{-\rho_h}\), we get
\begin{equation*}
    h_x^2
    \mathbb E_{k,x^\star,\xi_{\rm qry}}\delta_x^2
    \le
    Ch_x^{2\alpha_k}.
\end{equation*}
The bound \(h_x^2\log(en)\lesssim1\) absorbs the stochastic tangent factor,
and the exponential term is \(o(h_x^{2\alpha_k})\) by the choice of
\(a_{{\rm tan},n}\).
Also \(\sigma_{k,n}^2\lesssim h_x^{2\alpha_k}\).

Combining the deterministic and stochastic terms gives
\begin{equation*}
\mathbb E_{k,x^\star,\xi_{\rm qry}}
\left[
\{\widehat f_{\rm plug}^\Pi(x)-f(x^\star)\}^2
\mathbb{I}_{\mathcal E_{k,n}(x)}
\right]
\lesssim
h_x^{2\alpha_k}+{1\over n\pi_{k,n}h_x^{d_x}}.
\end{equation*}
On \(\mathcal E_{k,n}(x)^c\), both
\(\widehat f_{\rm plug}^\Pi(x)\) and \(f(x^\star)\) are bounded by
\(L_{\mathcal F}\), and \eqref{eq:main-good-event-failure} gives an additional
\(o(h_x^{2\alpha_k})\) contribution. Hence
\begin{equation*}
\mathbb E_{k,x^\star,\xi_{\rm qry}}
\{\widehat f_{\rm plug}^\Pi(x)-f(x^\star)\}^2
\lesssim
h_x^{2\alpha_k}+{1\over n\pi_{k,n}h_x^{d_x}}
\lesssim
(n\pi_{k,n})^{-2\alpha_k/(2\alpha_k+d_k)} .
\end{equation*}
Averaging over \(Z_{n+1}\) gives
\begin{equation*}
\sup_{f\in\mathcal H(\boldsymbol\alpha,\mathscr M;L_{\mathcal F})}
\sup_{P\in\mathcal P_\star}
\mathbb E_{P,f}
\left[
\{\widehat f_{\rm plug}^\Pi(X_{n+1})-f(X^\star_{n+1})\}^2
\right]
\lesssim
\mathfrak R_n(\boldsymbol\alpha,\mathbf d,\boldsymbol\pi).
\end{equation*}
The assertion is asymptotic; all constants above are uniform once the
feasible-scale assumptions hold.
\end{proof}

\section{Transformer approximation}
\label{app:transformer_approximation}

We prove Theorem~\ref{thm:transformer-approx} by an explicit compilation of
row-wise ReLU networks and attention-based copying and averaging. The
embedding dimension is allowed to grow polynomially with \(n\). In
particular, the intermediate states of a wide ReLU comparison tree are
stored in polynomially many transformer coordinates.

Throughout this appendix the manifolds, component parameters, selector nets,
and tangent grids are deterministic structural information encoded in the
constructed parameters, independently of the task function \(f\) and the
particular prompt. The only prompt inputs are
\(\mathfrak s=((X_i,Y_i)_{i=1}^n;X_{n+1})\). We abbreviate
\(\mathcal H=\mathcal H(\boldsymbol\alpha,\mathscr M;L_{\mathcal F})\).

\subsection{Architecture-compatible approximation primitives}
\label{app:geo-primitive}

The following elementary construction supplies all arithmetic approximants,
including globally defined versions of the spectral and inverse decoders.
It also gives a Lipschitz bound for the approximating networks themselves.

\begin{lemma}[ReLU approximation on a fixed-dimensional compact set]
\label{lem:scaled-arithmetic}
Fix input and output dimensions \(m,r\). Suppose
\(\mathcal K_n\subset[-n^B,n^B]^m\) is a nonempty compact set and
\(F_n:\mathcal K_n\to[-n^B,n^B]^r\) is \(L_n\)-Lipschitz in maximum norm, with
\(1\le L_n\le n^B\). For any fixed \(E>0\), there is a globally defined
ReLU network \(\widetilde F_n:\mathbb R^m\to[-n^B,n^B]^r\) such that
\[
 \sup_{z\in \mathcal K_n}\|\widetilde F_n(z)-F_n(z)\|_\infty\le n^{-E}.
\]
Its depth is at most \(C(B+E+1)\log(en)\), its width and parameter
magnitudes are at most \(n^{C(B+E+1)}\),
and its global Lipschitz constant in maximum norm is at most \(L_n\).
Here \(C\) depends only on \(m,r\).
\end{lemma}
\begin{proof}
Choose an \(\eta\)-net \(\{z_1,\ldots,z_M\}\subset \mathcal K_n\) in maximum norm,
where \(\eta=n^{-E}/(4L_n)\). A maximal separated set can be chosen with
\[
 M\lesssim(1+n^B/\eta)^m\lesssim n^{m(E+2B+1)}.
\]
For each output coordinate \(a\in[r]\), set
\[
 \widetilde F_{n,a}(z)
 =-n^B\lor
 \max_{j\in[M]}\{F_{n,a}(z_j)-L_n\|z-z_j\|_\infty\}\land n^B.
\]
For \(z\in \mathcal K_n\), Lipschitz continuity bounds each term by \(F_{n,a}(z)\).
A net point within \(\eta\) gives a term at least
\(F_{n,a}(z)-2L_n\eta\). Clipping preserves this error bound.
Absolute values, maxima, minima, and clipping are exact ReLU operations:
\[
 |t|=\operatorname{ReLU}(t)+\operatorname{ReLU}(-t),\qquad
 a\vee b=a+\operatorname{ReLU}(b-a).
\]
Compute the \(M\) cones in parallel and their maximum with a balanced
comparison tree. All intermediate widths are \(O(mrM)\), the depth is
\(O(1+\log M)\), and all scalar coefficients and biases are polynomially
bounded. Each cone is globally \(L_n\)-Lipschitz; taking maxima and clipping
preserves that bound coordinatewise.
\end{proof}

\noindent
The lemma applies to multiplication, fixed-degree polynomials, reciprocals
on polynomially separated positive intervals, and any fixed-dimensional
map with polynomial Lipschitz and range bounds. It approximates a map on
\(\mathcal K_n\) and defines a bounded network on all inputs; inverse and spectral
targets are evaluated only on their specified domains.

\begin{lemma}[Exact copying and averaging with uniform softmax]
\label{lem:softmax-communication}
Let \(N=n+1\) tokens contain protected markers
\(q_i=\mathbb{I}\{i=N\}\) and \(s_i=\mathbb{I}\{i\le n\}\).
Suppose an \(r\)-dimensional vector register \(v_i\) satisfies
\(\max_i\|v_i\|_\infty\le R\) on the input domain, with \(R\ge1\).
The architecture in Section~\ref{subsec:transformers} can implement either
(i) a broadcast of \(v_N\) to all rows or (ii) the average
\(n^{-1}\sum_{j\in[n]}v_j\) stored only in the query row, exactly on that
domain. Each operation uses one active attention head in one layer and at
most three transformer blocks, at most \(2r\) additional workspace
coordinates, FFN width at most \(Cr\), and parameter magnitudes at most
\(C(1+N+R)\), for a numerical constant \(C\). Other logical registers and
markers are preserved.
\end{lemma}
\begin{proof}
For \(|t|\le R\) and \(b\in\{0,1\}\), the binary value gate is
\[
 \mathsf{Gate}_R(t,b)
 :=\operatorname{ReLU}(t)-\operatorname{ReLU}(t-Rb)
   -\operatorname{ReLU}(-t)+\operatorname{ReLU}(-t-Rb)
 =bt.
\]
Allocate source and destination banks of width \(r\), disjoint from all
protected coordinates. In the update \(Z\mapsto Z+\mathrm{FFN}(Z)\),
a destination coordinate \(w\) can be cleared by adding
\(-\operatorname{ReLU}(w)+\operatorname{ReLU}(-w)=-w\).
Adding the displayed gate at the same time overwrites it by the gated
value. Thus one FFN residual update prepares the gated source and clears
the destination, using \(O(r)\) hidden units. Set its attention value matrices to zero.

For the communication sublayer set \(Q=K=0\). The attention matrix is
exactly \(N^{-1}\mathbf1_N\mathbf1_N^\top\), independently of all token
values. To broadcast the query, write \(q_jv_j\) into the source bank and
choose \(V\) to route \(N\) times that bank into the cleared destination.
Every destination row then equals
\[
 \frac1N\sum_{j\in[N]}Nq_jv_j=v_N.
\]
For averaging, prepare \(s_jv_j\) and route \(N/n\) times the source bank.
Every destination row then equals
\[
 \frac1N\sum_{j\in[N]}\frac Nn s_jv_j
 =\frac1n\sum_{j\in[n]}v_j.
\]
This vector has maximum norm at most \(R\). In the following FFN
residual update, overwrite the destination by
\(\mathsf{Gate}_R(w,q_i)\), using the same residual cancellation. The
average remains only in the query row. The temporary source bank can be
cleared in that update or in one further block with zero value matrices.

All vector coordinates are processed in parallel. Update columns of
protected registers are zero, and additional heads have zero value
matrices. The two scratch banks and the displayed gates give the stated
width bounds; their coefficients are bounded by \(C(1+N+R)\).
\end{proof}

\begin{lemma}[Compilation into residual transformer blocks]
\label{lem:relu-transformer-compilation}
Let a row-wise ReLU network have depth \(T\), maximum hidden or output width
\(W\), and scalar parameters bounded by \(B\ge1\). Its input is stored in
selected coordinates of each token. Any fixed set of other coordinates can
be preserved. The network can be implemented exactly by \(T+O(1)\)
transformer blocks of the form in Section~\ref{subsec:transformers}, with
all attention value matrices zero, embedding dimension \(O(W+r_0)\),
feed-forward dimension \(O(W+r_0)\), and parameter bound \(C(B+1)\).
Here \(r_0\) is the number of input and protected coordinates.

For \(N=n+1\) tokens with protected sample/query markers, a fixed number
\(J\) of broadcasts or sample averages of vectors of dimension at most
\(r\), bounded by \(R\ge1\) when communicated, can be inserted with at most
\(3J\) additional blocks by Lemma~\ref{lem:softmax-communication}.
The resulting embedding and FFN dimensions are \(O(W+r_0+Jr)\), and
parameters are bounded by \(C(B+N+R+1)\). One active head suffices for each
communication, independently of \(n,W,r\).
\end{lemma}
\begin{proof}
Allocate two disjoint scratch banks of width \(W\), and alternate their
roles as source and destination. Suppose the source bank is \(v\) and the
current destination bank is \(w\). A hidden layer \(v\mapsto
\operatorname{ReLU}(Av+b)\) is implemented by setting the FFN output
in the destination coordinates to
\[
 \operatorname{ReLU}(Av+b)
 -\operatorname{ReLU}(w)+\operatorname{ReLU}(-w).
\]
The block residual adds \(w\), so the new destination equals
\(\operatorname{ReLU}(Av+b)\) exactly. An affine output is handled by
\(t=\operatorname{ReLU}(t)-\operatorname{ReLU}(-t)\) and the same
cancellation of the old destination. Set all update columns of protected
coordinates to zero. These are one-hidden-layer FFN updates, with at most
a fixed multiple of \(W+r_0\) hidden units. Setting all value matrices to
zero gives \(\mathrm{MHA}\equiv0\), so \(Z+\mathrm{MHA}(Z)=Z\).

For communication, apply Lemma~\ref{lem:softmax-communication} with
separate banks, retaining the network's other state as protected registers.
Its gates are exact on the stated bounded domain, and its destination
values equal the required broadcast or average. Induction over the
row-wise layers and communication stages therefore gives the prescribed
computation exactly. Allocating the indicated banks and padding unused
heads with zero value matrices yields the resource bounds.
\end{proof}

\subsection{Embedding and readout maps}
\label{app:emb-and-read}

For \(d\in[d_{\max}]\), let
\[
 \mathfrak C_d=\{J\subset[D]:|J|=d\},\qquad
 \mathfrak C_{\rm all}=\{(d,J):d\in[d_{\max}],J\in\mathfrak C_d\}.
\]
Fix an enumeration \(\{(d_j,J_j):j\in[N_{\rm ch}]\}\) of this list.
The numbers
\[
 q_{\rm tan,\max}=\binom{D+2s_{\max}}{2s_{\max}},\qquad
 q_{\rm reg,\max}=\binom{d_{\max}+p_{\max}}{p_{\max}},\qquad N_{\rm ch}
\]
are fixed with respect to \(n\). Put \(N_n^{\rm end}=n+1\).

\paragraph{Embedding.}
The affine map \(\mathrm{Emb}_n\) places \((X_i,Y_i)\) in sample row \(i\),
places \((X_{n+1},0)\) in query row \(n+1\), and writes a constant coordinate
\(c_i=1\) and the sample/query markers in
Lemma~\ref{lem:softmax-communication}. All other coordinates are
zero. The embedding dimension \(d_{E,n}\) includes fixed logical registers
for the raw data, scale state, tangent summaries, and chartwise regression
summaries, together with polynomially many scratch coordinates. Its size
is fixed deterministically by the architecture bounds below.

\paragraph{Readout.}
Reserve two query coordinates \((w_j,v_j)\) per chart branch. Clip them to
\(\bar w_j\in[0,1]\), \(\bar v_j\in[-L_{\mathcal F},L_{\mathcal F}]\).
Using Lemma~\ref{lem:scaled-arithmetic} for the multiplication map on this
fixed rectangle, choose a globally bounded ReLU map \(M_n\) with
\[
 \sup_{(w,v)\in[0,1]\times[-L_{\mathcal F},L_{\mathcal F}]}
 |M_n(w,v)-wv|\le N_{\rm ch}^{-1}n^{-E_{\rm read}}.
\]
Define the fixed readout
\[
 \mathrm{Read}_n(H)
 =(-L_{\mathcal F})\lor
   \sum_{j\in[N_{\rm ch}]}M_n(\bar w_j,\bar v_j)\land L_{\mathcal F}.
\]
It is uniformly within \(n^{-E_{\rm read}}\) of the clipped exact weighted
sum. Moreover, the multiplication map on the rectangle has a fixed
Lipschitz constant, so the construction in
Lemma~\ref{lem:scaled-arithmetic} gives
\begin{equation}\label{eq:readout-lipschitz}
 \operatorname{Lip}_{\|\cdot\|_{\max}}(\mathrm{Read}_n)\le C
\end{equation}
on the entire transformer-output space, uniformly in \(n\). The readout
has logarithmic depth, polynomial size, and fixed parameters.

\subsection{Structural selection and geometric summaries}
\label{app:geometric-preconditional}

For a closed set \(A\), write
\(A^{\oplus r}=\{x:d(x,A)\le r\}\). The offsets
\(\mathcal M_k^{\oplus\sigma_{k,n}}\) are disjoint because
\(2\sigma_{\max,n}<\delta_{0,n}\). On the query support define
\(k_x=\kappa_n(x)=k\) when \(x\in\mathcal M_k^{\oplus\sigma_{k,n}}\), and set
\[
 b_n^\circ(x)=(e_{d_x},e_{p_x},e_{s_x},h_x,r_\pi(x)),\qquad
 r_\pi(x)=\pi_{k_x,n}^{-1/2}.
\]
Here the \(e\)'s are one-hot vectors of fixed lengths. Their values are
structural states computed internally.

\begin{lemma}[Polynomial structural bounds and exact selection]
\label{lem:encoded-parameter-polynomial-bounds}
Under Assumptions~\ref{ass:geometry}--\ref{ass:feasible-local}, the state
\(b_n^\circ(x)\) can be computed exactly on the query support by a ReLU
network with \(O(\log n)\) depth, polynomial width, and polynomially bounded
parameters. The following quantities are polynomially bounded:
\[
 K_n,\quad \max_k h_k^{-1},\quad \max_k\pi_{k,n}^{-1/2},\quad
 \delta_{0,n}^{-1},\quad \max_{d,s}|\mathcal Q_{d,s,n}|,\quad
 \Delta_{{\rm tan},n}^{-1}.
\]
\end{lemma}
\begin{proof}
Lemma~\ref{lem:ass3-consequences} gives
\[
 K_n\le n^{1-\kappa},\quad h_k^{-1}\le n^{\bar\rho_h},\quad
 \pi_{k,n}^{-1/2}\le n^{(1-\kappa)/2},\quad
 \delta_{0,n}^{-1}\le Cn^{\bar\rho_h}.
\]
The last inequality follows from
\(\delta_{0,n}\ge a_{\rm sc}^{-1}r_n\ge a_{\rm sc}^{-1}n^{-\bar\rho_h}\).
Lemma~\ref{lem:tangent-product-net} and the selected grids in
\eqref{eq:N-grid} give
\[
 \max_{d\in[d_{\max}],\,s\in[s_{\max}]}|\mathcal Q_{d,s,n}|\le Cn^{m_\Theta},
 \qquad
 m_\Theta:=D+\max_{d\in[d_{\max}]}d(D-d)
       +D\sum_{\ell=2}^{s_{\max}}\binom{D+\ell-1}{\ell}.
\]
The constant depends only on \(D,d_{\max},s_{\max},L_{\mathcal M}\),
and \(\Delta_{{\rm tan},n}^{-1}=n^3\).

Lemma~\ref{lem:ass3-consequences} also gives
\(\delta_{0,n}-2\sigma_{\max,n}\ge\delta_{0,n}/2\). Thus, for
\(r_{\rm sel}:=(\delta_{0,n}-2\sigma_{\max,n})/(16\sqrt D)\),
\[
 0<r_{\rm sel}^{-1}\le32\sqrt D\,\delta_{0,n}^{-1}
 \le Cn^{\bar\rho_h}.
\]
In each offset, fix a maximal set \(\mathcal C_{k,n}\) whose pairwise
maximum-norm distances exceed \(r_{\rm sel}/4\). It is an
\(r_{\rm sel}/4\)-net. Each offset lies in the fixed cube
\([-B_X-1,B_X+1]^D\), so the disjoint balls of radius
\(r_{\rm sel}/8\) about these centers give
\[
 \sum_{k\in[K_n]}|\mathcal C_{k,n}|
 \le CK_n(1+r_{\rm sel}^{-1})^D\le n^C.
\]
For each center put
\[
 H_c(x)=\left(1-\frac{2\|x-c\|_\infty}{r_{\rm sel}}\right)_+,\qquad
 \chi_k(x)=0\vee\left(2\max_{c\in\mathcal C_{k,n}}H_c(x)\right)\wedge1.
\]
If \(x\in\mathcal M_k^{\oplus\sigma_{k,n}}\), a center within
\(r_{\rm sel}/4\) gives \(\chi_k(x)=1\). Every center in a different offset
has maximum-norm distance at least
\((\delta_{0,n}-2\sigma_{\max,n})/\sqrt D=16r_{\rm sel}\), giving
\(\chi_\ell(x)=0\) for \(\ell\ne k\). Consequently,
\[
 \sum_{k\in[K_n]}\chi_k(x)(e_{d_k},e_{p_k},e_{s_k},h_k,\pi_{k,n}^{-1/2})=b_n^\circ(x)
\]
on the query support. The hats and finite maxima are exact ReLU maps;
balanced comparison trees have logarithmic depth and polynomial width.
Lemma~\ref{lem:relu-transformer-compilation} stores their intermediate
values in the polynomial-dimensional scratch banks.
\end{proof}

\paragraph{Tangent moment and loss-weighted projector.}
For \(x=X_{n+1}\), define
\[
 m^{\rm ora}_{\rm tan}(x)=\frac1n\sum_{i\in[n]}
 \mathsf A\!\left(\frac{\|X_i-x\|_1}{h_x}\right)
 \Phi_{D,2s_{\max}}\!\left(\frac{X_i-x}{h_x}\right).
\]
This is a fixed-dimensional vector, bounded by a fixed constant because the
cutoff vanishes when \(\|(X_i-x)/h_x\|_1\ge1\). For every grid candidate
\(q\), the polynomial \(\|R_q(z)\|^2\) has degree at most \(2s_{\max}\).
Write \(\|R_q(z)\|^2=c_q^\top\Phi_{D,2s_{\max}}(z)\), with coefficients
uniformly bounded over the parameter grid. Set
\begin{gather*}
 L_q(m,h)=h^2c_q^\top m,\qquad
 t_q(m,h)=\{\Delta_{{\rm tan},n}-(L_q(m,h)-\min_{q'}L_{q'}(m,h))\}_+,\\
 \bar P(m;d,s,h)=\frac{\sum_{q\in\mathcal Q_{d,s,n}}t_q(m,h)\Pi_q}
                     {\sum_{q\in\mathcal Q_{d,s,n}}t_q(m,h)}.
\end{gather*}
The denominator is at least \(\Delta_{{\rm tan},n}\) for every \((m,h)\).
On bounded moment domains and \(h\in[n^{-\bar\rho_h},1]\),
\begin{equation}\label{eq:tangent-summary-lipschitz}
 \operatorname{Lip}(\bar P)\le
 C|\mathcal Q_{d,s,n}|/\Delta_{{\rm tan},n}\le n^C.
\end{equation}
Indeed, finite minima are Lipschitz with the maximum candidate Lipschitz
constant, each \(t_q\) is uniformly Lipschitz, and
\(\|\sum_qt_q\Pi_q\|_{\rm op}\le\sum_qt_q\).
Define
\[
 \bar P_x^{\rm ora}=\bar P(m^{\rm ora}_{\rm tan}(x);d_x,s_x,h_x),\qquad
 P_{T,x}^{\rm ora}=\Pi_{(d_x)}(\bar P_x^{\rm ora}).
\]
These are the averaged and spectral projectors in Algorithm~\ref{alg:localpoly}.

\paragraph{Spectral decoding domain.}
For each \(d\) use the compact set of symmetric matrices
\[
 \mathcal K_d=\{B:\|B\|_{\rm op}\le2,\ \lambda_d(B)\ge2/3,
                       \ \lambda_{d+1}(B)\le1/3\}.
\]
The map \(B\mapsto\Pi_{(d)}(B)\) is uniformly Lipschitz on this set. For
example, write the projector as a resolvent contour integral along a fixed
rectangle with left edge at \(1/2\), right edge at \(3\), and imaginary
parts \(\pm1\). Every point of the contour is a fixed positive distance
from the spectrum. The resolvent identity bounds the difference of the
integrals by \(C\|B-B'\|_{\rm op}\).
Lemma~\ref{lem:scaled-arithmetic} therefore supplies a globally defined,
bounded and Lipschitz ReLU spectral decoder, accurate on \(\mathcal K_d\).
Its input and output are symmetrized by \(A\mapsto(A+A^\top)/2\).
This map is nonexpansive in maximum norm and fixes the target projector,
so the decoder remains bounded and Lipschitz with the same accuracy on
\(\mathcal K_d\). Its symmetric output is passed to the safe chart maps below.

\paragraph{Globally defined safe chart maps.}
Let \(E_J\) be the coordinate embedding associated with \(J\in\mathfrak C_d\),
put \(A_J(P)=E_J^\top PE_J\), and let
\(\tau_d=(4|\mathfrak C_d|)^{-1}\). Define
\[
 \vartheta_d(t)=0\vee(4|\mathfrak C_d|t-1)\wedge1.
\]
Choose a smooth bounded function \(\rho_d:\mathbb R\to\mathbb R\) equal
to \(t^{-1/2}\) for \(t\ge\tau_d\), smoothly continued below \(\tau_d\)
with bounded derivatives. For every symmetric \(P\) in a fixed bounded
matrix cube define, by spectral functional calculus,
\begin{equation}\label{eq:safe-chart-extension}
 U_J^{\rm safe}(P)=PE_J\rho_d(A_J(P)),\qquad
 \omega_J^{\rm safe}(P)=
 \frac{\vartheta_d(\det A_J(P))}
      {\max\{1/2,\sum_{J'\in\mathfrak C_d}\vartheta_d(\det A_{J'}(P))\}}.
\end{equation}
These maps are bounded and Lipschitz on the cube. The determinant is
polynomial, the denominator is bounded below, and for symmetric matrices
\(A,B\) the Frobenius-norm identity in their two orthonormal eigenbases gives
\[
 \|\rho_d(A)-\rho_d(B)\|_F
 \le \sup|\rho_d'|\,\|A-B\|_F.
\]
Thus \eqref{eq:safe-chart-extension} also covers perturbed, non-PSD matrices
and eigenvalues outside \([0,1]\). All decoder inputs are clipped to the
chosen cube before these maps are approximated.

If \(P\) is an exact rank-\(d\) orthogonal projector, Cauchy--Binet gives
\(\sum_J\det A_J(P)=1\). Hence the denominator in
\eqref{eq:safe-chart-extension} is at least one, the weights sum to one,
and \(\omega_J^{\rm safe}(P)>0\) implies
\(\lambda_{\min}(A_J(P))\ge\det A_J(P)>\tau_d\). For these charts only,
\begin{equation}\label{eq:positive-chart-basis}
 (U_J^{\rm safe}(P))^\top U_J^{\rm safe}(P)=I_d,\qquad
 U_J^{\rm safe}(P)(U_J^{\rm safe}(P))^\top=P.
\end{equation}

\paragraph{Ideal chart states and positive-weight charts.}
For \(d_j=d_x\), set
\[
 \omega_{j,x}^{\rm ora}=\omega_{J_j}^{\rm safe}(P_{T,x}^{\rm ora}),\qquad
 U_{j,x}^{\rm ora}=U_{J_j}^{\rm safe}(P_{T,x}^{\rm ora}),
\]
with zero columns appended to \(U_{j,x}^{\rm ora}\) when needed. Set both
quantities to zero for \(d_j\ne d_x\). Define
\begin{equation*}
 \mathcal J_x^+=\{j:d_j=d_x,\ \omega_{j,x}^{\rm ora}>0\}.
\end{equation*}
For every dimension-compatible branch, including those with zero weight,
put
\begin{gather*}
 u_{i,j}^{\rm ora}=(U_{j,x}^{\rm ora})^\top(X_i-x)/h_x,\\
 W_{i,j}^{\rm ora}=h_x^{-d_x}
 \mathsf A(\|X_i-x\|_1/h_x)\mathsf K(U_{j,x}^{\rm ora}u_{i,j}^{\rm ora}).
\end{gather*}
For dimension-incompatible branches set \(u_{i,j}^{\rm ora}=0\) and
\(W_{i,j}^{\rm ora}=0\). Only indices in \(\mathcal J_x^+\) are required
to provide orthonormal coordinates and an invertible regression Gram matrix.

\subsection{Implementation of the geometric block}
\label{app:geometric-implementation}

Choose polynomial envelopes
\[
 R_{h,n}=n^{\bar\rho_h},\quad R_{\pi,n}=n^{(1-\kappa)/2},\quad
 R_{W,n}=R_{h,n}^{d_{\max}},\quad R_{u,n}=C_uR_{h,n},
\]
where \(C_u\) dominates the fixed safe-chart and observation bounds.
All ideal coordinates and weights lie in these envelopes. When invoking an
arithmetic approximant, inputs are clipped to its declared box; this is an
exact nonexpansive operation that preserves the ideal inputs.

\begin{lemma}
\label{lem:geo-module-implementation}
For every \(E>0\), one transformer block stack computes \(b_n^\circ(x)\)
exactly on the query support and produces \(\mathsf\omega_{j,x}\),
\(\mathsf u_{i,j}(x)\), and \(\mathsf W_{i,j}(x)\) such that, on the event
\(\bar P_x^{\rm ora}\in\mathcal K_{d_x}\),
\begin{equation*}
 \max_j|\mathsf\omega_{j,x}-\omega_{j,x}^{\rm ora}|
 +\max_{i,j}\|\mathsf u_{i,j}(x)-u_{i,j}^{\rm ora}\|
 +\max_{i,j}|\mathsf W_{i,j}(x)-W_{i,j}^{\rm ora}|
 \le Cn^{-E}.
\end{equation*}
For every observed prompt, all outputs are defined and satisfy
\[
 0\le\mathsf\omega_{j,x}\le1,\qquad
 \|\mathsf u_{i,j}(x)\|_\infty\le R_{u,n},\qquad
 0\le\mathsf W_{i,j}(x)\le R_{W,n}.
\]
Depth is \(O(\log n)\); embedding dimension, feed-forward dimension, and
parameter magnitudes are polynomially bounded. The head count is fixed.
\end{lemma}
\begin{proof}
First compute the exact selector in
Lemma~\ref{lem:encoded-parameter-polynomial-bounds} on the query row, and
broadcast \((x,b_n^\circ(x))\) to all sample rows. The row map producing
\(\mathsf A(\|X_i-x\|_1/h_x)\Phi_{D,2s_{\max}}((X_i-x)/h_x)\)
has fixed input/output dimensions and polynomial range and Lipschitz
bounds on the declared observation and bandwidth domains. Approximate it
with Lemma~\ref{lem:scaled-arithmetic}, then average by
Lemma~\ref{lem:softmax-communication}. The resulting moment error is bounded
by the row-wise approximation error.

The moment-to-\(\bar P\) map has a polynomial Lipschitz bound by
\eqref{eq:tangent-summary-lipschitz}; approximate this fixed-dimensional
map by the same lemma, separately for each of the finitely many \((d,s)\).
One-hot states select the corresponding outputs. On the stated event,
the spectral decoder is accurate at the ideal matrix. Its global Lipschitz
bound also controls the change caused by the approximate moment and
averaged projector, including perturbed inputs outside \(\mathcal K_d\). The safe chart maps are globally specified on a fixed
cube, so the same argument produces the chart states. Clip chart weights
to \([0,1]\) and broadcast the chart matrices.

Finally the row maps producing \(u_{i,j}\) and \(W_{i,j}\) have fixed
input/output dimensions and polynomial Lipschitz constants, since
\(h_x^{-1}\le R_{h,n}\). Apply the arithmetic lemma and clip the resulting
coordinates and weights to their envelopes. For each dimension/degree
selector there are only finitely many labels. The exact binary gate in
Lemma~\ref{lem:softmax-communication} sets dimension-incompatible branch
states to zero on the declared envelopes.

There are a fixed number of nonlinear module stages and attention
communications. Their target maps and the chosen approximants have
polynomial Lipschitz bounds on the clipped domains. If each module error
is at most \(n^{-E'}\), the final maximum error is at most
\(Cn^{C_{\rm prop}}n^{-E'}\), with \(C_{\rm prop}\) independent of \(E'\).
Take \(E'>E+C_{\rm prop}+1\). The arithmetic lemma and the compilation
lemma then give the stated sizes, including scratch space for all
intermediate comparison states. Shared row computations on non-query
rows are harmless: only marker-selected rows are read in broadcasts,
averages, and the final readout.
\end{proof}

\subsection{Globally defined regression decoder}
\label{app:regression-solver}

A chart branch carries \(((u_i,Y_i,W_i)_{i=1}^n;e_d,e_p,r)\), where
\(r=\pi_{k,n}^{-1/2}\), the coordinates are padded to \(d_{\max}\), and
\(p\le p_{\max}\). Let \(\phi_{d,p}(u)=\Phi_{d,p}(u_{1:d})\), and form
\[
 G=\frac1n\sum_{i\in[n]} r^2W_i\phi_{d,p}(u_i)\phi_{d,p}(u_i)^\top,\qquad
 g=\frac1n\sum_{i\in[n]} r^2W_iY_i\phi_{d,p}(u_i).
\]
They have dimension at most \(q_{\rm reg,\max}\). Let
\(\mathcal B_n^{\Pi,\rm reg}\) be the box of branch prompts satisfying
\[
 \max_i\|u_i\|_\infty\le R_{u,n},\quad 0\le W_i\le R_{W,n},\quad
 |Y_i|\le B_Y,\quad 0<r\le R_{\pi,n}.
\]
Choose \(C_G\) so that \(\|G\|_{\rm op}+\|g\|_2\le n^{C_G}\) throughout
this box. Choose a fixed \(\lambda_{\rm reg}>0\) so small that the ideal
positive-weight chart matrices in the regression design good event have
\(\lambda_{\min}(G)\ge4\lambda_{\rm reg}\). Indeed, the reduced-Gram
bound following \eqref{eq:H-Pi-perturbation} in the proof of
Lemma~\ref{lem:regression-design-stability} is uniform over nearby
projectors and their orthonormal coordinates; normalization by
\(r^2=\pi_{k,n}^{-1}\) makes its lower bound a fixed positive constant.
Define
\[
 \mathcal G_n^{\Pi,\rm reg}
 =\{s^{\rm reg}\in\mathcal B_n^{\Pi,\rm reg}:\lambda_{\min}(G)\ge
 \lambda_{\rm reg}\},
\]
and define the exact local polynomial value only on this set by
\[
 f_{\rm LP}(s^{\rm reg})
 =(-L_{\mathcal F})\lor e_{0,d,p}^\top G^{-1}g\land L_{\mathcal F}.
\]

\begin{lemma}
\label{lem:reg-one-shot-decoder}
For any fixed \(E>0\), row-wise ReLU computations, one uniform-softmax
averaging stage from Lemma~\ref{lem:softmax-communication}, and a globally
defined ReLU decoder produce a branch value
\(V_n(s^{\rm reg})\in[-L_{\mathcal F},L_{\mathcal F}]\) on the entire
box \(\mathcal B_n^{\Pi,\rm reg}\), and
\[
 \sup_{s^{\rm reg}\in\mathcal G_n^{\Pi,\rm reg}}
 |V_n(s^{\rm reg})-f_{\rm LP}(s^{\rm reg})|\le Cn^{-E}.
\]
Depth is \(O(\log n)\) and all dimensions and parameters are polynomially
bounded. Inverse evaluation is restricted to well-conditioned branches.
\end{lemma}
\begin{proof}
Approximate the sample-wise entries of
\(r^2W_i\phi_{d,p}(u_i)\phi_{d,p}(u_i)^\top\) and
\(r^2W_iY_i\phi_{d,p}(u_i)\) to error \(n^{-E'}\) by
Lemma~\ref{lem:scaled-arithmetic}, and aggregate their averages by
Lemma~\ref{lem:softmax-communication}. These are fixed-degree maps in a fixed
number of coordinates on polynomial boxes.

On the compact set
\[
 \mathcal K_{d,p,n}^{\rm inv}
 =\{(G,g):G=G^\top,\ G\succeq\lambda_{\rm reg}I,
                       \ \|G\|_{\rm op}+\|g\|_2\le n^{C_G}\},
\]
the map \((G,g)\mapsto(-L_{\mathcal F})\lor e_0^\top G^{-1}g\land L_{\mathcal F}\) has
polynomial Lipschitz constant. Indeed,
\[
 G^{-1}-G'^{-1}=G^{-1}(G'-G)G'^{-1},\qquad
 \|G^{-1}\|_{\rm op},\|(G')^{-1}\|_{\rm op}\le\lambda_{\rm reg}^{-1}.
\]
Apply Lemma~\ref{lem:scaled-arithmetic} to this compact-set map and clip its
output to \([-L_{\mathcal F},L_{\mathcal F}]\). The resulting network is
defined on all inputs, including singular matrices. Its error at the
ideal summary is small, and its global polynomial Lipschitz bound controls
the approximate-summary error. Taking \(E'\) sufficiently large gives the
claim. Finitely many dimension/degree cases are selected using exact
one-hot gating. Compile these ReLU networks by
Lemma~\ref{lem:relu-transformer-compilation}.
\end{proof}

\begin{lemma}
\label{lem:one-shot-solver-stability}
Consider two branch prompts in \(\mathcal G_n^{\Pi,\rm reg}\) with the same
\((d,p,r,Y_i)\) but coordinates and weights \((u_i,W_i)\) and
\((u'_i,W'_i)\). Then
\[
 |f_{\rm LP}(s)-f_{\rm LP}(s')|
 \le Cn^C\left(\max_i\|u_i-u'_i\|+\max_i|W_i-W'_i|\right).
\]
\end{lemma}
\begin{proof}
The sample-wise polynomial summary maps are polynomially Lipschitz on the
branch boxes. Averaging preserves their maximum-error bound. Apply the
inverse identity in the preceding proof and the nonexpansiveness of
clipping.
\end{proof}

\paragraph{Compact localization through value coordinates.}
The cutoff $\mathsf A$ and local weights enter FFN-computed value coordinates
before uniform-softmax averaging. The ideal localized value is zero outside
the cutoff, although every attention probability is positive. Compact
localization thus acts through values, with approximate contributions
controlled by the module error bounds.

\begin{remark}[Common normalization of local regression weights]
\label{rem:regression-normalization}
Let $W_i\ge0$, $Z_W:=\sum_{i\in[n]}W_i>0$, and let $\phi_i$ be the local
polynomial feature vectors. Define
\[
 G=\sum_{i\in[n]}W_i\phi_i\phi_i^\top,\qquad b=\sum_{i\in[n]}W_i\phi_iY_i,
 \qquad a_i=\frac{W_i}{Z_W}.
\]
The normalized statistics satisfy
\[
 \widetilde G=\sum_{i\in[n]} a_i\phi_i\phi_i^\top=Z_W^{-1}G,\qquad
 \widetilde b=\sum_{i\in[n]} a_i\phi_iY_i=Z_W^{-1}b,
 \qquad \widetilde G^\dagger\widetilde b=G^\dagger b.
\]
Indeed, scaling a positive semidefinite matrix by $c>0$ scales its nonzero
eigenvalues by $c$, so $(cG)^\dagger=c^{-1}G^\dagger$. Thus common positive
normalization preserves the exact weighted least-squares coefficient and
its clipped intercept, including when $G$ is singular. In our construction,
common factors such as $n^{-1}r_\pi^2$ enter both summaries; their averages
are computed directly by Lemma~\ref{lem:softmax-communication}.
\end{remark}

\subsection{Model size and error propagation}
\label{app:model-size}

All selectors and approximants above are finite networks, chosen
independently of the particular prompt. Their row-wise inputs and outputs
have fixed dimensions, but their intermediate widths can grow. The exact
selector has polynomially many centers; the tangent grids and all
normalization factors have polynomial bounds by
Lemma~\ref{lem:encoded-parameter-polynomial-bounds}. The arithmetic
approximants operate on polynomial boxes, with fixed input/output
dimensions and polynomial Lipschitz constants.

Communication acts only on logical registers: raw observations, structural
states, clipped geometric/chart states, and moment or regression summaries.
The observation bounds, exact selector bounds, and declared module boxes
give a deterministic envelope \(R_n\le n^C\) for all these registers on
every admissible observed prompt, including outside the statistical good
event. This exponent is independent of the requested approximation
accuracy: the arithmetic lemma clips outputs to the fixed target range.
The intermediate states of wide comparison networks remain in local scratch
banks. Lemma~\ref{lem:softmax-communication} therefore applies at every
communication stage with this polynomial envelope, preserving exact
markers and all other logical registers.

More explicitly, for a prescribed final exponent \(A_0>0\), choose common
module accuracy \(n^{-E_*}\) with
\(E_*>A_0+2C_{\rm prop}+10\), increasing it to cover the regression
summary stability constant if necessary. There are a fixed number of
module stages, so their propagated scalar error is \(O(n^{-A_0/2})\).
Here \(C_{\rm prop}\) is computed from the target maps and the
Lipschitz-preserving approximants. It is independent of the requested
accuracy exponent. Table~\ref{tab:module-bounds} summarizes the domains.

\begin{table}[ht]
\centering\small
\renewcommand{\arraystretch}{1.10}
\caption{Domains and stability bounds for the constructed modules.}
\label{tab:module-bounds}
\begin{tabular}{>{\raggedright\arraybackslash}p{0.25\linewidth}>{\raggedright\arraybackslash}p{0.33\linewidth}>{\raggedright\arraybackslash}p{0.30\linewidth}}
\toprule
Module & Input/range control & Stability used \\
\midrule
Structural selector & Fixed observation cube; polynomially many centers & Exact ReLU construction \\
Tangent moment & \(h^{-1}\le n^{\bar\rho_h}\); fixed moment dimension & Polynomial Lipschitz bound \\
Moment-to-projector & Bounded moments; denominator \(\ge n^{-3}\) & \eqref{eq:tangent-summary-lipschitz} \\
Spectral/chart maps & Fixed compact spectral domain and chart cube & Fixed Lipschitz bounds; global decoders \\
Regression summaries & Polynomial coordinate and weight boxes & Fixed-degree polynomial maps \\
Inverse decoder & \(G\succeq\lambda_{\rm reg}I\) on comparison domain & Polynomial Lipschitz bound \\
Readout & Weights in \([0,1]\), values in \([-L_{\mathcal F},L_{\mathcal F}]\) & Uniform global Lipschitz bound \\
\bottomrule
\end{tabular}
\end{table}

Let \(W_n\) be the maximum total simultaneous width of the ReLU modules,
including all parallel chart branches and selector comparison states.
Lemma~\ref{lem:scaled-arithmetic} gives \(W_n\le n^{C_*}\) for a fixed
\(C_*\) depending on \(A_0\) and fixed model constants. Choose
\[
 d_{E,n}=C_0(W_n+r_0),\qquad d_{{\rm FFN},n}=C_1(W_n+r_0),
\]
where \(r_0\) is the fixed logical-register count and the constants provide
two ReLU scratch banks, protected coordinates, and the communication banks
in Lemmas~\ref{lem:softmax-communication} and
\ref{lem:relu-transformer-compilation}.

The communicated dimensions and the number of stages are fixed, so these additional banks have fixed total
width. Communication parameters are bounded by \(C(1+n+R_n)\).
The widths can be padded to deterministic polynomial envelopes common to
the whole model class.

Concatenating the fixed number of compiled modules and communication blocks gives
\begin{equation}\label{eq:revised-architecture-bounds}
 N_n^{\rm end}=n+1,\quad M_0=O(1),\quad
 L_n\le C_{\rm depth}\log(en),\quad
 d_{E,n}+d_{{\rm FFN},n}+B_n\le n^{C_{\rm size}}.
\end{equation}
The internal regression dictionary has size \(q_{\rm reg,\max}\), whereas
geometric registers and polynomial size exponents may depend on \(D\).
This is a constructive existence bound with polynomial workspace and
resource exponents that may depend on the ambient dimension.

\subsection{End-to-end approximation}
\label{app:approximation-error}

Let \(s_{x,j}^{\rm ora}\) and \(s_{x,j}^{\rm gen}\) denote the ideal and
generated branch prompts, respectively. For \(j\in\mathcal J_x^+\) on
the regression good event, define
\[
 f_{\rm chart}^{\rm ora}(x)
 =\sum_{j\in\mathcal J_x^+}\omega_{j,x}^{\rm ora}
                          f_{\rm LP}(s_{x,j}^{\rm ora}).
\]
This sum evaluates inverses only on positive-weight charts.

\begin{lemma}[Chart-intercept equivalence]
\label{lem:chart-equivalence}
If the ideal positive-weight branch prompts lie in
\(\mathcal G_n^{\Pi,\rm reg}\), then
\(f_{\rm chart}^{\rm ora}(x)=\widehat f_{\rm plug}^\Pi(x)\).
\end{lemma}
\begin{proof}
For \(j\in\mathcal J_x^+\), \eqref{eq:positive-chart-basis} implies
\(U_{j,x}^{\rm ora}(U_{j,x}^{\rm ora})^\top=P_{T,x}^{\rm ora}\), so the
ambient projected weights are exactly the ideal branch weights. Let
\(B=B_{U,p}\) from Lemma~\ref{lem:poly-feature-facts}(i). Write the
unnormalized reduced Gram and response vector as \(H,c\). Then the
ambient quantities are \(G_{\rm amb}=BHB^\top\) and \(b_{\rm amb}=Bc\),
where \(H\succ0\) and \(B\) has full column rank. For
\(w=G_{\rm amb}^\dagger b_{\rm amb}\), the fact
\(b_{\rm amb}\in\operatorname{Im}(G_{\rm amb})\) gives
\[
 BHB^\top w=Bc\quad\Longrightarrow\quad B^\top w=H^{-1}c.
\]
Since \(e_{0,D,p}=Be_{0,d,p}\),
\[
 e_{0,D,p}^\top w=e_{0,d,p}^\top H^{-1}c.
\]
The normalization \(r_\pi^2\) cancels between the branch Gram and response
vector. Thus every positive-weight branch has the same clipped intercept.
Their weights sum to one, proving the assertion.
\end{proof}

\begin{lemma}
\label{lem:end-to-end-good-event}
For every fixed \(A_0>0\), the geometric approximations can be chosen so
that there is an event \(\mathcal A_n\) with
\(\mathbb P(\mathcal A_n^c)\le Cn^{-A_0}\), uniformly over \(P,f\), on
which \(\bar P_x^{\rm ora}\in\mathcal K_{d_x}\), the regression design good
event holds, and
\[
 s_{x,j}^{\rm ora},\ s_{x,j}^{\rm gen}\in\mathcal G_n^{\Pi,\rm reg}
 \quad\text{for every }j\in\mathcal J_x^+.
\]
All generated branch prompts, including zero-weight charts, lie in
\(\mathcal B_n^{\Pi,\rm reg}\); the eigenvalue lower bound applies only
to branches in \(\mathcal J_x^+\).
\end{lemma}
\begin{proof}
Use the proof-level tangent concentration parameter
\[
 a_{{\rm tan},n}^{\rm app}
 =\frac{A_0+2\alpha_{\max}\bar\rho_h+3}{c_{\rm Ber}}\log(en).
\]
The estimator's grid and tolerance remain \(n^{-1}\) and \(n^{-3}\);
the displayed tail parameter is used only in the concentration proof. The window and
grid conditions in Appendix~\ref{app:tangent-local-poly-upper} remain
valid. Lemma~\ref{lem:tangent-projector-good}, with the fixed small scale
constant chosen there, gives
\[
 \lambda_{d_x}(\bar P_x^{\rm ora})\ge7/8,\qquad
 \lambda_{d_x+1}(\bar P_x^{\rm ora})\le1/8,
\]
and a projector error at most the regression-stability radius, except on
an event of probability \(Cn^{-A_0}\). The regression failure probability
is at most \(C\exp(-cN_kh_k^{d_k})\le Cn^{-A_0}\), uniformly over \(k\).

On the intersection, \eqref{eq:positive-chart-basis} makes each
\(j\in\mathcal J_x^+\) an orthonormal coordinate representation of the
stable projected design. The reduced-Gram bound following
\eqref{eq:H-Pi-perturbation} in
Lemma~\ref{lem:regression-design-stability}, together with the normalization
\(r^2=\pi_{k,n}^{-1}\), gives minimum eigenvalue at least
\(4\lambda_{\rm reg}\) uniformly over these charts.
The coordinate and weight error in
Lemma~\ref{lem:geo-module-implementation}, propagated through the
polynomial summary map, makes the generated Gram error at most
\(\lambda_{\rm reg}\) when the approximation exponent is large enough.
Hence both branch prompts belong to the required good set. This argument
is uniform as the positive weight approaches zero: positivity still
implies the fixed determinant threshold \(\det A_J>\tau_d\).
The box assertions hold for every prompt by clipping. There are finitely
many chart branches, and averaging the uniform conditional failure bounds
over the query component yields the assertion.
\end{proof}

\subsection{Proof of Theorem~\ref{thm:transformer-approx}}
\label{app:transformer-approx-main-proof}
\begin{proof}
Apply Lemmas~\ref{lem:geo-module-implementation} and
\ref{lem:reg-one-shot-decoder} with sufficiently large accuracy exponents,
and compose their transformer blocks as in
\eqref{eq:revised-architecture-bounds}. Their communication stages belong
directly to \(\mathcal T_{M_0}\) by Lemma~\ref{lem:softmax-communication};
all remaining stages use zero attention value matrices and ReLU FFNs
with block residual connections. Before the readout arithmetic
approximation, the predictor is the clipped version of
\[
 S_{\rm gen}=\sum_{j\in[N_{\rm ch}]}\mathsf\omega_{j,x}
                                     V_n(s_{x,j}^{\rm gen}).
\]
On \(\mathcal A_n\), Lemma~\ref{lem:chart-equivalence} and an exact
add-and-subtract decomposition give
\begin{align*}
 |S_{\rm gen}-\widehat f_{\rm plug}^\Pi(x)|
 &\le \sum_{j\in\mathcal J_x^+}\omega_{j,x}^{\rm ora}
       |V_n(s_{x,j}^{\rm gen})-f_{\rm LP}(s_{x,j}^{\rm gen})|\\
 &\quad+\sum_{j\in\mathcal J_x^+}\omega_{j,x}^{\rm ora}
       |f_{\rm LP}(s_{x,j}^{\rm gen})-f_{\rm LP}(s_{x,j}^{\rm ora})|\\
 &\quad+L_{\mathcal F}\sum_{j\in[N_{\rm ch}]}
                   |\mathsf\omega_{j,x}-\omega_{j,x}^{\rm ora}|.
\end{align*}
The first two terms use only positive-weight charts and are controlled by
Lemmas~\ref{lem:reg-one-shot-decoder} and
\ref{lem:one-shot-solver-stability}. The final sum controls all branches,
including generated leakage onto an oracle-zero-weight chart; its decoder
is bounded even when its Gram matrix is singular.

Choosing the module exponents and \(E_{\rm read}\) sufficiently large
makes the squared error on \(\mathcal A_n\) at most \(Cn^{-A_0}\).
On its complement both final predictions are clipped to
\([-L_{\mathcal F},L_{\mathcal F}]\), giving at most
\(4L_{\mathcal F}^2\mathbb P(\mathcal A_n^c)\le Cn^{-A_0}\).
Consequently,
\[
 \sup_{P\in\mathcal P_\star}\sup_{f\in\mathcal H}
 \mathbb E_{P,f}
 \{\mathrm{Read}_n\circ\mathrm{TF}_{\widehat\theta_n}\circ
   \mathrm{Emb}_n(\mathfrak s)-\widehat f_{\rm plug}^\Pi(X_{n+1})\}^2
 \le Cn^{-A_0}.
\]
The parameter vector depends only on the fixed structural model and \(n\),
independently of \(P\)'s allowed within-component density or perturbation law,
the prompt, and \(f\). The structural selector and uniform bounds
ensure the same constructed comparator works for all such \(P,f\).
\end{proof}

\subsection{Explicit dependence on the approximation exponent}
\label{app:accuracy-resource}
\begin{corollary}[Accuracy--resource bound]\label{cor:accuracy-resource}
The construction of Theorem~\ref{thm:transformer-approx} can be chosen so that,
for fixed positive coefficients $c_{L,0},c_{L,1},c_{S,0},c_{S,1}$ and every
fixed $A_0>0$, all sufficiently large $n$ satisfy
\[
 L_n\le(c_{L,0}+c_{L,1}A_0)\log(en),\qquad
 d_{E,n}+d_{{\rm FFN},n}+B_n\le n^{c_{S,0}+c_{S,1}A_0}.
\]
The coefficients depend only on the fixed model bounds and the finite module
pattern, independently of $A_0,n,K_n,P,f$. The large-$n$ threshold may depend on $A_0$.
The head count is fixed and $N_n^{\rm end}=n+1$.
\end{corollary}
\begin{proof}
There are finitely many nonlinear module families. Their input and output
dimensions, say $m_j,r_j$, depend only on $D,d_{\max},\alpha_{\max}$.
Choose one exponent $B_*>0$ bounding their target domains, ranges, and
Lipschitz constants by powers of $n$. This exponent is independent of the
requested accuracy: all bandwidth inverses, component-mass normalizers,
grid cardinalities, and the inverse near-minimizer tolerance are fixed by the
structural scale conditions. The compact spectral domain and positive-chart
conditioning constants are also independent of $A_0$.

The approximants in Lemma~\ref{lem:scaled-arithmetic} have the same
Lipschitz bounds as their target maps. Therefore composition over the fixed
module pattern amplifies a common absolute module error $\eta$ by at most
$Cn^{C_*}\eta$, where $C_*$ is independent of the internal accuracy.
This includes the polynomial stability of the regression summaries and
decoder. Exact one-hot gating, coordinate clipping, and averaging have
polynomial bounds independent of the approximation exponent as well.
Choose the internal exponent as $E_*=A_0/2+C_*+c_*$, with a fixed $c_*$
large enough to cover fixed factors and all output coordinates. On the good
event this makes the squared error $O(n^{-A_0})$. Choosing the proof-level
tail parameter proportional to $(A_0+1)\log(en)$ makes the failure probability
$O(n^{-A_0})$ while preserving the estimator, its grid, and its network modules.

For module $j$, the net in Lemma~\ref{lem:scaled-arithmetic} has at most
$C_j n^{m_j(E_*+2B_*+1)}$ centers. Its depth is bounded by
$C_j(B_*+E_*+1)\log(en)$, and its width and scalar parameters by
$n^{C_j(B_*+E_*+1)}$, after increasing fixed coefficients.
The exact structural selector has a polynomial number of centers with an
exponent independent of $A_0$. The number of chart branches and nonlinear
module stages is fixed. Consequently their maximum simultaneous width,
combined depth, and parameter magnitudes have exponents or depth constants
affine in $A_0$, after taking maxima and sums over a fixed list. The explicit
compilation in Lemma~\ref{lem:relu-transformer-compilation} allocates two
scratch banks of that width. Lemma~\ref{lem:softmax-communication} adds at
most three blocks per communication and fixed-width logical-register
banks. Its coefficients are bounded by \(C(1+n+R_n)\), with exponent
independent of \(A_0\). Absorbing fixed multiplicative factors for sufficiently large $n$ proves the
stated bounds.
\end{proof}
The conclusion concerns inverse-polynomial accuracy for each fixed $A_0$
and all sufficiently large $n$, with logarithmic depth sufficient.

\section{In-context generalization}
\label{app:generalization-error}

This appendix proves the in-context generalization result in
Theorem~\ref{thm:icl-erm}. We first bound the covering number of the full
bounded-parameter softmax transformer class and then apply a
bounded square-loss oracle inequality to its near empirical-risk minimizers.
The resulting entropy term is polynomial in \(n\) and logarithmic in the
number of training tasks \(\Gamma\). Covering-number arguments for attention
and transformer classes also appear in
\cite{edelman2022inductive,trauger2024sequence,havrilla2024understanding}.

\subsection{Covering number and entropy bound}
\label{app:transformer-entropy}

We use the softmax blocks of Section~\ref{subsec:transformers}, with a
residual connection around each attention and ReLU FFN map.
Every entry of the query, key, value, and FFN matrices and biases varies
freely in \([-B,B]\). For \(B,R\ge1\), put
\[
 \mathcal Z_N(R)
 :=\{H\in\mathbb R^{N\times d_E}:\|H\|_{\max}\le R\}.
\]
Let \(\mathrm{Read}\) have a Lipschitz bound \(L_{\mathrm{Read}}\ge1\),
with respect to maximum norm, on the transformer-output domain. Define
\[
\begin{gathered}
 \mathcal F_{\mathrm{Read}}
 :=\left\{H\mapsto\mathrm{Read}(\mathrm{TF}_\theta(H)):
 \mathrm{TF}_\theta\in\mathcal T_M(N,d_E,d_{\rm FFN},L,B)\right\},\\
 \|g\|_{\infty,R}:=\sup_{H\in\mathcal Z_N(R)}|g(H)|.
\end{gathered}
\]
We construct a parameter-sensitivity bound \(\mathfrak L_L(R)\) such that
\[
 \sup_{H\in\mathcal Z_N(R)}
 \|\mathrm{TF}_\theta(H)-\mathrm{TF}_{\theta'}(H)\|_{\max}
 \le\mathfrak L_L(R)\|\theta-\theta'\|_\infty
\]
for all parameter vectors in the full cube.

\begin{lemma}[Entropy of bounded softmax transformers]
\label{lem:transformer-entropy}
For every \(0<\delta\le1\),
\[
 \log\mathcal N(\mathcal F_{\mathrm{Read}},\delta,\|\cdot\|_{\infty,R})
 \le P_{\rm TF}
 \log\!\left(1+\frac{C B L_{\mathrm{Read}}\mathfrak L_L(R)}{\delta}\right),
\]
where the number of trainable scalar parameters satisfies
\[
 P_{\rm TF}=L\{3Md_E^2+2d_Ed_{\rm FFN}+d_{\rm FFN}+d_E\}
 \le C L(M+1)(d_E+d_{\rm FFN})^2.
\]
If \(M=O(1)\), \(L\le C_L\log(en)\),
\(B+R+d_E+d_{\rm FFN}+L_{\mathrm{Read}}\le n^{C_{\rm arch}}\),
and \(N\le C_Nn\), then the output radius and parameter-sensitivity bound
can be chosen so that
\[
\begin{gathered}
 \sup_{\theta\in[-B,B]^{P_{\rm TF}}}\sup_{H\in\mathcal Z_N(R)}
 \|\mathrm{TF}_\theta(H)\|_{\max}\le R_L,\\
 \log(1+R_L)\le C\log^2(en),\qquad
 \log\{1+\mathfrak L_L(R)\}\le C\log^3(en).
\end{gathered}
\]
Consequently, for constants \(C,c_{\rm cov}>0\) depending only on
\(C_L,C_N,C_{\rm arch}\), the fixed head-count bound, and fixed architecture
choices,
\[
 \log\mathcal N(\mathcal F_{\mathrm{Read}},\delta,\|\cdot\|_{\infty,R})
 \le C n^{c_{\rm cov}}\{1+\log(1/\delta)\}.
\]
All bounds hold uniformly over the entire parameter cube.
\end{lemma}

\begin{proof}
For \(z\in\mathbb R^N\), write \(p=\operatorname{softmax}(z)\).
Its derivative in direction \(v\in\mathbb R^N\) is
\[
 (D\operatorname{softmax}(z)v)_j
 =p_j\left(v_j-\sum_{a\in[N]} p_av_a\right).
\]
Since \(p_j\ge0\) and \(\sum_{j\in[N]}p_j=1\),
\[
 \|D\operatorname{softmax}(z)v\|_1
 \le\sum_{j\in[N]}p_j\left(|v_j|+\left|\sum_{a\in[N]} p_av_a\right|\right)
 \le2\|v\|_\infty.
\]
Integration along the segment between any two logits gives the global bound
\begin{equation}\label{eq:softmax-global-lipschitz}
 \|\operatorname{softmax}(z)-\operatorname{softmax}(z')\|_1
 \le2\|z-z'\|_\infty.
\end{equation}
The bound is independent of both \(N\) and the sizes of the logits.

We next control hidden states. Suppose the input to block \(\ell+1\) has
maximum norm at most \(R_\ell\), with \(R_0=R\). For one head, let
\[
 S(H;Q,K):=\frac{HQ(HK)^\top}{\sqrt{d_E}},
 \qquad P(H;Q,K):=\operatorname{softmax}_{\rm row}(S(H;Q,K)).
\]
Each row of \(P\) is a probability vector, so
\[
 \|P(H;Q,K)HV\|_{\max}
 \le\|HV\|_{\max}\le d_EBR_\ell.
\]
Thus \(H+\mathrm{MHA}(H)\) has radius at most
\(\widetilde R_\ell=(1+Md_EB)R_\ell\). For \(\|Z\|_{\max}\le S\),
the state \(Z+\mathrm{FFN}(Z)\) has radius at most
\((1+d_Ed_{\rm FFN}B^2)S+d_{\rm FFN}B^2+B\). We may therefore choose
\begin{equation}\label{eq:softmax-state-radius}
 R_{\ell+1}
 =(1+d_Ed_{\rm FFN}B^2)(1+Md_EB)R_\ell+d_{\rm FFN}B^2+B.
\end{equation}
These bounds apply to all parameter choices and all inputs in
\(\mathcal Z_N(R)\).

For the sensitivity calculation, take two block inputs \(H,H'\) with
maximum norm at most \(R_\ell\), and put
\(\Delta=\|H-H'\|_{\max}\). Let \(\eta\) bound the maximum difference
between corresponding parameters of the two blocks. Direct matrix
multiplication gives
\[
 \|HQ-H'Q'\|_{\max}\le d_E(B\Delta+R_\ell\eta),
 \qquad
 \|HV-H'V'\|_{\max}\le d_E(B\Delta+R_\ell\eta),
\]
with the same bound for the key projections. Since every projected entry
has absolute value at most \(d_EBR_\ell\), the scaled logits satisfy
\begin{equation*}
 \|S(H;Q,K)-S(H';Q',K')\|_{\max}
 \le2d_E^{5/2}\{B^2R_\ell\Delta+BR_\ell^2\eta\}.
\end{equation*}
For each output row, apply \eqref{eq:softmax-global-lipschitz} to its logits
and use the row sum of \(P\) to bound the value difference. This yields
\[
\begin{aligned}
 &\|P(H;Q,K)HV-P(H';Q',K')H'V'\|_{\max}\\
 &\quad\le d_E(B\Delta+R_\ell\eta)
 +2d_EBR_\ell\|S(H;Q,K)-S(H';Q',K')\|_{\max}.
\end{aligned}
\]
Set \(\widetilde H=H+\mathrm{MHA}_\theta(H)\) and
\(\widetilde H'=H'+\mathrm{MHA}_{\theta'}(H')\).
Summing over heads and including the residual bounds
\(\widetilde\Delta:=\|\widetilde H-\widetilde H'\|_{\max}\) by
\begin{equation}\label{eq:softmax-attention-sensitivity}
\begin{aligned}
 \widetilde\Delta
 &\le\{1+Md_EB+4Md_E^{7/2}B^3R_\ell^2\}\Delta\\
 &\quad+M\{d_ER_\ell+4d_E^{7/2}B^2R_\ell^3\}\eta.
\end{aligned}
\end{equation}
For two FFN inputs \(Z,Z'\) of radius at most \(S\), the one-Lipschitz
property of ReLU and the block residual similarly give
\begin{equation}\label{eq:softmax-ffn-sensitivity}
\begin{aligned}
 &\|Z+\mathrm{FFN}_\theta(Z)-Z'-\mathrm{FFN}_{\theta'}(Z')\|_{\max}\\
 &\quad\le(1+d_Ed_{\rm FFN}B^2)\|Z-Z'\|_{\max}\\
 &\quad+\{2d_{\rm FFN}B(d_ES+1)+1\}\eta.
\end{aligned}
\end{equation}
Indeed, the first affine map changes by at most
\(d_EB\|Z-Z'\|_{\max}+(d_ES+1)\eta\), and its outputs have radius at
most \(B(d_ES+1)\). The second affine map, its bias, and the residual
increment give the displayed bound.

Set \(\mathfrak L_0(R)=0\). Equations
\eqref{eq:softmax-attention-sensitivity}--\eqref{eq:softmax-ffn-sensitivity}
define a valid parameter-sensitivity bound recursively by
\[
\begin{aligned}
 \widetilde{\mathfrak L}_\ell
 &:=\{1+Md_EB+4Md_E^{7/2}B^3R_\ell^2\}\mathfrak L_\ell(R)\\
 &\quad+M\{d_ER_\ell+4d_E^{7/2}B^2R_\ell^3\},\\
 \mathfrak L_{\ell+1}(R)
 &:=(1+d_Ed_{\rm FFN}B^2)\widetilde{\mathfrak L}_\ell
       +2d_{\rm FFN}B(d_E\widetilde R_\ell+1)+1.
\end{aligned}
\]
Induction proves the required sensitivity inequality for identical initial
inputs and arbitrary parameter vectors \(\theta,\theta'\) in the cube.
In particular, both the state and sensitivity bounds allow freely varying
query and key matrices.

Under the polynomial architecture envelope, \eqref{eq:softmax-state-radius}
gives
\[
 1+R_{\ell+1}\le n^C(1+R_\ell),
 \qquad
 \log(1+R_\ell)\le C(\ell+1)\log(en).
\]
Since \(L\le C_L\log(en)\), this proves the stated
\(O(\log^2(en))\) bound at depth \(L\). The sensitivity recursion, whose
coefficients have fixed degree in the displayed dimensions, parameter
radius, and state radius, gives
\[
 1+\mathfrak L_{\ell+1}(R)
 \le n^C(1+R_\ell)^C\{1+\mathfrak L_\ell(R)\}.
\]
Consequently,
\[
\begin{aligned}
 \log\{1+\mathfrak L_L(R)\}
 &\le C\sum_{\ell\in[L-1]_0}\{\log(en)+\log(1+R_\ell)\}\\
 &\le C(L+1)^2\log(en)\le C\log^3(en).
\end{aligned}
\]

Finally, cover the full parameter cube by a maximum-norm net of mesh
\(\eta=\delta/(L_{\mathrm{Read}}\mathfrak L_L(R))\).
Its cardinality is at most
\[
 \left(1+\frac{2B L_{\mathrm{Read}}\mathfrak L_L(R)}\delta\right)^{P_{\rm TF}}.
\]
The parameter-sensitivity inequality and the readout Lipschitz bound make
its image a \(\delta\)-net in \(\mathcal F_{\mathrm{Read}}\).
Each block has \(3Md_E^2\) attention parameters and
\(2d_Ed_{\rm FFN}+d_{\rm FFN}+d_E\) FFN parameters, giving the stated
parameter count. Under the polynomial envelope, \(P_{\rm TF}\le Cn^C\),
\(\log(BL_{\mathrm{Read}})\lesssim\log(en)\), and
\(\log\{1+\mathfrak L_L(R)\}\lesssim\log^3(en)\).
Substituting these estimates into the cube-covering bound and absorbing
the fixed power of \(\log(en)\) into a polynomial proves the final
entropy bound.
\end{proof}

\subsection{End-to-end class and entropy proxy}
\label{app:end-to-end-entropy}

Fix
\(P\in
    \mathcal P_\star
    (\mathscr M,\boldsymbol\pi,\boldsymbol\sigma_X;
    \sigma_Y^2,B_\epsilon,c_\mu)\)
and a task distribution \(\rho_f\) supported on
\(\mathcal H(\boldsymbol\alpha,\mathscr M;L_{\mathcal F})\). Let
\[
    N_n^{\rm end}:=n+1 .
\]
Take \(B_n\ge1\), enlarging the deterministic parameter envelope if needed.
Motivated by the end-to-end construction in
Theorem~\ref{thm:transformer-approx}, define
\[
    \mathcal G_n^{\rm end}
    :=
    \left\{
    \mathfrak s\mapsto
    \mathrm{Read}_n
    \bigl(\mathrm{TF}_\theta(\mathrm{Emb}_n(\mathfrak s))\bigr):
    \mathrm{TF}_\theta
    \in
    \mathcal T_{M_0}
    (N_n^{\rm end},d_{E,n},d_{{\rm FFN},n},L_n,B_n)
    \right\}.
\]
Here \(\mathcal T_{M_0}\) is the softmax class from
Section~\ref{subsec:transformers}, and \(\mathrm{Emb}_n\) and
\(\mathrm{Read}_n\) are fixed deterministic interface maps. The readout truncates to
\([-L_{\mathcal F},L_{\mathcal F}]\), so every
\(g\in\mathcal G_n^{\rm end}\) is uniformly bounded by \(L_{\mathcal F}\).

\smallskip
Under the bounded-response condition,
\(|Y_i|\le
    B_Y:=L_{\mathcal F}+B_\epsilon
\)
almost surely. Since \(B_X\) is fixed and the covariate-perturbation radii are bounded
for all sufficiently large \(n\), there is a constant \(R_0\in[1,\infty)\),
independent of \(n,K_n,\Gamma,P,\rho_f\), such that
\[
    \|\mathrm{Emb}_n(\mathfrak s)\|_{\max}\le R_0
\]
almost surely. Define the common prompt domain and its uniform norm by
\[
 \mathcal S_n:=\{\mathfrak s:\|\mathrm{Emb}_n(\mathfrak s)\|_{\max}\le R_0\},
 \qquad
 \|g\|_\infty:=\sup_{\mathfrak s\in\mathcal S_n}|g(\mathfrak s)|.
\]
Every admissible prompt law is supported on \(\mathcal S_n\). We view
\(\mathcal G_n^{\rm end}\) as a class on this common domain, and use this
norm in its covering numbers and empirical-process bounds.
Put \(\bar Q:=B_Y+L_{\mathcal F}+1\).

Let \(P_n^{\rm TF}\) be the number of trainable scalar parameters of the
transformer in \(\mathcal G_n^{\rm end}\). It satisfies
\[
    P_n^{\rm TF}
    \le
    C L_n\{M_0d_{E,n}^2+d_{E,n}d_{{\rm FFN},n}+d_{{\rm FFN},n}\}.
\]
Let \(L_{\mathrm{Read},n}\ge1\) be a Lipschitz bound for the fixed map
\(\mathrm{Read}_n\) on the transformer-output domain generated by
\(B_n\)-bounded transformers applied to
\(\mathcal Z_{N_n^{\rm end}}(R_0)\). By \eqref{eq:readout-lipschitz}, the fixed readout is globally Lipschitz
with a constant independent of \(n\), including on outputs of arbitrary
\(B_n\)-bounded transformers in the class.

Define the entropy proxy
\[
    \mathfrak H_{n,\Gamma}^{\rm end}
    :=
    P_n^{\rm TF}
    \log\!\left(
    1+
    C B_n L_{\mathrm{Read},n}\bar Q^2\Gamma\,
    \mathfrak L_{L_n}(R_0)
    \right),
\]
where \(\mathfrak L_{L_n}(R_0)\) is evaluated with
\[
    (N_n^{\rm end},d_{E,n},d_{{\rm FFN},n},L_n,B_n,M_0).
\]
Applying the first bound in Lemma~\ref{lem:transformer-entropy} with
\(\delta=(\bar Q^2\Gamma)^{-1}\) gives
\[
    \log \mathcal N\!\left(
    \mathcal G_n^{\rm end},
    {1\over \bar Q^2\Gamma},
    \|\cdot\|_\infty
    \right)
    \le
    \mathfrak H_{n,\Gamma}^{\rm end}.
\]

\begin{corollary}
\label{cor:gamma-polynomial-sufficient}
For the end-to-end architecture in Theorem~\ref{thm:transformer-approx}, there
exist constants \(C_{\rm ent},c_{\rm ent}>0\) such that
\[
    \mathfrak H_{n,\Gamma}^{\rm end}
    \le
    C_{\rm ent}n^{c_{\rm ent}}\{1+\log(e\Gamma)\}.
\]
Here \(C_{\rm ent},c_{\rm ent}\) may additionally depend on \(A_0\),
the fixed head count, and the polynomial architecture exponents.
They are uniform in \(n,K_n,\Gamma,P,\rho_f,f\).
\end{corollary}

\begin{proof}
Theorem~\ref{thm:transformer-approx} gives
\[
    N_n^{\rm end}=n+1\le Cn,
    \qquad
    M_0\le C,
\]
and
\[
    L_n\le C_{\rm depth}\log(en),
    \qquad
    d_{E,n}+d_{{\rm FFN},n}+B_n\le n^{C_{\rm size}}.
\]
The input radius \(R_0\), the response bound \(B_Y\), and
\(\bar Q=B_Y+L_{\mathcal F}+1\) are fixed. The readout Lipschitz constant is uniformly bounded by
\eqref{eq:readout-lipschitz}. Therefore
Lemma~\ref{lem:transformer-entropy} gives
\[
    \log\{1+\mathfrak L_{L_n}(R_0)\}\le C\log^3(en),
    \qquad
    P_n^{\rm TF}\le Cn^C.
\]
Thus
\[
\begin{aligned}
    \mathfrak H_{n,\Gamma}^{\rm end}
    &=
    P_n^{\rm TF}
    \log\!\left(
    1+
    C B_n L_{\mathrm{Read},n}\bar Q^2\Gamma\,
    \mathfrak L_{L_n}(R_0)
    \right)  \\
    &\le
    C n^C\{1+\log(e\Gamma)\}.
\end{aligned}
\]
Renaming constants proves the claim.
\end{proof}

\subsection{Empirical process bound}
\label{app:end-to-end-generalization}

For training data \(\mathcal D_\Gamma^{\rm tr}\), let
\[
\begin{aligned}
 \widehat R_\Gamma(g)&:=\frac1\Gamma\sum_{\gamma\in[\Gamma]}
          \{Y_{n+1}^{(\gamma)}-g(\mathfrak s^{(\gamma)})\}^2,\\
 R_Y(g)&:=\mathbb E\{Y_{n+1}-g(\mathfrak s)\}^2,\\
 R_{P,\rho_f}^\star(g)&:=\mathbb E_{f\sim\rho_f}\mathbb E_{P,f}
          \{g(\mathfrak s)-f(X_{n+1}^\star)\}^2.
\end{aligned}
\]

\begin{lemma}[Bounded least-squares empirical process]
\label{lem:bounded-ls-empirical-process}
Let \((S,Y),\{(S_\gamma,Y_\gamma)\}_{\gamma\in[\Gamma]}\) be i.i.d. with
\(|Y|\le B_Y\) almost surely. Let \(\mathcal G\) be a class of measurable maps
from the sample space of \(S\) to \([-B_G,B_G]\), and define
\[
    R(g):=\mathbb E\{Y-g(S)\}^2,
    \qquad
    \widehat R_\Gamma(g)
    :=
    {1\over\Gamma}\sum_{\gamma\in[\Gamma]}
    \{Y_\gamma-g(S_\gamma)\}^2 .
\]
Let \(g_0(S):=\mathbb E[Y\mid S]\) and
\(Q_{\mathcal G}:=B_Y+B_G+1\). Assume the displayed supremum is measurable;
otherwise the same statement holds with outer probability. Then a universal
constant \(C>0\) satisfies, for every \(t>0\),
\[
\mathbb P\!\left(
\sup_{g\in\mathcal G}
\{R(g)-R(g_0)-2\widehat R_\Gamma(g)+2\widehat R_\Gamma(g_0)\}
\ge t
\right)
\le
C\,\mathcal N\!\left(
\mathcal G,{t\over C Q_{\mathcal G}^2},\|\cdot\|_\infty
\right)
\exp\!\left(-{\Gamma t\over C Q_{\mathcal G}^4}\right).
\]
\end{lemma}

\begin{proof}
For \(g\in\mathcal G\), set
\[
    \ell_g(S,Y):=(Y-g(S))^2,
    \qquad
    \ell_0(S,Y):=(Y-g_0(S))^2.
\]
Since \(|Y|\le B_Y\), we have \(|g_0(S)|\le B_Y\) almost surely. Since
\(g_0(S)=\mathbb E[Y\mid S]\),
\[
    R(g)-R(g_0)
    =
    \mathbb E\{g(S)-g_0(S)\}^2.
\]
Moreover,
\[
    |\ell_g-\ell_0|\le C Q_{\mathcal G}^2,
    \qquad
    \mathbb E(\ell_g-\ell_0)^2
    \le
    C Q_{\mathcal G}^2\{R(g)-R(g_0)\}.
\]
Write \(P\) and \(P_\Gamma\) for population and empirical averages, and set
\(\mu_g:=P(\ell_g-\ell_0)=R(g)-R(g_0)\ge0\). Then
\[
 \mu_g-2P_\Gamma(\ell_g-\ell_0)\ge t
 \quad\Longrightarrow\quad
 (P-P_\Gamma)(\ell_g-\ell_0)\ge\frac{\mu_g+t}{2}.
\]
Bernstein's inequality gives the bound
\(\exp\{-\Gamma t/(C Q_{\mathcal G}^2)\}\).
Taking a \({t/(C Q_{\mathcal G}^2)}\)-net in \(\|\cdot\|_\infty\), using
\[
    \|\ell_g-\ell_{g'}\|_\infty
    \le
    C Q_{\mathcal G}\|g-g'\|_\infty,
\]
and applying a union bound proves the claim.
\end{proof}

\begin{lemma}[Near-ERM oracle inequality]
\label{lem:erm-oracle-transformer}
Assume \(|Y_{n+1}|\le B_Y\) almost surely, where \(B_Y<\infty\) is independent
of \(n,K_n,\Gamma,P,\rho_f\). Assume also that the fresh query-response noise
satisfies
\[
    \mathbb E[
    \epsilon_{n+1}
    \mid
    \mathfrak s,f,X^\star_{n+1},\xi_{n+1},Z_{n+1}
    ]=0.
\]
Fix a deterministic tolerance \(\varepsilon_{{\rm opt},\Gamma}\ge0\).
Let \(\widehat g_\Gamma\in\mathcal G_n^{\rm end}\) be a measurable
empirical-risk \(\varepsilon_{{\rm opt},\Gamma}\)-minimizer, that is,
almost surely,
\[
    \widehat R_\Gamma(\widehat g_\Gamma)
    \le
    \inf_{g\in\mathcal G_n^{\rm end}}\widehat R_\Gamma(g)
    +
    \varepsilon_{{\rm opt},\Gamma}.
\]
Then
\[
    \mathbb E_{\mathcal D_\Gamma^{\rm tr}}
    R_{P,\rho_f}^\star(\widehat g_\Gamma)
    \le
    2\inf_{g\in\mathcal G_n^{\rm end}}R_{P,\rho_f}^\star(g)
    +
    C{\mathfrak H_{n,\Gamma}^{\rm end}+1\over \Gamma}
    +
    2\varepsilon_{{\rm opt},\Gamma}.
\]
Here \(C\) depends only on \(B_Y,L_{\mathcal F}\) and fixed construction
constants, uniformly in \(n,K_n,\Gamma,P,\rho_f\).
\end{lemma}

\begin{proof}
Let \(g_0(\mathfrak s):=\mathbb E[Y_{n+1}\mid\mathfrak s]\), and define
\[
    Z_\Gamma
    :=
    \sup_{g\in\mathcal G_n^{\rm end}}
    \{R_Y(g)-R_Y(g_0)-2\widehat R_\Gamma(g)+2\widehat R_\Gamma(g_0)\}.
\]
Lemma~\ref{lem:bounded-ls-empirical-process}, applied on \(\mathcal S_n\) with
\(\mathcal G=\mathcal G_n^{\rm end}\), \(B_G=L_{\mathcal F}\), and hence
\(Q_{\mathcal G}=\bar Q\), gives, for every \(t>0\),
\[
\mathbb P(Z_\Gamma\ge t)
\le
C\,
\mathcal N\!\left(
\mathcal G_n^{\rm end},
{t\over C\bar Q^2},
\|\cdot\|_\infty
\right)
\exp\!\left(-{\Gamma t\over C\bar Q^4}\right).
\]
Choose
\[
    t
    =
    A\bar Q^4
    {\mathfrak H_{n,\Gamma}^{\rm end}+u\over\Gamma},
    \qquad
    u\ge1,
\]
with \(A\) sufficiently large. Since
\({t\over C\bar Q^2}
    \ge
    {1\over \bar Q^2\Gamma}\),
monotonicity of covering numbers and the definition of
\(\mathfrak H_{n,\Gamma}^{\rm end}\) imply
\[
\mathbb P\!\left(
Z_\Gamma
\ge
C{\mathfrak H_{n,\Gamma}^{\rm end}+u\over\Gamma}
\right)
\le
Ce^{-u}.
\]
For
\(\Xi_\Gamma
    :=
    R_Y(\widehat g_\Gamma)-R_Y(g_0)
    -2\widehat R_\Gamma(\widehat g_\Gamma)+2\widehat R_\Gamma(g_0)\),
we have \(\Xi_\Gamma\le Z_\Gamma\). Integrating the preceding tail bound gives
\[
    \mathbb E_{\mathcal D_\Gamma^{\rm tr}}\Xi_\Gamma
    \le
    C{\mathfrak H_{n,\Gamma}^{\rm end}+1\over\Gamma}.
\]

For every \(g\in\mathcal G_n^{\rm end}\),
\(\widehat R_\Gamma(\widehat g_\Gamma)
    \le
    \widehat R_\Gamma(g)+\varepsilon_{{\rm opt},\Gamma}\). Hence
\[
    R_Y(\widehat g_\Gamma)-R_Y(g_0)
    \le
    \Xi_\Gamma
    +
    2\widehat R_\Gamma(g)
    -
    2\widehat R_\Gamma(g_0)
    +
    2\varepsilon_{{\rm opt},\Gamma}.
\]
Taking expectations gives
\[
    \mathbb E_{\mathcal D_\Gamma^{\rm tr}}R_Y(\widehat g_\Gamma)
    \le
    2R_Y(g)-R_Y(g_0)
    +
    C{\mathfrak H_{n,\Gamma}^{\rm end}+1\over\Gamma}
    +
    2\varepsilon_{{\rm opt},\Gamma}.
\]

It remains to replace noisy query-response risk by latent regression risk. For any
deterministic \(h\), and also for \(h=\widehat g_\Gamma\) after conditioning on
\(\mathcal D_\Gamma^{\rm tr}\), the prediction \(h(\mathfrak s)\) is a function
of the fresh prompt \(\mathfrak s\) only. The conditional mean-zero assumption gives
\[
    \mathbb E[
    \{h(\mathfrak s)-f(X_{n+1}^\star)\}\epsilon_{n+1}
    ]=0.
\]
Thus
\[
    R_Y(h)
    =
    R_{P,\rho_f}^\star(h)+\sigma_{\rho,Y}^2,
    \qquad
    \sigma_{\rho,Y}^2:=\mathbb E\epsilon_{n+1}^2.
\]
Moreover,
\[
    R_Y(g_0)
    =
    \mathbb E\{\operatorname{Var}(Y_{n+1}\mid\mathfrak s)\}
    \ge
    \sigma_{\rho,Y}^2.
\]
Indeed, by the law of total variance,
\[
\begin{aligned}
    \operatorname{Var}(Y_{n+1}\mid\mathfrak s)
    &\ge
    \mathbb E\!\left[
    \operatorname{Var}(
    Y_{n+1}
    \mid
    \mathfrak s,f,X^\star_{n+1},\xi_{n+1},Z_{n+1})
    \,\middle|\,\mathfrak s
    \right]  \\
    &=
    \mathbb E\!\left[
    \operatorname{Var}(
    \epsilon_{n+1}
    \mid
    \mathfrak s,f,X^\star_{n+1},\xi_{n+1},Z_{n+1})
    \,\middle|\,\mathfrak s
    \right].
\end{aligned}
\]
Integrating and using the conditional mean-zero assumption gives
\(R_Y(g_0)\ge\sigma_{\rho,Y}^2\).

Substitution gives
\[
    \mathbb E_{\mathcal D_\Gamma^{\rm tr}}
    R_{P,\rho_f}^\star(\widehat g_\Gamma)
    \le
    2R_{P,\rho_f}^\star(g)
    +
    C{\mathfrak H_{n,\Gamma}^{\rm end}+1\over\Gamma}
    +
    2\varepsilon_{{\rm opt},\Gamma}.
\]
Taking the infimum over \(g\in\mathcal G_n^{\rm end}\) proves the claim.
\end{proof}

\begin{theorem}
\label{thm:formal-icl-erm}
Fix \(A_0>0\). Assume the hypotheses of
Theorem~\ref{thm:transformer-approx}, and suppose that
\(\mathcal G_n^{\rm end}\) contains the end-to-end predictor constructed there.
Let
\(P\in
    \mathcal P_\star
    (\mathscr M,\boldsymbol\pi,\boldsymbol\sigma_X;
    \sigma_Y^2,B_\epsilon,c_\mu)\),
and let \(\rho_f\) be supported on
\(\mathcal H(\boldsymbol\alpha,\mathscr M;L_{\mathcal F})\).
Assume \(|Y_{n+1}|\le B_Y:=L_{\mathcal F}+B_\epsilon\) almost surely and
\[
    \mathbb E[
    \epsilon_{n+1}
    \mid
    \mathfrak s,f,X^\star_{n+1},\xi_{n+1},Z_{n+1}
    ]=0 .
\]
For a deterministic tolerance \(\varepsilon_{{\rm opt},\Gamma}\ge0\), every
measurable empirical-risk
\(\varepsilon_{{\rm opt},\Gamma}\)-minimizer
\(\widehat g_\Gamma\) over \(\mathcal G_n^{\rm end}\), as defined in
Lemma~\ref{lem:erm-oracle-transformer}, satisfies
\[
    \mathbb E_{\mathcal D_\Gamma^{\rm tr}}
    R_{P,\rho_f}^\star(\widehat g_\Gamma)
    \le
    C\{\mathfrak R_n(\boldsymbol\alpha,\mathbf d,\boldsymbol\pi)+n^{-A_0}\}
    +
    C{\mathfrak H_{n,\Gamma}^{\rm end}+1\over\Gamma}
    +
    2\varepsilon_{{\rm opt},\Gamma}.
\]
Here \(C\) may additionally depend on \(A_0\), uniformly in
\(n,K_n,\Gamma,P,\rho_f,f\).
\end{theorem}

\begin{proof}
By 
Theorem~\ref{thm:oracle-upper} and Theorem~\ref{thm:transformer-approx}, the class \(\mathcal G_n^{\rm end}\)
contains a comparator \(g_n^\circ\) such that, for every
\(f\in\mathcal H(\boldsymbol\alpha,\mathscr M;L_{\mathcal F})\),
\[
    \mathbb E_{P,f}
    \{g_n^\circ(\mathfrak s)-f(X^\star_{n+1})\}^2
    \le
    C\{\mathfrak R_n(\boldsymbol\alpha,\mathbf d,\boldsymbol\pi)+n^{-A_0}\}.
\]
Averaging over \(f\sim\rho_f\) gives
\[
    \inf_{g\in\mathcal G_n^{\rm end}}
    R_{P,\rho_f}^\star(g)
    \le
    C\{\mathfrak R_n(\boldsymbol\alpha,\mathbf d,\boldsymbol\pi)+n^{-A_0}\}.
\]
Substitution into Lemma~\ref{lem:erm-oracle-transformer} proves the theorem.
\end{proof}

\subsection{Proofs of Theorem~\ref{thm:icl-erm} and Corollary~\ref{cor:main-polynomial-meta-sample}}
\label{app:proof-main-polynomial-generalization}
Theorem~\ref{thm:icl-erm} follows directly from
Theorem~\ref{thm:formal-icl-erm} and
Corollary~\ref{cor:gamma-polynomial-sufficient}. The conditional query-noise
assumption is verified below. We then prove the sufficient regime in
Corollary~\ref{cor:main-polynomial-meta-sample}.

The conditional mean-zero condition for the fresh query noise used below
follows from the i.i.d. task construction in
Section~\ref{subsubsec:generation} and Assumption~\ref{ass:latent-model}.
Indeed, for the target task,
\[
\mathbb E[
\epsilon_{n+1}
\mid
\mathfrak s,
f,
X^\star_{n+1},
\xi_{n+1},
Z_{n+1}
]
=
0.
\]

\begin{proof}
By Theorem~\ref{thm:formal-icl-erm},
\[
    \mathbb E_{\mathcal D_\Gamma^{\rm tr}}
    R_{P,\rho_f}^\star(\widehat g_\Gamma)
    \le
    C\{\mathfrak R_n+n^{-A_0}\}
    +
    C{\mathfrak H_{n,\Gamma}^{\rm end}+1\over\Gamma}
    +
    2\varepsilon_{{\rm opt},\Gamma}.
\]
Let
\(a_k:={2\alpha_k\over2\alpha_k+d_k}\) and
\(\lambda_{\max}:={2\alpha_{\max}\over2\alpha_{\max}+1}\),
then summing over \(k\) gives
\(\mathfrak R_n^{-1}\le n^{\lambda_{\max}}\).
Thus any fixed \(A_0\ge\lambda_{\max}\) is sufficient. Assume also
\(\varepsilon_{{\rm opt},\Gamma}\lesssim\mathfrak R_n\). By Corollary~\ref{cor:gamma-polynomial-sufficient},
\[
    \mathfrak H_{n,\Gamma}^{\rm end}
    \le
    C_{\rm ent}n^{c_{\rm ent}}\{1+\log(e\Gamma)\}.
\]
Let
\(\Gamma_0
    :=
    Cn^{c_{\rm ent}}\log(en)\mathfrak R_n^{-1}\)
for a sufficiently large fixed \(C>0\).
Then \(\Gamma_0\le n^C\log(en)\), so
\(\log(e\Gamma_0)\le C\log(en)\) for all large \(n\). Since
\(\Gamma\mapsto {1+\log(e\Gamma)\over\Gamma}
\)
is decreasing for \(\Gamma\ge1\), every \(\Gamma\ge\Gamma_0\) satisfies
\[
    {\mathfrak H_{n,\Gamma}^{\rm end}+1\over\Gamma}
    \lesssim
    \mathfrak R_n .
\]
Combining the preceding bounds gives
\(\mathbb E_{\mathcal D_\Gamma^{\rm tr}}
    R_{P,\rho_f}^\star(\widehat g_\Gamma)
    \lesssim
    \mathfrak R_n\), as claimed.
\end{proof}

\end{document}